\documentclass{article}

\usepackage[utf8]{inputenc} %
\usepackage[T1]{fontenc}    %
\usepackage[hyperindex,breaklinks,colorlinks,citecolor=blue]{hyperref} %
\usepackage{url}            %
\usepackage{booktabs}       %
\usepackage{amsfonts}       %
\usepackage{nicefrac}       %
\usepackage{microtype}      %
\usepackage{xspace}
\usepackage{enumitem}
\usepackage{float}
\usepackage{nicefrac}
\usepackage{multirow}
\usepackage{subfigure}
\usepackage[font=small]{caption}
\usepackage{algorithmic,algorithm}

\usepackage{graphicx}
\usepackage{natbib}

\usepackage{xspace}
\newcommand{\algopt}{\textsc{extrap-SGD}\xspace}
\newcommand{\algoptadam}{\textsc{extrap-Adam}\xspace}

\floatname{algorithm}{Algorithm}
\renewcommand{\algorithmicrequire}{\textbf{input:}}
\renewcommand{\algorithmicensure}{\textbf{output:}}

\usepackage{amsmath,amsthm,amssymb}
\newtheorem{theorem}{Theorem}[section]
\newtheorem{lemma}[theorem]{Lemma}
\newtheorem{corollary}[theorem]{Corollary}

\newtheorem{remark}[theorem]{Remark}
\newtheorem{assumption}{Assumption}

\newenvironment{talign*}
{\csname align*\endcsname}
{\endalign}

\providecommand{\abs}[1]{\left\lvert#1\right\rvert}
\providecommand{\norm}[1]{\left\lVert#1\right\rVert}

\providecommand{\R}{\mathbb{R}} %

\providecommand{\E}{{\mathbb E}}
\providecommand{\Eb}[1]{{\mathbb E}\left[#1\right] }       %

\providecommand{\var}[1]{{\text{var}}\left[#1\right] }       %

\providecommand{\0}{\mathbf{0}}

\renewcommand{\aa}{\mathbf{a}}
\providecommand{\bb}{\mathbf{b}}

\renewcommand{\gg}{\mathbf{g}}

\providecommand{\mm}{\mathbf{m}}

\providecommand{\vv}{\mathbf{v}}

\providecommand{\xx}{\mathbf{x}}
\providecommand{\yy}{\mathbf{y}}

\providecommand{\xiv}{\mathbf{\xi}}
\providecommand{\zetav}{\mathbf{\zeta}}

\providecommand{\cI}{\mathcal{I}}

\providecommand{\cO}{\mathcal{O}}

\renewcommand{\algorithmiccomment}[1]{\hfill $\triangleright$ #1}

\usepackage[accepted]{icml2020}

\icmltitlerunning{Extrapolation for Large-batch Training in Deep Learning}

\begin{document}

\twocolumn[
	\icmltitle{Extrapolation for Large-batch Training in Deep Learning}

	\icmlsetsymbol{equal}{*}

	\begin{icmlauthorlist}
		\icmlauthor{Tao Lin}{equal,to}
		\icmlauthor{Lingjing Kong}{equal,to}
		\icmlauthor{Sebastian U. Stich}{to}
		\icmlauthor{Martin Jaggi}{to}
	\end{icmlauthorlist}

	\icmlaffiliation{to}{EPFL, Lausanne, Switzerland}

	\icmlcorrespondingauthor{Tao Lin}{tao.lin@epfl.ch}

	\icmlkeywords{Machine Learning, ICML}

	\vskip 0.3in
]

\printAffiliationsAndNotice{\icmlEqualContribution} %

\begin{abstract}
	Deep learning networks are typically trained by Stochastic Gradient Descent (SGD) methods that iteratively improve the model parameters by estimating a gradient on a very small fraction of the training data. A major roadblock faced when increasing the batch size to a substantial fraction of the training data for improving training time
	is the persistent degradation in performance (generalization gap).
	To address this issue, recent work propose to add small perturbations to the model parameters when computing the stochastic gradients and report improved generalization performance due to smoothing effects.
	However, this approach is poorly understood; it requires often model-specific noise and fine-tuning.\\
	To alleviate these drawbacks, we propose to use instead computationally efficient extrapolation (extragradient) to stabilize the optimization trajectory
	while still benefiting from smoothing to avoid sharp minima.
	This principled approach is well grounded from an optimization perspective and we show that a host of variations can be covered in a unified framework that we propose.
	We prove the convergence of this novel scheme and rigorously evaluate its empirical performance on ResNet, LSTM, and Transformer.
	We demonstrate that in a variety of experiments the scheme allows scaling to much larger batch sizes than before whilst reaching or surpassing SOTA accuracy.

\end{abstract}

\section{Introduction}
The workhorse training algorithm for most machine learning applications---including deep learning---is Stochastic Gradient Descent (SGD).
Recently, data parallelism has emerged in deep learning, where large-batch~\citep{goyal2017accurate} is used to reduce the gradient computation rounds so as to accelerate the training.
However, in practice, these large-batch variants suffer from severe issues of test quality loss~\citep{shallue2018measuring, mccandlish2018empirical, golmant2018computational}, which voids gained computational advantage.
While still not completely understood, recent research has linked part of this loss of efficiency to the existence of sharp minima. In contrast, landscapes with flat minima have in empirical studies shown generalization benefits~\citep{keskar2016large,yao2018hessian,lin2020dont}, though this topic is still actively debated~\citep{dinh2017sharp}.

Another line of research tries to understand general deep learning training from an optimization perspective,
in terms of the optimization trajectory in the loss surface~\citep{neyshabur2017implicit, jastrzebski2020the}.
\citet{golatkar2019time} empirically show that regularization techniques only affect early learning dynamics (initial optimization phase) but matter little in the final phase of training (converging to a local minimum),
similar to the critical initial learning phase described in~\cite{achille2018critical} and the break-even point analysis on the entire optimization trajectory of~\citet{jastrzebski2020the}.\\
These new insights are consistent with empirically developed techniques for the SOTA large-batch training.
For example, gradual learning rate warmup for the first few epochs~\citep{goyal2017accurate,you2019large} is often used;
and in local SGD~\citep{lin2020dont} for better generalization,  stochastic noise injection is only applied \emph{after} the first phase of training (post-local SGD).

These discussions on optimization and generalization motivate us to answer the following questions when using or developing large-batch techniques for better training:\emph{
	Does the proposed technique improve (initial) optimization, or help to converge to a better local minimum, or both?
	How and when should we apply the technique?
}

In this paper, we first revisit the classical smoothing idea (which was recently attributed to avoiding sharp minima in deep learning~\citep{wen2018smoothout,haruki2019gradient}) from optimization perspective.
We then propose a computational efficient local extragradient method as a way of smoothing for distributed large-batch training, referred to as~\algopt;
we further extend it to a general framework (extrapolated SGD) for distributed training.
We thoroughly evaluate our method on diverse tasks to understand how and when it improves the training in practice.
Our empirical results justify the benefits of~\algopt, and explain the effects of smoothing ill-conditioned loss landscapes as to exhibit more well-behaved regions, which contain multiple solutions of good generalization properties~\citep{garipov2018loss}.
We show the importance of using~\algopt in the critical initial optimization phase, for the later convergence to better local minima;
the conjecture is verified by the combination with post-local SGD, which achieves SOTA large-batch training.
Our main contributions can be summarized as follows:
\begin{itemize}[leftmargin=0pt,nosep]
	\item We propose \algopt and extend it to a unified framework (extrapolated SGD) for distributed large-batch training.
	      Extensive empirical results on three benchmarking tasks justify the effects of accelerated optimization and better generalization.
	\item We provide convergence analysis for methods in the proposed framework, as well as the SOTA large batch training method (i.e. mini-batch SGD with Nesterov momentum).
	      Our analysis explains the large batch optimization inefficiency (diminishing linear speedup) observed in previous empirical work.
\end{itemize}

\section{Related Work}
\paragraph{Large-batch training.}
The test performance degradation (often denoted as \emph{generalization gap}) caused by large batch training has recently drawn significant attention~\citep{keskar2016large,hoffer2017train,shallue2018measuring,masters2018revisiting}.
\citet{hoffer2017train} argue that the generalization gap in some cases can be closed by increasing training iterations and adjusting the learning rate proportional to the square root of the batch size.
\citet{goyal2017accurate} argue the poor test performance is due to the optimization issue; they try to bridge the generalization gap with the heuristics of linear scale-up of the learning rate during training or during a warmup phase.
\citet{you2017large} propose Layer-wise Adaptive Rate Scaling (LARS) for better optimization and scaling to larger mini-batch sizes; but the generalization gap does not vanish.
\citet{lin2020dont} further propose post-local SGD on top of these optimization techniques to inject stochastic noise (to mimic the training dynamics of small batch SGD) during the later training phase.

In addition to the techniques developed for improving optimization and generalization, the optimization ineffectiveness (in terms of required training steps-to-target performance) of large-batch training has been observed.
\citet{shallue2018measuring, mccandlish2018empirical, golmant2018computational} empirically demonstrate the existence of diminishing linear speedup region across different domains and architectures.
Such a limit is also theoretically characterized in~\cite{ma2017power,yin2017gradient} for mini-batch SGD in the convex~setting.

\paragraph{Smoothing the ``sharp minima''.}
Some research links generalization performance to flatness of minima.
Entropy SGD~\citep{chaudhari2017entropy} proposes Langevin dynamics in the inner optimizer loop to smoothen out sharp valleys of the loss landscape.
From the perspective of large-batch training, \citet{wen2018smoothout} perform ``sequential averaging'' over models perturbed by isotropic noise, as a way to combat sharp minima.
\citet{haruki2019gradient} claim that injecting different anisotropic stochastic noises on local workers can smoothen sharper minima.
However, the claimed ``sharper minima'' is debatable~\citep{dinh2017sharp};
it is also unclear whether the improved results
are due to obtained flatter local minima or improved initial optimization brought by smoothing.
We defer detailed discussion to Section~\ref{subsec:unified_framework} and~\ref{subsec:experiment_results}.

\paragraph{Smoothing in classical optimization.}
Randomized smoothing has a long history in the optimization literature, see e.g. \citet{nesterov2011random,duchi2012randomized,scaman2018optimal}
which show that a faster convergence rates can be achieved by convolving non-smooth convex functions with Gaussian noise.
In contrast to non-smooth convex functions, we focus on
the smooth non-convex functions, motivated by deep neural networks.

\paragraph{Extragradient methods and optimization stability.}
Another useful building block from optimization is the extragradient method,
which is a well-known technique
to stabilize the training at each iteration by approximating the implicit update.
The method was first introduced in~\cite{korpelevich1976extragradient}
and extended to many variants, e.g. mirror-prox~\citep{nemirovski2004prox},
Optimistic Mirror Descent (OMD)~\citep{juditsky2011solving} (using past gradient information),
extragradient method with Nesterov momentum~\citep{diakonikolas2017accelerated}.
Recently its stochastic variants have found new applications in machine learning,
e.g., Generative Adversarial Network (GAN) training~\citep{daskalakis2017training,gidel2018variational,chavdarova2019reducing,mishchenko2019revisiting},
and low bit model training~\citep{leng2018extremely}.

On the theoretical side, several papers analyze the convergence of stochastic variants of extragradient.
\citet{juditsky2011solving} study stochastic mirror-prox under  restrictive assumptions.
\citet{Xu_2019} analyze stochastic extragradient in a more general non-convex setting and demonstrate tighter upper bounds than mini-batch SGD,
when using mini-batch size $\cO(1 / \epsilon^2)$.
\citet{mishchenko2019revisiting} revisit and slightly extend the stochastic extragradient for better implicit update approximation.
However, their work focuses on min-max GAN training
and argues the stochastic extragradient method might not be better than SGD for traditional function minimizations tasks.
Our work is the first that combines the idea of Nesterov momentum and extragradient (from past information)
for stochastic optimization in the setting of distributed training.

\section{Optimization with Extrapolation}
\textbf{Problem Setting and Notation.} We consider sum-structured optimization problems of the form
$
	\min_{\xx \in \R^d} f(\xx) := \frac{1}{N} \sum_{i=1}^N f_i (\xx) \,,
$
where $\xx$ denotes the parameters of the model (neural networks in our case),
and $f_i$ denotes the loss function of the $i$-th (out of $N$) training data examples.
To introduce our notation, we recall a standard update of mini-batch SGD at iteration $t$, computed on $K$ devices:
\begin{align} \label{eq:distributed_mini_batch_sgd}
	\textstyle
	\xx_{t+1} := \xx_t - \gamma_t
	\left[
	\frac{1}{K B} \sum_{k=1}^K \sum_{i \in \cI^k_{t}}
	\nabla f_{i}(\xx_t)
	\right]\,.
\end{align}
Here $\cI^k_{t}$ denotes a subset of the training points selected on device $k$ (typically selected uniformly at random) and we denote by $B := \abs{ \cI^k_{t} }$ the local mini-batch size and by $\gamma_t$ the step-size (learning rate).

\subsection{Accelerated (Stochastic) Local Extragradient for Distributed Training}
Motivated by the idea of randomized smoothing in the classic optimization literature (as for reducing the Lipschitz constant of the gradient),
we here introduce the novel idea of using extragradient locally, as a way of smoothing loss surface, for efficient distributed large-batch training.

The original idea of extrapolation (or extragradient~\citep{korpelevich1976extragradient})
was developed to stabilize optimization dynamics on saddle-point problems for a single worker, such as e.g.\ in GAN training~\citep{gidel2018variational,chavdarova2019reducing}.
The idea is to compute the gradient at an extrapolated point, different from the current point from which the update will be performed:
\begin{align*}
	\begin{split}
		\textstyle
		\xx_{t+\frac{1}{2}} = \xx_{t} - \gamma \nabla f( \xx_{t} ) \,, \qquad
		\xx_{t+1} = \xx_{t} - \gamma \nabla f(\xx_{t + \frac{1}{2}}) \,.
	\end{split}
\end{align*}
This step is intrinsically different from the well-known and widely used accelerated method (i.e. Nesterov momentum):
\begin{talign*}
	\begin{split}
		&\xx_{t+\frac{1}{2}} = \xx_t + u \vv_t \,,
		\vv_{t+1} = u \vv_t - \gamma \nabla f( \xx_{t+\frac{1}{2}} ) \,, \\
		&\xx_{t + 1} = \xx_{t+\frac{1}{2}} - \gamma \nabla f( \xx_{t+\frac{1}{2}} ) \,,
	\end{split}
\end{talign*}
where here $1 > u \geq 0$ denotes the momentum parameter.
The key difference lies in the lookahead step for the gradient computation of these two methods.

Considering different extrapolated local models (with Nesterov momentum) under the distributed training,
\algopt combines the effects of randomized smoothing and the stabilized optimization dynamics through extrapolation.
Our~\algopt follows the idea of extrapolating from the past
(currently only used for single worker training~\citep{gidel2018variational,daskalakis2017training}),
as a means to avoid additional cost for gradients used to form an extrapolation point.
The~\algopt method is detailed in Algorithm~\ref{algo:our_main_algo}\footnote{
	we omit the extrapolation step in line $2$ when $t \!=\! 0$.
},
where the previous local gradients are used for the extrapolation
and the superscript $k$ stresses that local models are different.
Using the past local mini-batch gradients for extrapolation
allows the extrapolation scale $\hat{\gamma}$ to directly take on the learning rate used for small mini-batch training with size $B$,
thus avoids the difficulty of hyper-parameter tuning.
Note that setting $\hat{\gamma} \!=\! 0$ in Algorithm~\ref{algo:our_main_algo} recovers the SOTA large batch training method,
i.e., mini-batch SGD with Nesterov momentum.

To the best of our knowledge, it is the first time such an extragradient method is used locally with Nesterov Momentum
under the framework of smoothing, for accelerated and smoothed distributed optimization.

\begin{algorithm}[!]\footnotesize
	\caption{\small\itshape {\algopt}}
	\label{algo:our_main_algo}
	\begin{algorithmic}[1]
		\REQUIRE learning rate $\gamma$, inner learning rate $\hat{\gamma}$, momentum factor $u$,
		initial parameter $\xx_0$, initial moment vector $\vv_0 \!=\! 0$, time step $t \!=\! 0$, worker index $k$.

		\WHILE{$\xx_t$ not converged}
		\STATE $\xx_{t+\frac{1}{4}}^k = \xx_t - \frac{\hat{\gamma}}{B} \sum_{i \in \cI^k_{t}} \nabla f_i(\xx_{t - \frac{1}{2}}^k)$            \COMMENT{extrapolation step}
		\STATE $\xx_{t+\frac{1}{2}}^k = \xx_{t+\frac{1}{4}}^k + u \vv_t$                                                            \COMMENT{Nesterov momentum}
		\STATE $\vv_{t+1} = u \vv_t - \frac{\gamma}{KB} \sum_{k, i \in \cI^k_{t}} \nabla f_i(\xx_{t + \frac{1}{2}}^k)$         \COMMENT{update buffer}
		\STATE $\xx_{t + 1} = \xx_{t} + \vv_{t+1}$                                                                                  \COMMENT{actual update}
		\ENDWHILE
		\ENSURE $\xx_t$.
	\end{algorithmic}
\end{algorithm}

\subsection{Unified Extrapolation Framework} \label{subsec:unified_framework}
Our Algorithm~\ref{algo:our_main_algo} can be extended to a more general extrapolation framework for distributed training,
by using diverse extrapolation choices $\zetav_t^k$.
It is achieved by replacing line~2 in Algorithm~\ref{algo:our_main_algo}
by $\xx_{t+\frac{1}{4}}^k \!=\! \xx_t - \hat{\gamma} \zetav_t^k$.
We denote our framework as \emph{extrapolated SGD},
covering different choices of noise $\zetav_t^k$ (e.g. Gaussian noise, uniform noise, and stochastic gradient noise)
and~\algopt (past mini-batch gradients).
We detail some choices of $\zetav_t^k$ below and use them as our close baselines for~\algopt:
\begin{itemize}[nosep,leftmargin=0pt,parsep=2pt,itemindent=10pt]
	\item $\zetav_t^k$ as a form of isotropic noise.
	      The noise can be sampled from e.g.\ an isotropic Gaussian or uniform distribution.
	      Following the idea of~\citet{li2017visualizing},
	      the strength of noises added to a filter can be linearly scaled by the $l_2$ norm of the filter,
	      instead of fixing a constant perturbation strength over different layers.
	      Formally, the scaled noise $\zetav^k_{t}$ of $j$-th filter at layer $i$ on worker $k$ follows
	      $\norm{ \xx_{t, i, j} } \cdot \hat{\zetav}^k_{t, i, j} / \bigl\| \hat{\zetav}^k_{t, i, j} \bigr\|$.
	      A similar idea was proposed in SmoothOut~\citep{wen2018smoothout},
	      corresponding to letting $\zetav^k_{t} := \zetav_{t}$ in our framework.

	\item $\zetav_t^k$ as a form of anisotropic noise.
	      \citet{kleinberg2018alternative} interpret sequential SGD updates as GD with stochastic gradient noise convolution over update steps.
	      This motivates to use stochastic gradient noise for smoothing (similarly proposed in~\citet{haruki2019gradient}),
	      thus $\zetav_{t}^k$ can be chosen as:
	      \begin{align*}
		      \textstyle
		      \frac{1}{B} \sum_{i \in \cI^k_{t}} \nabla f_i(\xx_{t - \frac{1}{2}}^k) - \frac{1}{KB} \sum_{k} \sum_{i \in \cI^k_{t}} \nabla f_i(\xx_{t - \frac{1}{2}}^k) \,.
	      \end{align*}
\end{itemize}

The side effects of these noise extrapolation variants are the training setup sensitivity, causing the hyperparameter tuning difficulty and limited practical applications.
For example, the isotropic noise requires to manually design the noise distribution for each model and dataset;
the anisotropic noise distribution will be dynamically varied by different choices of the number of workers, the local mini-batch size, and the objective of the learning task~\citep{zhang2019adam}.

Despite the existence of variants~\citep{wen2018smoothout,haruki2019gradient} and their reported empirical results\footnote{
	The empirical results of~\citet{haruki2019gradient} are not solid.
	Taking the results of CIFAR-10 for mini-batch size $8{,}192$ into account
	and use three trials' experimental results for the same choice of $\zetav_t^k$ (e.g. layerwise uniform noise) as an example,
	our experimental results can reach reasonable test top-1 accuracy (at around $91$),
	much better than their presented results (at around $63$).
},
none of them has analyzed their convergence behaviors.
In the next section, we provide rigorous convergence analysis for our algorithm for distributed training
(illustrated in Algorithm~\ref{algo:our_main_algo}, which also includes the SOTA practical training algorithm).
In Section~\ref{sec:experiments}, we empirically evaluate all related methods to better understand the benefits of using extrapolation with smoothing for distributed large-batch training.

\section{Theoretical Analysis of Nesterov Momentum and~\algopt} \label{sec:theoretical_analysis}
We now turn to the theoretical convergence analysis, i.e.\ we derive an upper bound on the number of iterations to find an
approximate solution with small gradient norm.

Following the convention in distributed stochastic optimization, We denote by $f^\star$ a lower bound on the values of $f(\xx)$ and use the following assumptions:
\begin{assumption}[Unbiased Stochastic Gradients] \label{assumption:unbiase}
	$\forall i \in [N], t \in [T]$, it holds
	$ \Eb{ \nabla f_{i} (\xx_t) } = \nabla f (\xx_t) $.
\end{assumption}

\begin{assumption}[Bounded Gradient Variance] \label{assumption:boundedvariance}
	$\exists \sigma^2 > 0, \forall i \in [N], t \in [T]$, s.t.
	$ \Eb{ \norm{ \nabla f_i(\xx_t) - \nabla f(\xx_t) }^2} \leq \sigma^2 $.
\end{assumption}
Here $\sigma^2$ quantifies the variance of stochastic gradients at each local worker
and we assume workers access IID training dataset (e.g., data center setting).

\begin{assumption}[Lipschitz Gradient] \label{assumption:smoothness}
	$\exists L > 0$, s.t. $\forall \xx, \yy \in \R^d, i \in [N] $,
	the objective function $f_i: \R^d \rightarrow \R $ satisfies the following condition
	$ \norm{ \nabla f_i (\xx) - \nabla f_i (\yy) } \leq L \norm{ \xx - \yy }$.
\end{assumption}

\subsection{Analysis for Mini-batch SGD (with Nesterov Momentum)}\label{subsec:minibatch_sgd_speedup}
In this section we first recall the known convergence guarantees for mini-batch SGD (without momentum) on non-convex functions for later reference, and
then derive new guarantees for mini-batch SGD with Nesterov momentum.

\begin{theorem}[Convergence of stochastic distributed mini-batch SGD for non-convex functions~\citep{ghadimi2016accelerated}]\label{thm:standard_convergence_minibatchsgd_nonconvex}
	Under Assumptions~\ref{assumption:unbiase}--\ref{assumption:smoothness}, after $T$ mini-batch gradient updates, each using $KB$ samples, the mini-batch SGD returns an iterate $\xx$ satisfying
	\begin{align*}
		\textstyle
		\Eb{ \norm{  \nabla f(\xx) }^2 } \leq \cO \left( \frac{L \left( f(\xx_0) - f^\star \right)}{T} + \frac{\sigma \sqrt{ L \left( f(\xx_0) - f^\star \right) } }{ \sqrt{ KBT } } \right) \,.
	\end{align*}
\end{theorem}

The second term in the rate is asymptotically dominant
as long as $KB \!=\! \cO \bigl( \frac{ \sigma^2 T }{ L \left( f(\xx_0) - f^\star \right) } \bigr)$.
In this regime, increasing the mini-batch size gives a linear speedup, as $T \!=\! \cO \bigl( \frac{ \sigma^2 L \left( f(\xx) - f^\star \right) }{KB \epsilon^2} \bigr)$ decreases in $KB$, as similarly pointed out by~\citet{wang2017stochastic}.
However, when we increase the mini-batch size beyond this critical point, the first term dominates the rate and increasing the mini-batch size further will have less impact on the convergence.
This phenomenon has also been empirically verified in deep learning applications~\citep{shallue2018measuring}.

In practice, mini-batch SGD with Nesterov momentum~\citep{nesterov1983method} is the state-of-the-art deep learning training scheme.
However, previous theoretical analysis normally relies on the strong assumption of the bounded mini-batch gradients~\citep{yan2018unified}.
Here we provide a better convergence analysis without such an assumption
for distributed mini-batch SGD with Nesterov momentum following closely~\citet{yu2019linear}.
The convergence rate is detailed in Theorem~\ref{thm:non_convex_convergence_nesterov_minibatch},
and for the proof details we refer to Section~\ref{sec:nonconvex_proof_nesterov_momentum} of Appendix~\ref{part:convergence_proof}.

\begin{theorem}[Convergence of mini-batch SGD with Nesterov momentum for non-convex functions] \label{thm:non_convex_convergence_nesterov_minibatch}
	Under Assumption~\ref{assumption:unbiase}--\ref{assumption:smoothness},
	for mini-batch SGD with Nesterov momentum, i.e.,
	$\xx_{t+\frac{1}{2}} \!=\! \xx_t + u \vv_t$,
	$\vv_{t+1} \!=\! u \vv_t - \frac{\gamma}{KB} \sum_{k=1}^K \sum_{i \in \cI_t^k} \nabla f_i(\xx_{t + \frac{1}{2}})$,
	$\xx_{t + 1} \!=\! \xx_{t} + \vv_{t+1}$,
	we can show that for optimally tuned stepsize (cf.\ Lemma~\ref{lemma:bound}) $\gamma \leq \frac{ 2 ( 1 - u )^2 }{ L ( u^3 + 1 ) }$,
	it holds
	\begin{talign*}
		\begin{split}
			\Eb{ \norm{ \nabla f(\xx)}^2 }
			= \cO \left(
			\frac{L r_0 (u^3 + 1)}{T (1-u)} + \sqrt{ \frac{2 L r_0 \sigma^2}{KBT (1 - u)} }
			\right) \,,
		\end{split}
	\end{talign*}
	where here $\xx$ denotes a uniformly at random selected $\xx_{t+\frac{1}{2}}$ iterate, i.e.\ $\frac{1}{T} \sum_{t=0}^{T-1} \bigl\| \nabla f ( \xx_{t + \frac{1}{2}} ) \bigr\|^2$,
	and $r_0 := f(\xx_0) - f^\star$.
\end{theorem}
The second term is asymptotically dominant as long as
$KB = \cO \left( \frac{ (1 - u) }{ (u^3 + 1)^2 } \frac{\sigma^2 T}{ L( f(\xx_0) - f^\star ) } \right)$
and for small $K$ we can achieve the linear speedup where $T = \cO \left( \frac{ L \sigma^2 (f(\xx_0) - f^\star) }{(1 - u) KB \epsilon^2} \right)$.

\begin{remark}
	\citet{yu2019linear}
	argue that a linear speedup $\cO (1 / \sqrt{KT})$ for  SGD with the local Nesterov momentum%
	can be achieved.
	However, this claim is only valid for large $T$ and stepsize $\gamma = \sqrt{K / T}$, cf.\ \citep[Cor. 1]{yu2019linear}.
	We here tune the stepsize differently and show a tighter bound that holds for all $T$,
	providing a better critical mini-batch size analysis (diminishing linear speedup in terms of optimization) for mini-batch SGD with Nesterov momentum~\citep{shallue2018measuring}.
\end{remark}

\subsection{Convergence of~\algopt} \label{subsec:convergence}
In this subsection, we show the convergence analysis for our novel~\algopt for non-convex functions.
We also include the analysis for other noise variants of extrapolated SGD,
which explains their potential limitations.
The proof details can be found in the Section~\ref{sec:extrap_sgd_convergence_proof} of Appendix~\ref{part:convergence_proof}.

\begin{theorem}[Convergence of~\algopt for non-convex functions] \label{thm:extrasgd_nonconvex_convergence}
	Under Assumption~\ref{assumption:unbiase}--\ref{assumption:smoothness},
	and by defining $\bar{\xx}_{t + \frac{1}{2}}:=\frac{1}{K} \sum_{k=1} \xx^k_{t+\frac{1}{2}}$,
	it holds for  $\hat{\gamma} \leq \frac{ u^2 }{ ( 1 - u )^2 } \gamma$ and $\gamma \leq \frac{ ( 1 - u )^2 }{ L ( 1 + 3 u + u^3 ) }$:
	\begin{talign*}
		\begin{split}
			&\E{ \frac{1}{T} \sum_{t=0}^{T-1} \norm{ \nabla f ( \bar{\xx}_{t+\frac{1}{2}} ) }^2 }\\
			&\leq \frac{2 (1 - u)}{\gamma T} \Eb{ f( \bar{\xx}_{0} ) - f^\star }
			+ \left( \frac{ 4 \hat{\gamma}^2 L^2 }{ B } + \frac{ \gamma L (1 + 3 u) }{ (1 - u)^2 B K } \right) \sigma^2 \,.
		\end{split}
	\end{talign*}
\end{theorem}

\begin{remark}
	Using past local gradients for extrapolation in~\algopt allows us to directly set $\hat{\gamma} \approx \frac{\gamma}{K}$ for~\algopt,
	where the constraint of $\hat{\gamma} \leq \frac{ u^2 }{ ( 1 - u )^2 } \gamma$ is normally satisfied.
\end{remark}

\begin{corollary} \label{corollary:critical_minibatch_size_for_extragradient}
	Considering Theorem~\ref{thm:extrasgd_nonconvex_convergence} and tuning the stepsize as in Lemma~\ref{lemma:bound},
	with $\gamma \leq \frac{ ( 1 - u )^2 }{ L ( 1 + 3 u + u^3 ) }$ and $\hat{\gamma} \leq \frac{ \gamma }{ K }$,
	and $r_0 := f(\xx_0) - f^\star$,
	we can rewrite the convergence rate of Theorem~\ref{thm:extrasgd_nonconvex_convergence} as
	\resizebox{\linewidth}{!}{\vbox{
			\begin{talign*}
				\begin{split}
					&\Eb{ \frac{1}{T} \sum_{t=0}^{T-1} \norm{ \nabla f ( \bar{\xx}_{t+\frac{1}{2}} ) }^2 } \\
					& = \cO \left( 4 (u^3 + 3 u + 1) \frac{ L r_0 }{T (1-u)}
					+ 2 \sqrt{ \frac{( 19 u + 1 )}{( u^3 + 3 u + 1 )}  \frac{2 L r_0 \sigma^2 }{ KBT ( 1 - u ) } } \right) \,.
				\end{split}
			\end{talign*}
		}}\vspace{-5mm}
\end{corollary}
The second term is asymptotically dominant as long as
$KB = \cO \bigl(  \frac{ (19 u + 1) (1 - u) }{ (u^3 + 3 u + 1)^3 } \frac{\sigma^2 T}{ L( f(\xx_0) - f^\star ) } \bigr)$
and for small $K$ we can achieve the linear speedup where $T = \cO \left( \frac{ L \sigma^2 (f(\xx_0) - f^\star) }{ KB \epsilon^2} \frac{ 19 u + 1 }{ ( u^3 + 3 u + 1 ) ( 1 - u ) } \right)$,

\begin{remark}
	By setting $u = 0$, we recover the same rates for \algopt in non-convex cases (Theorem~\ref{thm:non_convex_convergence_nesterov_minibatch} and Corollary~\ref{corollary:critical_minibatch_size_for_extragradient})
	as for standard mini-batch SGD (Theorem~\ref{thm:standard_convergence_minibatchsgd_nonconvex}) but we cannot show an actual speedup over mini-batch SGD by setting $u> 0$.
	However, thus might not necessarily be a limitation of our approach as to the best of our knowledge, there exist so far no   theoretical results for stochastic momentum methods that can show a speedup over mini-batch SGD.
\end{remark}

The analysis below extends the proof of~\algopt to the other cases of our extrapolated SGD framework.
\begin{theorem} \label{thm:extrasgd_nonconvex_convergence_noise}
	Under the extrapolation framework,
	IID random noise $ \zeta_t^k $ (instead of the past local mini-batch gradients in Algorithm~\ref{algo:our_main_algo}) is used for the local extrapolation,
	where $ \Eb{ \zetav_t^k } = 0 $ and $ \Eb{ \norm{ \zetav_t^k }^2 } \leq \hat{\sigma}^2 $.
	Under Assumption~\ref{assumption:unbiase}--\ref{assumption:smoothness},
	it holds for $ \gamma \leq \frac{ ( 1 - u )^2 }{ L ( 1 + u + u^3 ) } $:
	\begin{talign*}
		\begin{split}
			&\E{ \frac{1}{T} \sum_{t=0}^{T-1} \norm{ \nabla f ( \bar{\xx}_{t+\frac{1}{2}} ) }^2 }
			\leq \frac{2 (1 - u)}{\gamma T} \Eb{ f( \bar{\yy}_{0} ) - f ( \bar{\yy}_{T} ) } \\
			&\qquad + \frac{ \gamma L ( 1 + u ) }{ (1 - u)^2 BK } \sigma^2
			+ (L^2 + \frac{ (1 - u)^2 L }{ \gamma u^3 K })2 \hat{\gamma}^2 T \hat{\sigma}^2 \,.
		\end{split}
	\end{talign*}
\end{theorem}

\begin{remark}
	The choice of using random noise for extrapolation
	in Theorem~\ref{thm:extrasgd_nonconvex_convergence_noise}
	requires the manual introduction of the noise distribution
	($\hat{\sigma}^2$) for each problem setup.
	The dependence on the unknown relationship between $\hat{\sigma}^2$ and $\sigma^2$
	results in the difficulty of providing a concise convergence analysis
	(e.g. exact convergence rate, the critical mini-batch size) for Theorem~\ref{thm:extrasgd_nonconvex_convergence_noise}.
\end{remark}

\begin{figure*}[!h]
	\centering
	\subfigure[\small Learning curves for different methods (w/o learning rate decay).]{
		\includegraphics[width=0.475\textwidth,]{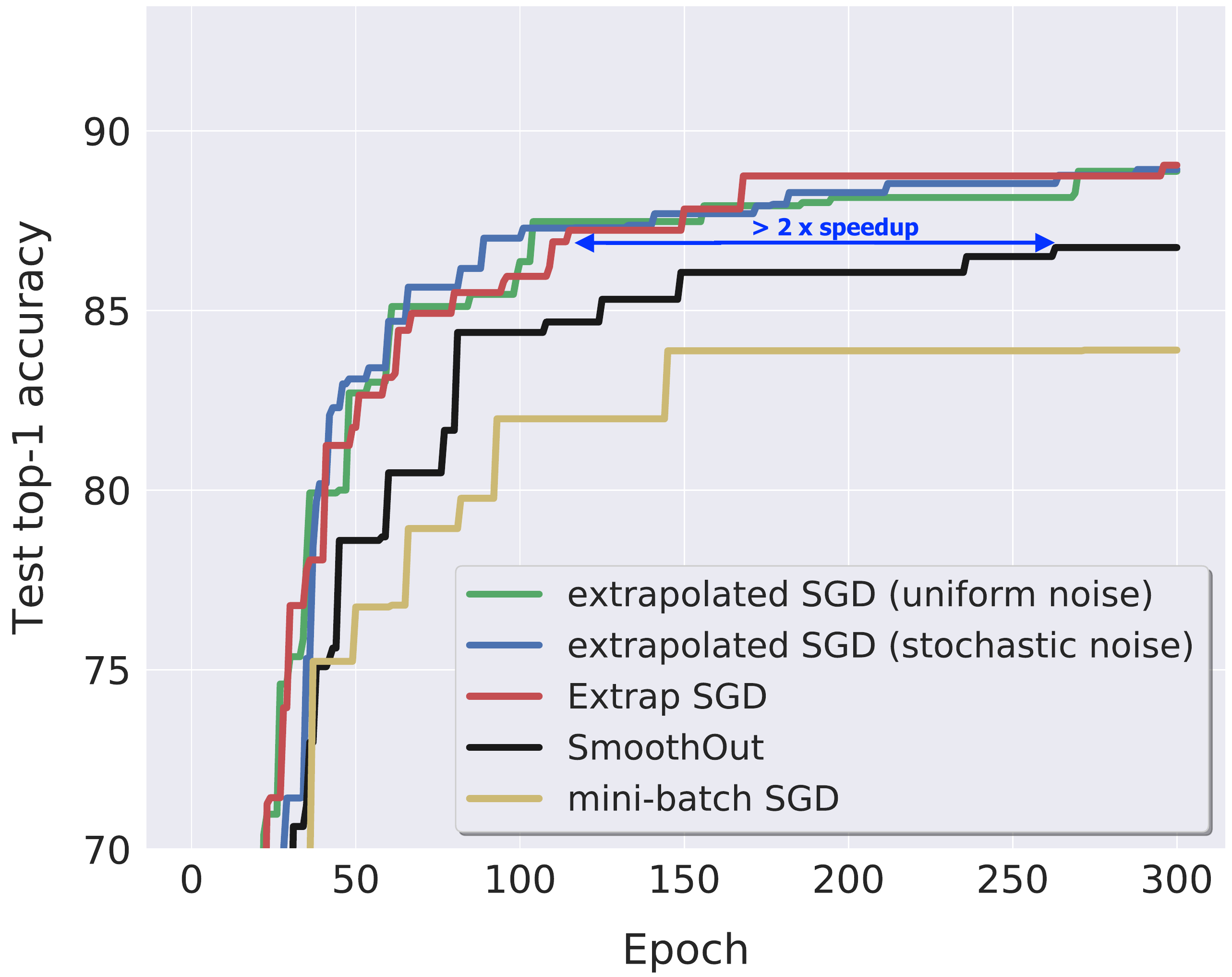}
		\label{fig:resnet20_cifar10_8k_learning_curves_different_methods_without_learning_rate_decay}
	}
	\hfill
	\subfigure[\small Better smoothness of~\algopt.]{
		\includegraphics[width=0.475\textwidth,]{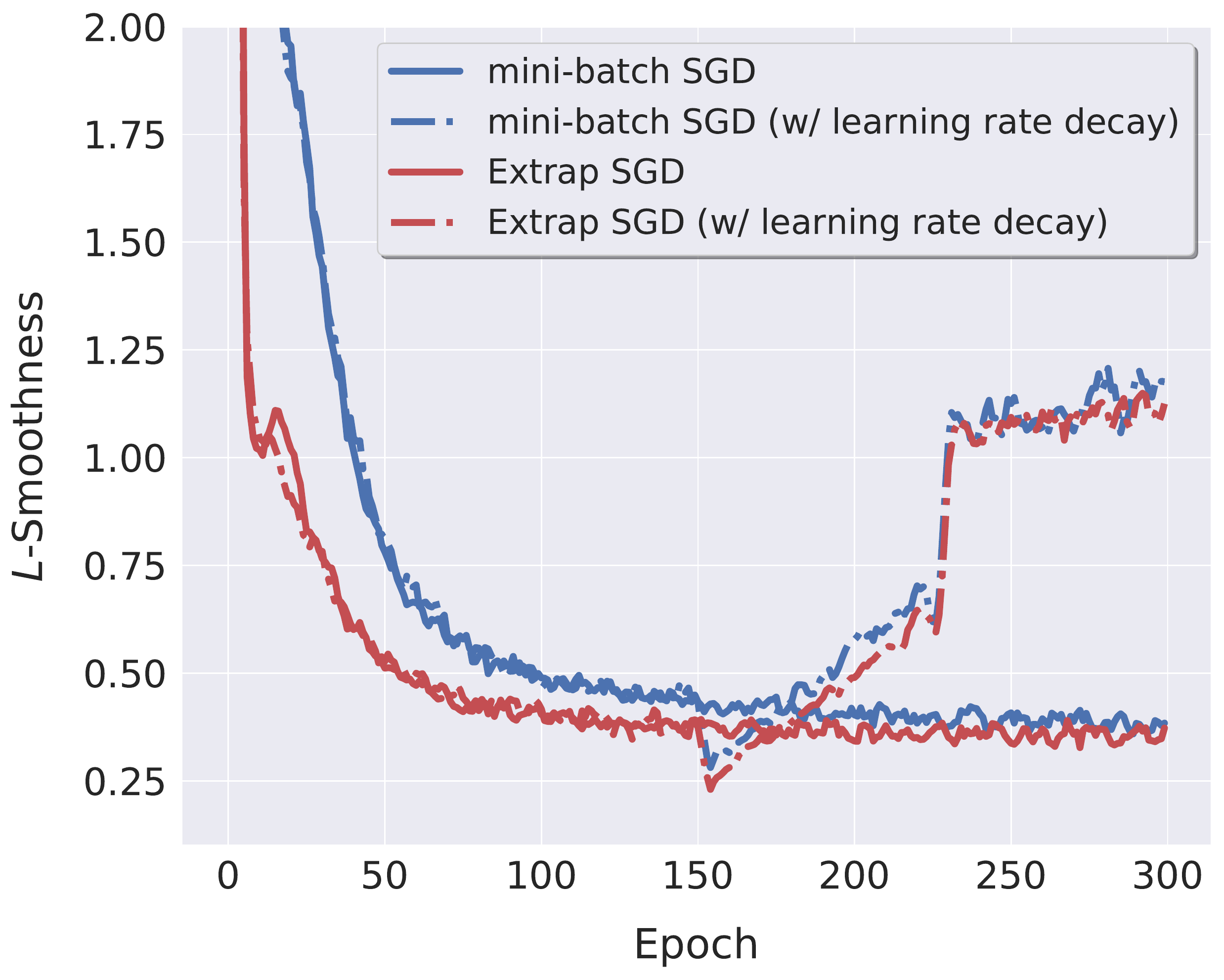}
		\label{fig:resnet20_cifar10_smoothness}
	}
	\vspace{-1em}
	\caption{\small
		Understanding the learning behaviors of different methods on the large-batch training
		(with mini-batch size $8,192$ on $32$ workers) for training ResNet-20 on CIFAR-10.
		The visualization of smoothness in Figure~\ref{fig:resnet20_cifar10_smoothness}
		follows the idea in~\citep{santurkar2018does,haruki2019gradient};
		it takes $8$ additionally steps (each with $30\%$ of the update) in the direction of the update for each training step,
		and the smoothness of a training step is expressed by the maximum value of $L$ (evaluated after local update steps)
		satisfying the Assumption~\ref{assumption:smoothness}.
		We use the learning rate scaling and warmup in~\citet{goyal2017accurate} for the first $5$ epochs,
		and $\gamma$ and $\tilde{\gamma}$ are fine-tuned for different methods.
	}
	\vspace{-0.5em}
	\label{fig:resnet20_cifar10_8k_256x32_learning_curves_and_smoothness}
\end{figure*}

\section{Experiments} \label{sec:experiments}
We first briefly outline the general experimental setup below (for more details refer to Appendix~\ref{appendix:detailed_experiment_setup})
and then thoroughly evaluate our framework on different challenging large-batch training tasks.
We aim to challenge and push the limit of large-batch training.
We limit our attention to three standard and representative benchmarking tasks, with a controlled epoch budget (for each task).
We ensure the used mini-batch size is a significant fraction of the whole dataset.
Performing experiments on a much larger dataset for the same demonstration purposes is out of our computational ability\footnote{
	E.g., we use the mini-batch size of $8{,}192$ (roughly $16\%$ of the total data),
	out of $50{,}000$ samples for CIFAR as our main tool of justification.
	While for ImageNet~\citep{russakovsky2015imagenet} with $1.28$ million data samples in total,
	the same fraction would result in roughly $800$ workers for local mini-batch of size $256$.
}.

\subsection{Experimental Setup}
\paragraph{Datasets.} We evaluate all methods on the following three tasks:
(1) Image Classification for CIFAR-10/100~\citep{krizhevsky2009learning} ($50$K training samples and $10$K testing samples with $10/100$ classes)
with the standard data augmentation and preprocessing scheme~\citep{he2016deep,huang2016deep};
(2) Language Modeling for WikiText2~\citep{merity2016pointer}
(the vocabulary size is $33$K, and its train and validation set have $2$ million tokens and $217$K tokens respectively);
and (3) Neural Machine Translation for Multi30k~\citep{elliott2016multi30k}.

\paragraph{Models and training schemes.}
Several benchmarking models are used in our experimental evaluation.
(1) ResNet-20~\citep{he2016deep} and VGG-11~\citep{simonyan2014very} on CIFAR for image classification,
(2) two-layer LSTM~\citep{merity2017regularizing} with hidden dimension of size $128$ on WikiText-2 for language modeling,
and (3) a down-scaled transformer (factor of $2$ w.r.t.< the base model in~\citet{vaswani2017attention}) for neural machine translation.
Weight initialization schemes for the three tasks follow~\citet{goyal2017accurate,he2015delving},~\citet{merity2017regularizing} and~\citet{vaswani2017attention} respectively.

We use mini-batch SGD with a Nesterov momentum of $0.9$ without dampening for image classification and language modeling tasks,
and Adam for neural machine translation tasks.
In the following experiment section,
the term ``mini-batch SGD'' indicates the mini-batch SGD with Nesterov momentum unless mentioned otherwise.

For experiments on image classification and language modeling,
unless mentioned otherwise the models are trained for $300$ epochs;
the local mini-batch sizes are set to $256$ and $64$ respectively.
By default, all related experiments will use learning rate scaling and warmup scheme\footnote{
	Since we will fine-tune the (to be scaled) learning rate,
	there is no difference between learning rate linear scaling~\citep{goyal2017accurate} and square root scaling~\citep{hoffer2017train} in our case.
}~\citep{goyal2017accurate,hoffer2017train}.
The learning rate is always gradually warmed up from a relatively small value for the first few epochs.
Besides, the learning rate $\gamma$ in image classification task will be dropped by a factor of $10$
when the model has accessed $50\%$ and $75\%$ of the total number of training samples~\citep{he2016deep,huang2016densely}.
The LARS is only applied on image classification task\footnote{
	Our implementation relies on the PyTorch extension of \href{https://github.com/NVIDIA/apex}{NVIDIA apex}
	for mixed precision and distributed training.
}~\citep{you2017large}.

For experiments on neural machine translation, we use standard inverse square root learning rate schedule~\citep{vaswani2017attention}.
The warmup step is set to $4000$ for the mini-batch size of $64$ and will be linearly scaled down by the global mini-batch size\footnote{
	We follow the practical \href{https://github.com/NVIDIA/DeepLearningExamples/tree/master/PyTorch/Translation/Transformer}{instruction} from NVIDIA.
}.

We carefully tune the learning rate $\gamma$ and
the trust term $\tilde{\gamma}$ in~\citet{you2017large}.
The tuning procedure ensures that the best hyper-parameter lies in the middle of our search grids;
otherwise, we extend our search grid.
The procedure of hyperparameter tuning can be found in Appendix~\ref{appendix:hyperparameter_values}.

\subsection{Evaluation on Large-batch Training} \label{subsec:experiment_results}
\paragraph{Superior performance of~\algopt on different tasks.}
We evaluate our extrapolation framework and compare it with SOTA large-batch training methods
on CIFAR-10 image classification (Figure~\ref{fig:resnet20_cifar10_8k_256x32_learning_curves_and_smoothness})
and WikiText2 language modeling (Figure~\ref{fig:lstm_wikitext2_learning_curves}).
To better exhibit the optimization behaviors of different methods, for these two tasks we do not decay the learning rate.
\emph{The extrapolated SGD framework in general significantly accelerates the optimization and leads to better test performance than the existing SOTA methods.}
For example, the smaller gradient Lipschitz constant illustrated in Figure~\ref{fig:resnet20_cifar10_smoothness} demonstrates the improved optimization landscape,
which explains the at least $2 \times$ speedup in terms of the convergence (after thorough hyperparameter tuning) in
Figure~\ref{fig:resnet20_cifar10_8k_learning_curves_different_methods_without_learning_rate_decay} and Figure~\ref{fig:lstm_wikitext2_learning_curves}.

We further extend\footnote{
	It is non-trivial to adapt Adam to the noise variants of the extrapolated SGD.
} the extrapolation idea of~\algopt to~\algoptadam and validate its effectiveness (compared with Adam) on neural machine translation with Transformer.
The algorithmic description refers to Algorithm~\ref{algo:our_main_algo_adam_variant} in Appendix~\ref{appendix:extrap_adam}.
Figure~\ref{fig:multi30k_transformer} shows the results of large-batch training (using $4\%$ and $16\%$ of the training data per mini-batch)
and \emph{\algoptadam again outperforms the Adam with at least $2\times$ speedup.}

\begin{figure}[!h]
	\centering
	\includegraphics[width=0.4\textwidth,]{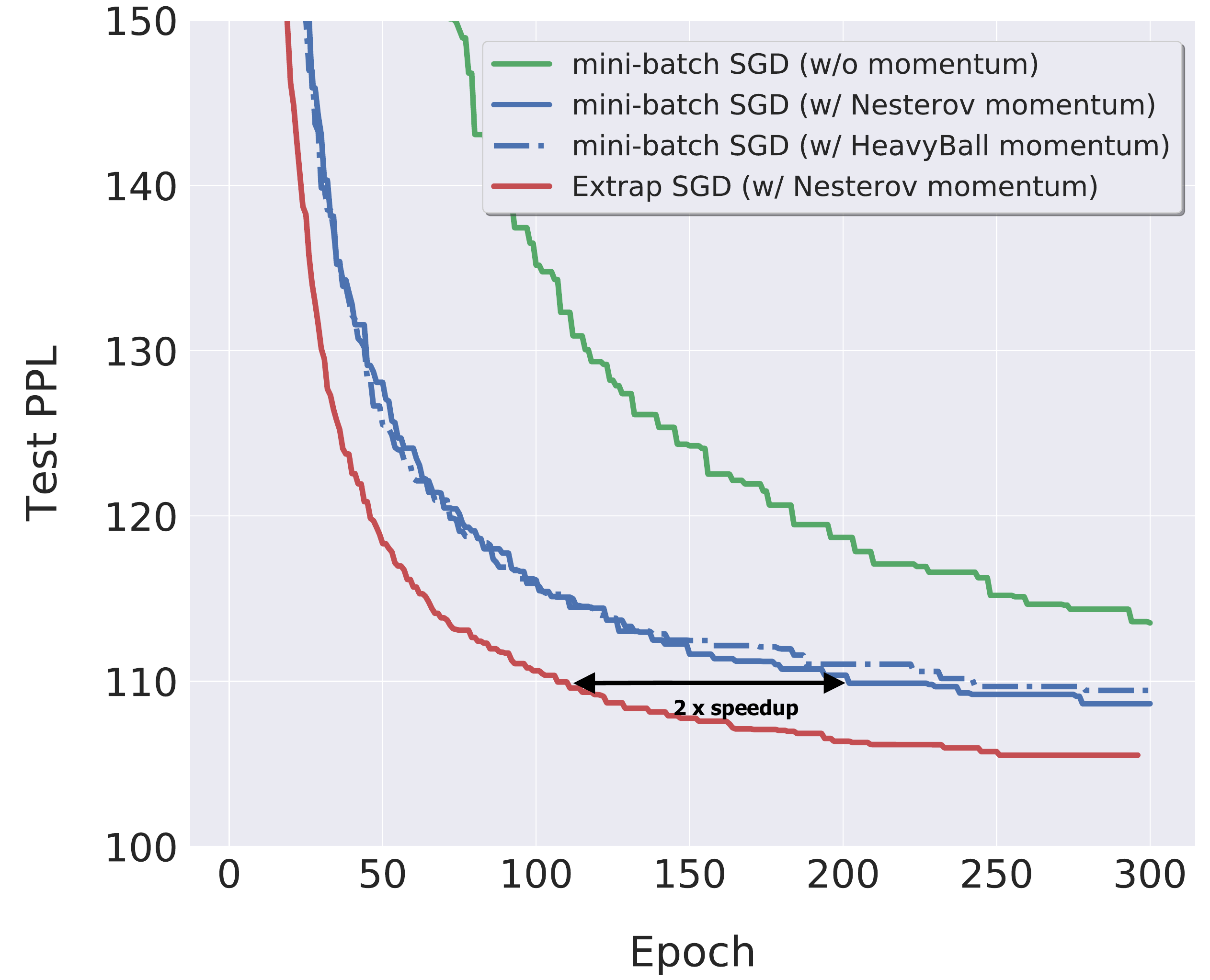}\\
	\vspace{0.5em}
	\resizebox{0.8\linewidth}{!}{
		\def\arraystretch{1.5}
		\setlength{\tabcolsep}{3pt}
		\begin{tabular}{ccccc}
			\toprule
			               & $K\!=\!24$        & $K\!=\!48$        \\ \midrule
			mini-batch SGD & $108.39 \pm 0.31$ & $110.16 \pm 0.67$ \\
			\algopt        & $105.86 \pm 0.32$ & $107.86 \pm 0.50$ \\ \bottomrule
		\end{tabular}%
	}
	\caption{\small
		The perplexity (PPL, the lower the better) of training WikiText-2 on LSTM.
		The global mini-batch size are $1{,}536$ and $3{,}072$ for $K\!=\!24$ and $K\!=\!48$ respectively,
		accounting for $2\%$ and $4\%$ of the total training data.
		We use the learning rate scaling and warmup in~\citet{goyal2017accurate}, and use constant learning rate after the warmup.
		We finetune the $\gamma$ for mini-batch SGD (and its momentum variants);~\algopt reuses the hyper-parameter from mini-batch SGD.
		The results of the inline table are averaged over three different seeds.
		The displayed learning curves are based on $K\!=\!24$
		and more details refer to Appendix~\ref{appendix:lstm_wikitext2}.
	}
	\vspace{-0.5em}
	\label{fig:lstm_wikitext2_learning_curves}
\end{figure}

\begin{figure}[!h]
	\centering
	\subfigure[\small Perplexity (validation).]{
		\includegraphics[width=0.215\textwidth,]{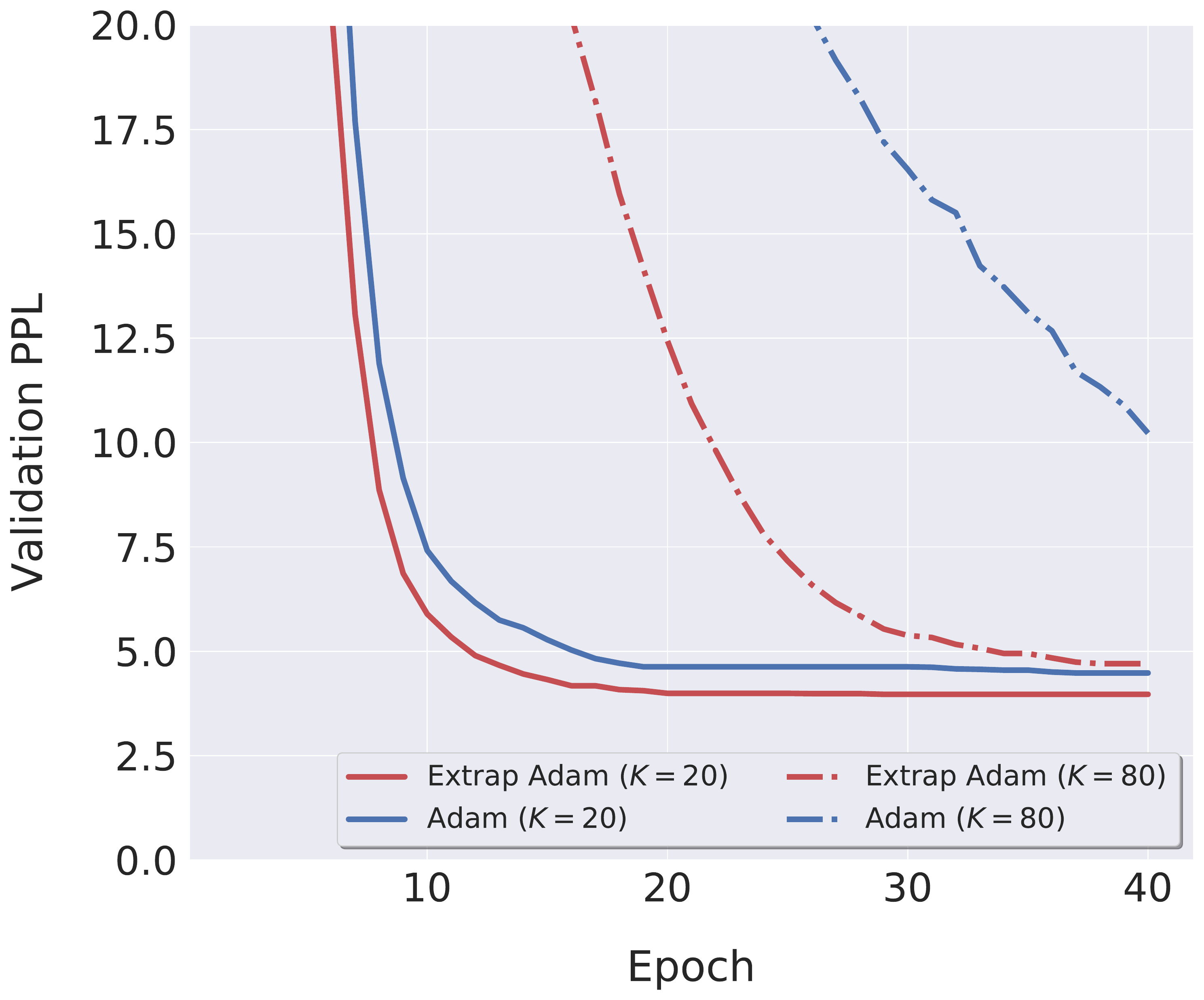}
		\label{fig:multi30k_transformer_val_ppl}
	}
	\hfill
	\subfigure[\small Accuracy (validation).]{
		\includegraphics[width=0.215\textwidth,]{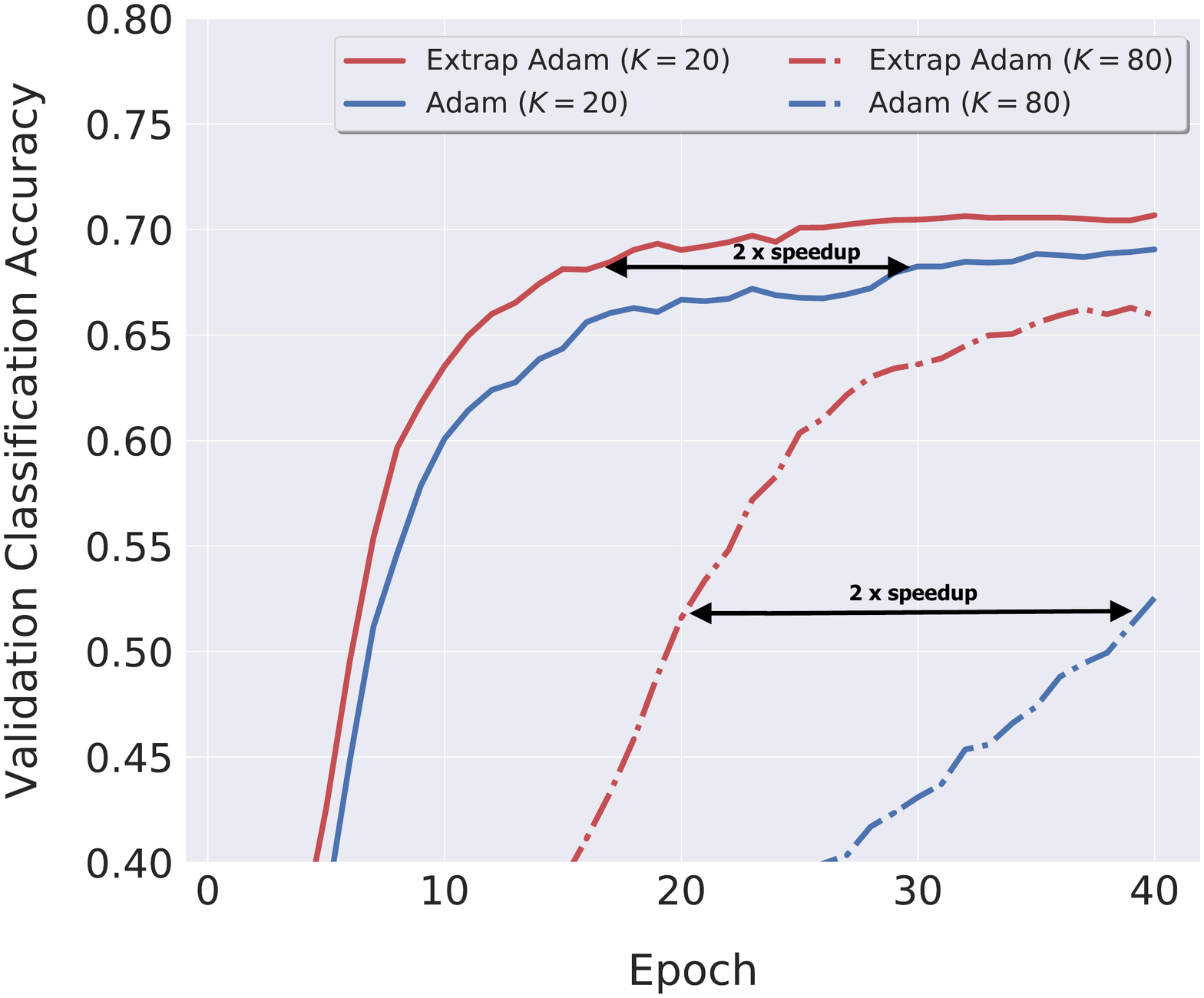}
		\label{fig:multi30k_transformer_val_top1}
	}
	\caption{\small
		Training \algoptadam on Transformer for Multi30k.
		The evaluations are performed on the validation dataset,
		for the training on $K\!=\!20$ and $K\!=\!80$ workers (corresponding to roughly $4\%$ and $16\%$ of the total training data).
		We use the standard inverse square root learning rate schedule~\citep{vaswani2017attention}
		and scale the warmup step based on the number of workers.
	}
	\vspace{-0.5em}
	\label{fig:multi30k_transformer}
\end{figure}

\begin{table*}[!h]
	\caption{\small
		The performance comparison of different methods for large-batch training on CIFAR-10 (mini-batch of size $8{,}192$ on $K\!=\!32$).
		Two neural architectures (ResNet-20 and VGG-11) are considered (with and without batch normalization).
		Unless mentioned otherwise each method will use the learning rate scaling and warmup in~\citet{goyal2017accurate} and LARS in~\citet{you2017large}.
		We finetune $\gamma$ and $\tilde{\gamma}$ for mini-batch SGD and SmoothOut (with $\hat{\gamma}$ additionally).
		For the results of extrapolated SGD,
		we reuse the optimal $\gamma$ and $\tilde{\gamma}$ tuned on mini-batch SGD;
		$\hat{\gamma}$ is fine-tuned for noise-based extrapolation variants.
		The results are averaged over three different seeds.
	}
	\label{tab:two_tasks_extreme_minibatch_size_for_all_methods}
	\centering
	\resizebox{1.\textwidth}{!}{%
		\begin{tabular}{ccc|ccccc}
			\toprule
			& \parbox{3cm}{\centering mini-batch SGD \\ (w/o LARS)}     & \parbox{3cm}{\centering \algopt \\ (w/o LARS)}    & \multicolumn{1}{c}{mini-batch SGD}    & \parbox{3cm}{\centering SmoothOut \\ \citep{wen2018smoothout}}    & \parbox{3cm}{\centering extrapolated SGD, \\ uniform noise}   & \parbox{3cm}{\centering extrapolated SGD, \\ stochastic noise}            & \algopt \\ \midrule
			ResNet-20 on CIFAR-10 & $90.00 \pm 0.48$ & $90.47 \pm 0.16$ & $91.36 \pm 0.19$ & $91.55 \pm 0.20$ & $91.53 \pm 0.25$ & $91.66 \pm 0.24$ & $91.72 \pm 0.11$ \\
			VGG-11 on CIFAR-10    & $73.09 \pm 9.35$ & $76.79 \pm 3.5$  & $86.64 \pm 0.10$ & $86.92 \pm 0.15$ & $87.00 \pm 0.31$ & $86.04 \pm 0.43$ & $87.00 \pm 0.26$ \\
			\bottomrule
		\end{tabular}%
	}
\end{table*}

\paragraph{Optimization v.s. Generalization benefits.}
To better understand when and how \algopt (and extrapolated SGD) help,
we switch our attention to the commonly accepted training practices:
using an initial large learning rate and decaying when the training plateaus.
The common beliefs\footnote{
	Recent work~\citep{li2019towards,you2020how} complement the understanding of this phenomenon from the learning of different patterns via different learning rate scales;
	we leave the connection to this aspect for future work.
}~\citep{lecun1991second,kleinberg2018alternative} argue that,
the initial large learning rate
accelerates the transition from random initialization to convergence regions (optimization),
and the decaying leads the convergence to local minimum (generalization).

Table~\ref{tab:two_tasks_extreme_minibatch_size_for_all_methods}
thoroughly evaluates the large-batch training performance (ResNet-20 on CIFAR-10 for mini-batch size $8{,}192$, with learning rate decay schedule) for all related methods.
Though the remarkable optimization improvements in Figure~\ref{fig:resnet20_cifar10_8k_learning_curves_different_methods_without_learning_rate_decay}
justify \emph{the effects of \algopt (and extrapolated SGD), i.e.\ smooth the ill-conditioned loss landscape},
decaying the learning rate diminishes our advantages, as illustrated in Table~\ref{tab:two_tasks_extreme_minibatch_size_for_all_methods} and Figure~\ref{fig:resnet20_cifar10_learning_curves_complete_1} in Appendix.
\emph{We argue that \algopt and its variants can smoothen the loss surface thus avoid some bad local minima regions, but it cannot guarantee to converge to a much better local minimum}\footnote{
	Additional experiments show that switching \algopt to mini-batch SGD after the first learning rate decay
	has similar test performance as using \algopt alone.
},
contrary to the statements in~\citep{wen2018smoothout,haruki2019gradient}.
They use (limited) empirical generalization metrics as the main arguments of the ``flatter minima'', ignoring the complex training dynamics;
their ignored SOTA optimization techniques (e.g. LARS in some experiments) might also result in the improper understandings.
Our empirical results provide insights:
\emph{the primary benefits of~\algopt are on the optimization phase (early training phase) and will diminish in the later training phase}.
It is also reflected in Figure~\ref{fig:resnet20_cifar10_smoothness} in terms of the smoothness.

\paragraph{Combining \algopt with post-local SGD.}
Given our new insights in the previous paragraph,
here we try to understand the importance of better optimization brought by~\algopt for the eventual generalization.
We consider the post-local SGD in~\citet{lin2020dont}, a known technique arguing to converge to flatter local minimum for better generalization.
This choice comes from the noticeable optimization benefits of~\algopt in the initial training phase
while post-local SGD targeting to converge to ``flatter minima'' for the later training phase.
Please refer to Algorithm~\ref{algo:our_main_algo_postlocal_variant} in Appendix~\ref{appendix:local_sgd} for training details.

\emph{We argue that \algopt biases the optimization trajectory towards a better-conditioned loss surface, where solutions with good generalization properties can be found more easily.}
Figure~\ref{fig:resnet20_cifar10_16k_extragsgd_with_postlocal} challenges the extreme large-batch training
(mini-batch of size $16{,}384$ on $64$ workers, accounting for $33\%$ of training data) for ResNet-20 on CIFAR-10 with different epoch budgets.
We can notice that the improper optimization in the critical initial learning phase of the mini-batch SGD results in a significant generalization gap,
which cannot be addressed by adding post-local SGD or increasing the number of training epochs alone.
\algopt, on the contrary, avoids the regions containing bad local minima, complementing the ability of post-local SGD for converging to better solutions.
Table~\ref{tab:resnet20_cifar_8k_extrapsgd_with_postlocal} in the Appendix~\ref{appendix:resnet20_cifar}
additionally reports a similar observation for mini-batch of size $8{,}192$ on more datasets.

\begin{figure}[!h]
	\centering
	\includegraphics[width=0.4\textwidth,]{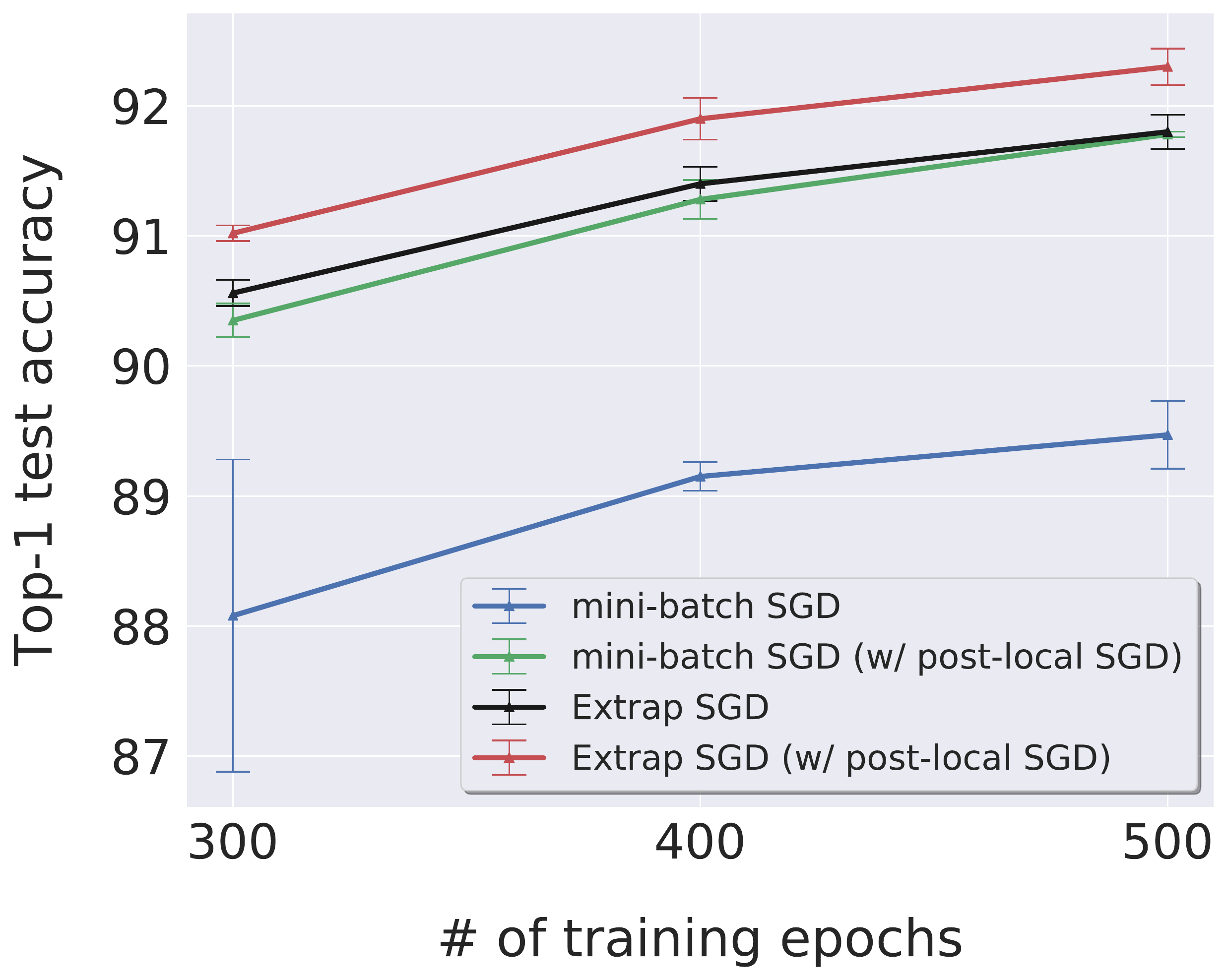}
	\caption{\small
		The test top-1 accuracy of integrating \algopt with post-local SGD~\citep{lin2020dont}.
		The performance of ResNet-20 on CIFAR-10/100 is evaluated with a global mini-batch size of $16,384$ ($33\%$ of the total training data with $64$ workers).
		By default we use the learning rate scaling and warmup in~\citet{goyal2017accurate} and LARS in~\citet{you2017large}.
		We individually finetune $\gamma$ and $\tilde{\gamma}$ for each base method;
		the local update step $H$ is tuned and set to $H\!=\!4$.
		The learning rate is decayed by $10$ when the model has accessed $50\%$ and $75\%$ of the total training samples.
		The results are averaged over three different seeds.
	}
	\label{fig:resnet20_cifar10_16k_extragsgd_with_postlocal}
\end{figure}

\subsection{Ablation study} \label{subsec:ablation_study}
\paragraph{\algopt for different local mini-batch sizes and number of workers.}
Table~\ref{tab:resnet20_cifar10_8k_different_combination_of_minibatch_size_and_n_workers} in the Appendix~\ref{appendix:resnet20_cifar}
evaluates how different combinations of the local mini-batch size and the number of workers will impact the performance,
for a given global mini-batch size ($8{,}192$ for ResNet-20 on CIFAR-10).
\emph{The benefits of \algopt can be further pronounced when increasing the number of workers,
	which is the common practice for large-batch training.
}
Similar observation can be found in Figure~\ref{fig:multi30k_transformer} for the increased number of workers (as well as the global mini-batch size).

\paragraph{Understanding the effect of momentum.}
Given the mixed effects of local extrapolation and momentum acceleration in~\algopt,
Figure~\ref{fig:resnet20_cifar10_impact_of_momentum_highlight} in Appendix~\ref{appendix:resnet20_cifar10_momentum_impact}
decouples these two factors for the training ResNet-20 on CIFAR-10 with mini-batch of size $8{,}192$.
We can witness that
(1)~\algopt can take advantages of both extragradient as well as the acceleration from the momentum,
and thus can always be applied for accelerated and stabilized distributed optimization;
(2) using extrapolation alone (no momentum) in~\algopt still outperforms mini-batch SGD with tuned momentum factor;
(3) tuning momentum factor for mini-batch SGD can marginally improve the optimization performance but cannot eliminate the optimization difficulty.

\section{Conclusion}
In this work, we adopt the idea of randomized smoothing to distributed training and propose \algopt to perform extrapolation with past local mini-batch gradients.
The idea further extends to a unified framework, covering multiple noise extrapolation variants.
We provide convergence guarantees for methods within this framework,
and empirically justify the remarkable benefits of our methods on image classification, language modeling and neural machine translation tasks.
We further investigate and understand the properties of our methods;
the algorithms smoothen the ill-conditioned loss landscape for faster optimization and biases the optimization trajectory to well-conditioned regions.
These insights further motivate us to combine our methods with post-local SGD for SOTA large-batch training.

\clearpage
\bibliography{icml2020}
\bibliographystyle{icml2020}

\clearpage
\appendix
\onecolumn

\part{Omitted proof for the convergence analysis} \label{part:convergence_proof}

\section{Nonconvex proof for Nesterov Momentum}  \label{sec:nonconvex_proof_nesterov_momentum}
One iterate of SGD with Nesterov momentum can be expressed as follows:
\begin{align} \label{eq:nesterov_minibatch_sgd}
	\begin{split}
		\xx_{t+\frac{1}{2}}     &= \xx_t + u \vv_t \,, \qquad
		\vv_{t+1}               = u \vv_t - \frac{\gamma}{BK} \sum_{k=1}^K \sum_{i \in \cI_t^k} \nabla f_i(\xx_{t + \frac{1}{2}}) \,, \qquad
		\xx_{t + 1}             = \xx_{t} + \vv_{t+1} \,,
	\end{split}
\end{align}
where $\vv_0 = 0$.
In the rest of Section~\ref{sec:nonconvex_proof_nesterov_momentum}, we use $\gg_t := \frac{1}{KB} \sum_{k=1}^K \sum_{i \in \cI_t^k} \nabla f_i(\xx_{t})$ to simplify notation.
With this simplication the iterate~\eqref{eq:nesterov_minibatch_sgd} can be expressed as follows:
\begin{align} \label{eq:simplified_nesterov_minibatch_sgd}
	\begin{split}
		\xx_{t+\frac{1}{2}}     &= \xx_t + u \vv_t \,, \qquad
		\vv_{t+1}               = u \vv_t - \gamma \gg_{t+\frac{1}{2}}\,, \qquad
		\xx_{t + 1}             = \xx_{t} + \vv_{t+1} \,,
	\end{split}
\end{align}

We follow the idea of~\citet{yu2019linear} and define an auxiliary sequence $ \bar{\yy}_t$ for~\eqref{eq:simplified_nesterov_minibatch_sgd}:
\begin{align} \label{eq:nesterov_momentum_auxiliary_sequence}
	\bar{\yy}_{t} =
	\left\{ \begin{array}{ll}
		\xx_{\frac{1}{2}} = \xx_{0}
		  & \textrm{if $t = 0$}    \\
		\frac{1}{ 1 - u } \xx_{ t + \frac{1}{2} } - \frac{ u }{ 1 - u } \xx_{ t - \frac{1}{2} }
		+ \frac{ \gamma u  }{ 1 - u }  \gg_{t - \frac{1}{2}}
		  & \textrm{if $t \geq 1$}
	\end{array} \right. \,.
\end{align}
We have the following two auxiliary lemmas:
\begin{lemma} \label{lemma:nesterov_minibatch_sgd_virtual_virtual_difference}
	Consider the sequence $\{ \bar{\yy} \}$ in~\eqref{eq:nesterov_momentum_auxiliary_sequence}
	and for all $t \geq 0$, we have $ \bar{\yy}_{t+1} - \bar{\yy}_{t} = - \frac{ \gamma }{ 1 - u } \gg_{t+\frac{1}{2}} $.
\end{lemma}
\begin{proof}
	For the case $t = 0$, we have
	\begin{align*}
		\begin{split}
			\bar{\yy}_{t+1} - \bar{\yy}_{t}
			&= \bar{\yy}_{1} - \bar{\yy}_0 \\
			&= \frac{1}{ 1 - u } \xx_{ \frac{3}{2} } - \frac{ u }{ 1 - u } \xx_{ \frac{1}{2} }
			+ \frac{ \gamma u  }{ 1 - u } \gg_{\frac{1}{2}}
			- \xx_{ \frac{1}{2} } \\
			&= \frac{1}{ 1 - u } ( \xx_{ \frac{3}{2} } - \xx_{ \frac{1}{2} } )
			+ \frac{ \gamma u  }{ 1 - u } \gg_{\frac{1}{2}}  \\
			&= \frac{1}{ 1 - u } ( u \vv_1 - \gamma \gg_{ \frac{1}{2}} )
			+ \frac{ \gamma u  }{ 1 - u } \gg_{\frac{1}{2}}  \\
			&= \frac{1}{ 1 - u } ( - u \gamma \gg_{\frac{1}{2}}  - \gamma \gg_{\frac{1}{2}} )
			+ \frac{ \gamma u  }{ 1 - u } \gg_{\frac{1}{2}}  \\
			& = - \frac{\gamma}{1 - u} \gg_{\frac{1}{2}}
		\end{split}
	\end{align*}

	For the case $t \geq 1$, we have
	\begin{align*}
		\begin{split}
			\bar{\yy}_{t+1} - \bar{\yy}_{t}
			&= \frac{ 1 }{ 1 - u } ( \xx_{t+\frac{3}{2}} - \xx_{t+\frac{1}{2}} ) - \frac{ u }{ 1 - u } ( \xx_{t+\frac{1}{2}} - \xx_{t-\frac{1}{2}} )
			+ \frac{ \gamma u }{ 1 - u } \left( \gg_{t+\frac{1}{2}} - \gg_{t-\frac{1}{2}} \right) \\
			&=  \frac{ 1 }{ 1 - u } ( u \vv_{t+1} - \gamma \gg_{t+\frac{1}{2}} )
			- \frac{ u }{ 1 - u } ( u \vv_{t} - \gamma \gg_{t-\frac{1}{2}} )
			+ \frac{ \gamma u }{ 1 - u } ( \gg_{t+\frac{1}{2}} - \gg_{t-\frac{1}{2}} ) \\
			&= - \frac{ \gamma }{ 1 - u } \gg_{t+\frac{1}{2}} + \frac{ u }{ 1 - u } ( \vv_{t+1} - u \vv_{t} + \gamma \gg_{t+\frac{1}{2}} )
			+ \frac{ \gamma u }{ 1 - u } ( \gg_{t-\frac{1}{2}} - \gg_{t-\frac{1}{2}} ) \\
			&= - \frac{ \gamma }{ 1 - u } \gg_{t+\frac{1}{2}} \,.
		\end{split}
	\end{align*}
\end{proof}

\begin{lemma} \label{lemma:nesterov_minibatch_sgd_virtual_real_difference}
	For all $t \geq 0$ and $\xx_{t + \frac{1}{2}}$ defined in~\eqref{eq:simplified_nesterov_minibatch_sgd}, we have
	\begin{align}
		\sum_{t=0}^{T-1} \norm{ \bar{\yy}_t - \xx_{t+\frac{1}{2}} }^2
		\leq \frac{ u^4 \gamma^2 }{ ( 1 - u )^4 } \sum_{t=0}^{T-1} \norm{ \gg_{t+\frac{1}{2}} }^2  \,.
	\end{align}
\end{lemma}
\begin{proof}
	Following the definition for $ t = 0 $, we have
	\begin{align*}
		\begin{split}
			\bar{\yy}_{0} - \xx_{\frac{1}{2}} = 0
		\end{split}
	\end{align*}

	For $ t \ge 1 $, we have
	\begin{align*}
		\begin{split}
			\bar{\yy}_{t} - \xx_{t+\frac{1}{2}}
			& = \frac{u}{1-u} ( \xx_{t+\frac{1}{2}} - \xx_{t-\frac{1}{2}} ) + \frac{ \gamma u }{ 1 - u } \gg_{t-\frac{1}{2}} \\
			& = \frac{u}{1-u} ( u \vv_{t} - \gamma \gg_{t-\frac{1}{2}} ) + \frac{ \gamma u }{ 1 - u } \gg_{t-\frac{1}{2}} \\
			& = \frac{ u^2 }{ 1 - u } \vv_{t} \,.
		\end{split}
	\end{align*}

	Letting $ s_t := \sum_{i=0}^{t-1} u^{t - i - 1} =  \frac{ 1 - u^t }{ 1 - u }$ and fixing $T \ge 2$, we have
	\begin{align*}
		\begin{split}
			\sum_{t=0}^{T-1} \norm{ \bar{\yy}_t - \xx_{t+\frac{1}{2}} }^2
			= \sum_{t=1}^{T-1} \norm{ \bar{\yy}_t - \xx_{t+\frac{1}{2}} }^2
			= \frac{u^4}{(1 - u)^2} \sum_{ t = 1 }^{ T - 1 } \norm{ \vv_{t} }^2
		\end{split}
	\end{align*}

	Further we bound $ \sum_{ t = 1 }^{ T - 1 } \norm{ \vv_{t} }^2 $ as follows:
	\begin{align*}
		\begin{split}
			\sum_{ t = 1 }^{ T - 1 } \norm{ \vv_{t} }^2
			&= \gamma^2 \sum_{t = 1}^{ T-1 } \norm{ \sum_{ i = 1 }^{ t } u^{ t - i } \gg_{ i - \frac{1}{2}} }^2
			= \gamma^2 \sum_{t = 1}^{ T-1 } s_t^2 \norm{ \sum_{ i = 0 }^{ t - 1 } \frac{u^{ t - i - 1 }}{ s_t } \gg_{ i + \frac{1}{2}} }^2 \\
			& \leq \gamma^2 \sum_{t = 1}^{ T-1 } s_t^2 \sum_{ i = 0 }^{ t - 1 } \frac{u^{ t - i - 1 }}{ s_t } \norm{ \gg_{ i + \frac{1}{2}} }^2
			= \gamma^2 \sum_{t = 1}^{ T-1 } s_t \sum_{ i = 0 }^{ t - 1 } u^{ t - i - 1 } \norm{ \gg_{ i + \frac{1}{2}} }^2 \\
			&\leq \frac{ \gamma^2 }{ 1 - u } \sum_{t = 1}^{ T - 1 } \sum_{ i = 0 }^{ t - 1 } u^{ t - i - 1 } \norm{ \gg_{ i + \frac{1}{2}} }^2
			= \frac{ \gamma^2 }{ 1 - u } \sum_{i = 0}^{ T - 2 } \norm{ \gg_{ i + \frac{1}{2}} }^2 \sum_{ t = i + 1 }^{ T - 1 } u^{ t - i - 1 } \\
			& \leq \frac{ \gamma^2 }{ 1 - u } \sum_{ i = 0 }^{ T - 1 } \norm{ \gg_{ i + \frac{1}{2}} }^2 \sum_{t=0}^\infty u^{t}
			= \frac{ \gamma^2 }{ (1 - u)^2 } \sum_{ t = 0 }^{ T - 1 } \norm{ \gg_{ t + \frac{1}{2}} }^2 \,,
		\end{split}
	\end{align*}
	which implies $\sum_{t=0}^{T-1} \norm{ \bar{\yy}_t - \xx_{t+\frac{1}{2}} }^2
		\leq \frac{ u^4 \gamma^2 }{ ( 1 - u )^4 } \sum_{t=0}^{T-1} \norm{ \gg_{t+\frac{1}{2}} }^2 $.
\end{proof}

\subsection{Main proof of Theorem~\ref{thm:non_convex_convergence_nesterov_minibatch}}
\begin{theorem}[Non-convex convergence of mini-batch SGD with Nesterov momentum]
	Under Assumption~\ref{assumption:unbiase},~\ref{assumption:boundedvariance}, and~\ref{assumption:smoothness},
	for the update rule of mini-batch SGD with Nesterov momentum~\eqref{eq:simplified_nesterov_minibatch_sgd}.
	we can show that,
	under the condition of $\gamma \leq \frac{ 2 ( 1 - u )^2 }{ L ( u^3 + 1 ) }$,
	\begin{align*}
		\begin{split}
			\Eb{ \frac{1}{T} \sum_{t=0}^{T-1} \norm{ \nabla f ( \xx_{t + \frac{1}{2}} ) }^2 }
			&\leq
			\frac{1}{ 1 - \frac{ L \gamma ( u^3 + 1 ) }{ 2 ( 1 - u)^2 } }
			\left(
			\frac{1}{T \frac{\gamma}{1-u} } \Eb{ f ( \xx_{0} ) - f^\star }
			+ \frac{ \gamma L }{ 2 (1 - u)^2 } \frac{ \sigma^2 }{ BK }
			\right) \,.
		\end{split}
	\end{align*}
\end{theorem}

\begin{proof}

	From the standard smoothness condition we have, $\forall t \ge 0$,
	\begin{align*}
		\begin{split}
			\Eb{ f ( \bar{\yy}_{t+1} ) - f ( \bar{\yy}_{t} ) }
			&\leq \Eb{ \langle \nabla f ( \bar{\yy}_t ) ,\; \bar{\yy}_{t+1} - \bar{\yy}_t \rangle + \frac{L}{2} \norm{ \bar{\yy}_{t+1} - \bar{\yy}_{t} }^2 } \\
			&\overset{(a)}{=} \Eb{ -\frac{ \gamma }{ 1 - u }
				\left\langle \nabla f ( \bar{\yy}_t ) \,, \gg_{t+\frac{1}{2}} \right\rangle
				+ \frac{L}{2} \norm{ \frac{ \gamma }{ 1 - u } \gg_{t+\frac{1}{2}} }^2 }  \\
			&= \Eb{ -\frac{ \gamma }{ 1 - u }
				\left\langle \nabla f ( \xx_{t + \frac{1}{2}} ) + \nabla f ( \bar{\yy}_t ) - \nabla f ( \xx_{t + \frac{1}{2}} ) \,, \gg_{t+\frac{1}{2}} \right\rangle
				+ \frac{L}{2} \norm{ \frac{ \gamma }{ 1 - u } \gg_{t+\frac{1}{2}} }^2 }  \\
			&= \Eb{ -\frac{ \gamma }{ 1 - u } \norm{ \nabla f ( \xx_{t + \frac{1}{2}} ) }^2
				- \frac{ \gamma }{ 1 - u } \left\langle \nabla f ( \bar{\yy}_t) - \nabla f ( \xx_{t+\frac{1}{2}} ) \,, \nabla f ( \xx_{t+\frac{1}{2}} ) \right\rangle
				+ \frac{L}{2} \norm{ \frac{ \gamma }{ 1 - u } \gg_{t+\frac{1}{2}} }^2 } \,,
		\end{split}
	\end{align*}
	where we use Lemma~\ref{lemma:nesterov_minibatch_sgd_virtual_virtual_difference} for $(a)$.

	We can note that
	\begin{align*}
		\begin{split}
			&- \frac{ \gamma }{ 1 - u } \left\langle \nabla f ( \bar{\yy}_t) - \nabla f ( \xx_{t+\frac{1}{2}} ) \,, \nabla f ( \xx_{t+\frac{1}{2}} ) \right\rangle \\
			&\qquad = \left\langle - \frac{ \sqrt{1 - u} }{\sqrt{L} u^{3/2} } \left( \nabla f ( \bar{\yy}_t) - \nabla f ( \xx_{t+\frac{1}{2}} ) \right) \,,
			\frac{ \gamma \sqrt{L} u^{3/2} }{ (1 - u)^{3/2} } \nabla f ( \xx_{t+\frac{1}{2}} ) \right\rangle \\
			&\qquad \overset{(b)}{\leq} \frac{ 1 - u }{2 L u^{3} } \norm{ \nabla f ( \bar{\yy}_t) - \nabla f ( \xx_{t+\frac{1}{2}} ) }^2 +
			\frac{ \gamma^2 L u^{3} }{ 2 (1 - u)^{3} } \norm{ \nabla f ( \xx_{t+\frac{1}{2}} ) }^2 \,,
		\end{split}
	\end{align*}
	where $(b)$ follows the basic inequality $\langle \aa, \bb \rangle \leq \frac{1}{2} \norm{\aa}^2 + \frac{1}{2} \norm{\bb}^2$.

	Thus, we have
	\begin{align*}
		\begin{split}
			&\Eb{ f ( \bar{\yy}_{t+1} ) - f ( \bar{\yy}_{t} ) } \\
			&\leq \Eb{
				-\frac{ \gamma }{ 1 - u } \norm{ \nabla f ( \xx_{t + \frac{1}{2}} ) }^2 + \frac{ 1 - u }{2 L u^{3} } \norm{ \nabla f ( \bar{\yy}_t) - \nabla f ( \xx_{t+\frac{1}{2}} ) }^2
				+ \frac{ \gamma^2 L u^3 }{ 2 ( 1 - u )^3 } \norm{ \nabla f ( \xx_{t+\frac{1}{2}} ) }^2
				+ \frac{ \gamma^2 L }{ 2 (1 - u)^2 } \norm{ \gg_{t+\frac{1}{2}} }^2 } \\
			&\leq \Eb{
				-\frac{ \gamma }{ 1 - u } \norm{ \nabla f ( \xx_{t + \frac{1}{2}} ) }^2 + \frac{ ( 1 - u ) L }{ 2 u^3 } \norm{ \bar{\yy}_t - \xx_{t+\frac{1}{2}}}^2
				+ \frac{ \gamma^2 L u^3 }{ 2 ( 1 - u )^3 } \norm{ \nabla f ( \xx_{t+\frac{1}{2}} ) }^2
				+ \frac{ \gamma^2 L }{ 2 (1 - u)^2 } \norm{ \gg_{t+\frac{1}{2}} }^2 } \\
			&= \Eb{ ( -\frac{ \gamma }{ 1 - u } + \frac{ \gamma^2 L u^3 }{ 2 ( 1 - u )^3 } ) \norm{ \nabla f ( \xx_{t + \frac{1}{2}} ) }^2 + \frac{ ( 1 - u ) L }{ 2 u^3 } \norm{ \bar{\yy}_t - \xx_{t+\frac{1}{2}}}^2
				+ \frac{ \gamma^2 L }{ 2 (1 - u)^2 } \norm{ \gg_{t+\frac{1}{2}} }^2 } \,.
		\end{split}
	\end{align*}

	Taking sum over t and averaging by $ \frac{1}{T} $ yields:
	\begin{align*}
		\begin{split}
			&\frac{1}{T} \sum_{t=0}^{T-1} \Eb{ f ( \bar{\yy}_{t+1} ) - f( \bar{\yy}_{t} ) }
			= \frac{1}{T} \Eb{ f ( \bar{\yy}_{T} ) - f( \bar{\yy}_{0} ) }  \\
			&\leq \Eb{ \left( -\frac{ \gamma }{ 1 - u } + \frac{ L \gamma^2 u^3 }{ 2 ( 1 - u )^3 } \right) \frac{1}{T} \sum_{t=0}^{T-1} \norm{ \nabla f ( \xx_{t + \frac{1}{2}} ) }^2 + \frac{ ( 1 - u ) L }{ 2 u^3 } \frac{1}{T} \sum_{t=0}^{T-1} \norm{ \bar{\yy}_t - \xx_{t+\frac{1}{2}}}^2
				+ \frac{ \gamma^2 L }{ 2 (1 - u)^2 } \frac{1}{T} \sum_{t=0}^{T-1} \norm{ \gg_{t+\frac{1}{2}} }^2 } \\
			&\overset{(c)}{\leq} \left( -\frac{ \gamma }{ 1 - u }
			+ \frac{ L \gamma^2 u^3 }{ 2 ( 1 - u )^3 } \right) \frac{1}{T} \sum_{t=0}^{T-1} \norm{ \nabla f ( \xx_{t + \frac{1}{2}} ) }^2
			+ \Eb{ \frac{ ( 1 - u ) L \gamma^2 }{ 2 u^3 } \frac{u^4}{(1-u)^4} \frac{1}{T} \sum_{t=0}^{T-1} \norm{ \gg_{t+\frac{1}{2}} }^2 } \\
			& \qquad + \Eb{ \frac{ \gamma^2 L }{ 2 (1 - u)^2 } \frac{1}{T} \sum_{t=0}^{T-1} \norm{ \gg_{t+\frac{1}{2}} }^2 } \\
			& \overset{(d)}{\leq} \left( -\frac{ \gamma }{ 1 - u } + \frac{ L \gamma^2 u^3 }{ 2 ( 1 - u )^3 } \right)
			\frac{1}{T} \sum_{t=0}^{T-1} \norm{ \nabla f ( \xx_{t + \frac{1}{2}} ) }^2
			+ \left( \frac{ \gamma^2 L }{ 2 (1 - u)^2 } + \frac{ L u \gamma^2 }{ 2 ( 1 - u )^3 } \right)
			\left( \frac{1}{T} \sum_{t=0}^{T-1} \norm{ \nabla f ( \xx_{t + \frac{1}{2}} ) }^2 + \frac{ \sigma^2 }{ BK } \right) \\
			& = \left( -\frac{ \gamma }{ 1 - u } + \frac{ L \gamma^2 u^3 }{ 2 ( 1 - u )^3 } + \frac{ \gamma^2 L }{ 2 (1 - u)^2 } + \frac{ L u \gamma^2 }{ 2 ( 1 - u )^3 } \right)
			\frac{1}{T} \sum_{t=0}^{T-1} \norm{ \nabla f ( \xx_{t + \frac{1}{2}} ) }^2
			+ \left( \frac{ \gamma^2 L }{ 2 (1 - u)^2 } + \frac{ L u \gamma^2 }{ 2 ( 1 - u )^3 } \right) \frac{ \sigma^2 }{ BK } \,,
		\end{split}
	\end{align*}
	where the inequality $(c)$ comes from Lemma~\ref{lemma:nesterov_minibatch_sgd_virtual_real_difference}.
	For (d), by using $ \Eb{ \norm{X}^2 } = \var{X} + \norm{ \Eb{X} }^2 $ and the fact
	that $\var{\sum_i X_i} = \sum_i \var{X_i}$ if $X_i$'s are independent , we get
	\begin{align*}
		\begin{split}
			\Eb{ \norm{\gg_{ t + \frac{1}{2} } }^2 }
			& = \Eb{ \norm{ \frac{1}{BK} \sum_{k=1}^K \sum_{i \in \cI_t^k} \gg_{ t + \frac{1}{2}, i } }^2 } \\
			& = \frac{1}{B^2 K^2} \sum_{k=1}^K \sum_{i \in \cI_t^k} \var{ \gg_{t + \frac{1}{2}, i} }
			+ \norm{ \Eb{ \frac{1}{BK} \sum_{k=1}^K \sum_{i \in \cI_t^k} \gg_{t + \frac{1}{2}, i} } }^2_2
			\leq \frac{1}{BK} \sigma^2 + \norm{ \nabla f (\xx_{t + \frac{1}{2}}) }^2 \,.
		\end{split}
	\end{align*}

	Note that the coefficient of the gradient norm of
	$
		-\frac{ \gamma }{ 1 - u } + \frac{ L \gamma^2 u^3 }{ 2 ( 1 - u )^3 } + \frac{ \gamma^2 L }{ 2 (1 - u)^2 } + \frac{ L u \gamma^2 }{ 2 ( 1 - u )^3 }
	$
	can be simplified as
	\begin{align*}
		\begin{split}
			\frac{\gamma}{1-u} \left( -1 + \frac{ L \gamma u^3 }{ 2 ( 1 - u )^2 } + \frac{ L \gamma }{ 2 ( 1 - u ) } + \frac{ L u \gamma }{ 2 ( 1 - u )^2 } \right)
			= \frac{\gamma}{1-u} \left( -1 + \frac{ L \gamma ( u^3 + 1 ) }{ 2 ( 1 - u)^2 } \right) \,.
		\end{split}
	\end{align*}

	By rearranging, we have
	\begin{align*}
		\begin{split}
			\frac{1}{T} \sum_{t=0}^{T-1} \norm{ \nabla f ( \xx_{t + \frac{1}{2}} ) }^2
			&\leq \frac{1}{T \frac{\gamma}{1-u} \left( 1 - \frac{ L \gamma ( u^3 + 1 ) }{ 2 ( 1 - u)^2 } \right) } \Eb{ f ( \bar{\yy}_{0} ) - f( \bar{\yy}_{T} ) }
			+ \frac{ \left( \frac{ \gamma^2 L }{ 2 (1 - u)^2 } + \frac{ L u \gamma^2 }{ 2 ( 1 - u )^3 } \right) }{
				\frac{\gamma}{1-u} \left( 1 - \frac{ L \gamma ( u^3 + 1 ) }{ 2 ( 1 - u)^2 } \right) }
			\frac{ \sigma^2 }{ BK } \\
			&\leq \frac{1}{T \frac{\gamma}{1-u} \left( 1 - \frac{ L \gamma ( u^3 + 1 ) }{ 2 ( 1 - u)^2 } \right) } \Eb{ f ( \xx_{0} ) - f^\star }
			+ \frac{ \left( \frac{ \gamma L }{ 2 (1 - u) } + \frac{ L u \gamma }{ 2 ( 1 - u )^2 } \right) }{ \left( 1 - \frac{ L \gamma ( u^3 + 1 ) }{ 2 ( 1 - u)^2 } \right) }
			\frac{ \sigma^2 }{ BK } \\
			&\leq
			\frac{1}{ 1 - \frac{ L \gamma ( u^3 + 1 ) }{ 2 ( 1 - u)^2 } }
			\left(
			\frac{1}{T \frac{\gamma}{1-u} } \Eb{ f ( \xx_{0} ) - f^\star }
			+ \frac{ \gamma L }{ 2 (1 - u)^2 } \frac{ \sigma^2 }{ BK }
			\right) \,,
		\end{split}
	\end{align*}
	where we have to make the overall coefficient of the RHS positive,
	i.e., $\gamma \leq \frac{ 2 ( 1 - u )^2 }{ L ( u^3 + 1 ) }$.
\end{proof}

\begin{lemma}[Lemma 13 of~\citet{stich2019error}] \label{lemma:bound}
	For every non-negative sequence $\{ r_t \}_{t \geq 0}$ and any parameters $d \geq 0, c \geq 0, T \geq 0$,
	there exists a constant $\gamma \leq \frac{1}{d}$, such that for constant stepsizes $\{  \gamma_t = \gamma \}_{t \geq 0}$ it holds
	\begin{align*}
		\textstyle
		\Psi_T := \frac{1}{T+1} \sum_{t=0}^T \left( \frac{r_t}{\gamma_t} - \frac{r_{t+1}}{\gamma_t} + c \gamma_t \right)
		\leq \frac{d r_0}{\gamma (T+1)} + c \gamma \,.
	\end{align*}
\end{lemma}

\begin{corollary} \label{corollary:critical_minibatch_size_for_nesterov_momentum}
	Consider the Theorem~\ref{thm:non_convex_convergence_nesterov_minibatch}
	and Lemma~\ref{lemma:bound}
	with $\gamma \leq \frac{ 2 ( 1 - u )^2 }{ L ( u^3 + 1 ) }$,
	we can rewrite the convergence rate of Theorem~\ref{thm:non_convex_convergence_nesterov_minibatch} as
	\begin{talign*}
		\begin{split}
			\E{ \frac{1}{T} \sum_{t=0}^{T-1} \norm{ \nabla f ( \xx_{t + \frac{1}{2}} ) }^2 }
			= \cO \left(
			\frac{L r_0 (u^3 + 1)}{T (1-u)} + \sqrt{ \frac{2 L r_0 \sigma^2}{KBT (1 - u)} }
			\right) \,.
		\end{split}
	\end{talign*}
\end{corollary}

\begin{proof}[Proof of Corollary~\ref{corollary:critical_minibatch_size_for_nesterov_momentum}]
	We first simplify the notations and constraints in Theorem~\ref{thm:non_convex_convergence_nesterov_minibatch} by considering
	$\Psi_T := \frac{1}{T \gamma } r_0 + \frac{ \gamma L \sigma^2 }{ KB }$,
	where $r_0 := f ( \xx_{0} ) - f^\star$ and $\gamma \leq \frac{1}{L}$.
	Following the techniques in the proof of the Lemma 13 in~\citet{stich2019error},
	we consider two cases:
	(1) if $\frac{ r_0 K B }{L \sigma^2 T } \leq \frac{1}{L^2}$,
	then we choose the stepsize $\gamma = \sqrt{ \frac{ r_0 K B }{L \sigma^2 T } }$
	and get $\Psi_T \leq 2 \sqrt{ \frac{L r_0 \sigma^2}{KBT} }$;
	(2) if $\frac{ r_0 K B }{L \sigma^2 T } > \frac{1}{L^2}$,
	then we choose $\gamma = \frac{1}{L}$ and get
	$\Psi_T \leq \frac{L r_0}{T} + \frac{\sigma^2}{ KB } \leq \frac{L r_0}{T} + \frac{L r_0}{ T } $.
	These two bounds together show that $\Psi_T \leq 2 \sqrt{ \frac{L r_0 \sigma^2}{KBT} } + 2 \frac{L r_0}{T}$.

	We then evaluate the exact case of Theorem~\ref{thm:non_convex_convergence_nesterov_minibatch} by considering
	$\Psi_T' := \frac{1}{T \frac{\gamma}{1-u} } r_0 + \frac{ \gamma L }{ 2 (1 - u)^2 } \frac{ \sigma^2 }{ KB }$.
	Similarly, we have the potential $\gamma$ to minimize $\Psi_T'$,
	where $\gamma = \sqrt{ \frac{ 2 r_0 K B (1 - u)^3 }{L \sigma^2 T } }$
	and $\gamma = \frac{ 2 ( 1 - u )^2 }{ L ( u^3 + 1 ) }$.
	Thus, by considering two cases as in previous paragraph, we have
	$\Psi_T' \leq \frac{L r_0 (u^3 + 1)}{T (1-u)} + \sqrt{ \frac{2 L r_0 \sigma^2}{KBT (1 - u)} }$.
\end{proof}

\section{Nonconvex proof for \algopt} \label{sec:extrap_sgd_convergence_proof}
In the main text we propose to reuse past local gradients (with mini-batch size $B$) for extrapolation.
For the convergence analysis illustrated in this section, we consider a more general extrapolation framework.

Recall that a general form of~\algopt can be defined as:
\begin{align} \label{eq:general_form_algopt}
	\begin{split}
		\xx_{t+\frac{1}{4}}     = \xx_{t} + u \vv_t \,,
		\xx_{t+\frac{1}{2}}^k   = \xx_{t+\frac{1}{4}} -  \frac{ \hat{\gamma} }{ b } \sum_{i \in \hat{\cI}_t^k} \xiv_{t, i}^k  \,,
		\vv_{t+1}               = u \vv_t - \frac{\gamma}{BK} \sum_{k=1}^K \sum_{i \in \cI_t^k} \nabla f_i(\xx^k_{t + \frac{1}{2}}) \,,
		\xx_{t + 1}             = \xx_{t} + \vv_{t+1} \,,
	\end{split}
\end{align}
where $\vv_0 = \0$ and $\xiv_{0, i}^k = \0$.
The $\hat{\cI}_t^k \subseteq \cI_t^k$ indicates the possibility of using the samples from the subset of $\cI_t^k$,
and $b = \abs{ \hat{\cI}_t^k }$.
The $\xiv_{t, i}^k$ could be either fresh local gradient $\nabla f_i(\xx^k_{t})$,
or old local gradient $\nabla f_i(\xx^k_{t - \frac{1}{2}})$.
The later choice corresponds to Algorithm~\ref{algo:our_main_algo} demonstrated in the main paper.
Furthermore we could choose $ \frac{ 1 }{ b } \sum_{i \in \hat{\cI}_t^k} \xiv_{t, i}^k := \zeta_t^k $
as various types of i.i.d. noise e.g. Gaussian noise, uniform noise and stochastic noise as discussed in Section~\ref{subsec:unified_framework}.
For ease of exposition, we adopt the following notations:
\begin{align} \label{eq:extrap_notations}
	\bar{\xiv}_{t}              := \frac{1}{b K} \sum_{k=1}^K \sum_{i \in \hat{\cI}_t^k} \xiv_{t, i}^k \,, \qquad
	\gg_{t, i}^k                := \nabla f_{i}( \xx_{t}^k ) \,, \qquad
	\bar{\gg}_{t}               := \frac{1}{bK} \sum_{k=1}^K \sum_{i \in \hat{\cI}_t^k} \gg_{t, i}^k \,, \qquad
	\bar{\gg}_{t+\frac{1}{2}}   := \frac{1}{BK} \sum_{k=1}^K \sum_{i \in \cI_t^k} \gg_{t+\frac{1}{2}, i}^k \,.
\end{align}

We follow the idea of using virtual sequence~\citep{yu2019linear,stich2019error}
and define two auxiliary sequences for our own setup~\eqref{eq:general_form_algopt}:
\begin{align} \label{eq:extrap_averaged_model}
	\bar{\xx}_{t+\frac{1}{2}} := \frac{1}{K} \sum_{k=1}^K \xx^k_{t+\frac{1}{2}} \,,
\end{align}
and
\begin{align} \label{eq:extrap_auxiliary_sequence}
	\bar{\yy}_{t}
	  & :=
	\left\{ \begin{array}{ll}
		\bar{\xx}_{ \frac{1}{2} } = \xx_{ 0 }                                    & \textrm{if $t = 0$}    \\
		\frac{1}{ 1 - u } \bar{\xx}_{ t + \frac{1}{2} } - \frac{ u }{ 1 - u } \bar{\xx}_{ t - \frac{1}{2} }
		+ \frac{ \gamma u  }{ 1 - u } \bar{\gg}_{t - \frac{1}{2}}
		+ \frac{ \hat{ \gamma } }{ 1 - u } ( \bar{\xiv}_t - u \bar{\xiv}_{t-1} ) & \textrm{if $t \geq 1$}
	\end{array} \right. \,.
\end{align}

Following the definition of the virtual sequence in~\eqref{eq:extrap_averaged_model},
the update scheme in~\eqref{eq:general_form_algopt} can be rewritten as
\begin{align} \label{eq:rewrote_general_extrap_sgd}
	\begin{split}
		\bar{\xx}_{t+\frac{1}{2}}
		&= \frac{1}{K} \sum_{k=1}^K \xx_{t+\frac{1}{2}}^k
		= \frac{1}{K} \sum_{k=1}^K \left( \xx_{t} + u \vv_t -  \frac{ \hat{\gamma} }{ b } \sum_{i \in \hat{\cI}_{t, i}^k} \xiv_{t, i}^k \right)
		= \xx_{t} + u \vv_t - \frac{ \hat{\gamma} }{ b K } \sum_{k=1}^K \sum_{i \in \hat{\cI}_t^k} \xiv_{t, i}^k
		= \xx_{t} + u \vv_t - \hat{\gamma} \bar{\xiv}_t \,, \\
		\vv_{t+1}
		&= u \vv_t - \frac{\gamma}{BK} \sum_{k=1}^K \sum_{i \in \cI_t^k} \gg_{t+\frac{1}{2}, i}^k
		= u \vv_{t} - \gamma \bar{\gg}_{t+\frac{1}{2}} \,, \\
		\xx_{t + 1}
		&= \xx_{t} + \vv_{t+1} \,,
	\end{split}
\end{align}

We have the following three lemmas:
\begin{lemma} \label{lemma:virtual_iterate}
	Consider the sequence $\{ \bar{\yy} \}$ in~\eqref{eq:extrap_auxiliary_sequence} and for $t \geq 0$, we have
	\begin{align*}
		\bar{\yy}_{t+1} - \bar{\yy}_{t} = - \frac{ \gamma }{ 1 - u } \bar{\gg}_{t+\frac{1}{2}} \,.
	\end{align*}
\end{lemma}
\begin{proof}
	For the case $ t = 0 $, we have
	\begin{align*}
		\begin{split}
			\bar{\yy}_{t+1} - \bar{\yy}_{t}
			&= \bar{\yy}_{1} - \bar{\yy}_{0} \\
			&= \frac{ 1 }{ 1 - u } ( \bar{\xx}_{\frac{3}{2}} - \bar{\xx}_{\frac{1}{2}} )
			+ \frac{ \gamma u }{ 1 - u } \bar{\gg}_{\frac{1}{2}}
			+ \frac{\hat{\gamma}}{1-u} \bar{\xiv}_{1} \\
			&=  \frac{ 1 }{ 1 - u } ( u \vv_{1} - \gamma \bar{\gg}_{\frac{1}{2}} - \hat{\gamma} \bar{\xiv}_{1} )
			+ \frac{ \gamma u }{ 1 - u } \bar{\gg}_{\frac{1}{2}}
			+ \frac{ \hat{\gamma} }{ 1 - u } \bar{\xiv}_{1}  \\
			&=  \frac{ 1 }{ 1 - u } ( - u \gamma \bar{\gg}_{\frac{1}{2}} - \gamma \bar{\gg}_{\frac{1}{2}} - \hat{\gamma} \bar{\xiv}_{1} )
			+ \frac{ \gamma u }{ 1 - u } \bar{\gg}_{\frac{1}{2}}
			+ \frac{ \hat{\gamma} }{ 1 - u } \bar{\xiv}_{1}  \\
			&= - \frac{ \gamma }{ 1 - u } \bar{\gg}_{\frac{1}{2}} \,.
		\end{split}
	\end{align*}

	For the case $ t \ge 1 $ we have
	\begin{align*}
		\begin{split}
			\bar{\yy}_{t+1} - \bar{\yy}_{t}
			&= \frac{ 1 }{ 1 - u } ( \bar{\xx}_{t+\frac{3}{2}} - \bar{\xx}_{t+\frac{1}{2}} ) - \frac{ u }{ 1 - u } ( \bar{\xx}_{t+\frac{1}{2}} - \bar{\xx}_{t-\frac{1}{2}} )
			+ \frac{ \gamma u }{ 1 - u } ( \bar{\gg}_{t+\frac{1}{2}} - \bar{\gg}_{t-\frac{1}{2}} )
			+ \frac{\hat{\gamma}}{1-u} ( \bar{\xiv}_{t+1} - \bar{\xiv}_t - u \bar{\xiv}_{t} + u \bar{\xiv}_{t-1} ) \\
			&=  \frac{ 1 }{ 1 - u } ( u \vv_{t+1} - \gamma \bar{\gg}_{t+\frac{1}{2}} + \hat{\gamma} \bar{\xiv}_t - \hat{\gamma} \bar{\xiv}_{t+1} )
			- \frac{ u }{ 1 - u } ( u \vv_{t} - \gamma \bar{\gg}_{t-\frac{1}{2}} + \hat{\gamma} \bar{\xiv}_{t-1} - \hat{\gamma} \bar{\xiv}_{t} ) \\
			&\qquad + \frac{ \gamma u }{ 1 - u } ( \bar{\gg}_{t+\frac{1}{2}} - \bar{\gg}_{t-\frac{1}{2}} )
			+ \frac{ \hat{\gamma} }{ 1 - u } ( \bar{\xiv}_{t+1} - \bar{\xiv}_t - u \bar{\xiv}_{t} + u \bar{\xiv}_{t-1} ) \\
			&=  \frac{ 1 }{ 1 - u } ( u \vv_{t+1} - \gamma \bar{\gg}_{t+\frac{1}{2}} ) + \frac{ \hat{\gamma}  }{ 1 - u } ( \bar{\xiv}_t - \bar{\xiv}_{t+1} )
			- \frac{ u }{ 1 - u } ( u \vv_{t} - \gamma \bar{\gg}_{t-\frac{1}{2}} ) - \frac{ u \hat{\gamma} }{ 1 - u } ( \bar{\xiv}_{t-1} - \bar{\xiv}_{t} ) \\
			&\qquad + \frac{ \gamma u }{ 1 - u } ( \bar{\gg}_{t+\frac{1}{2}} - \bar{\gg}_{t-\frac{1}{2}} )
			+ \frac{ \hat{\gamma} }{ 1 - u } ( \bar{\xiv}_{t+1} - \bar{\xiv}_t ) - \frac{ u\hat{\gamma} }{ 1 - u } ( \bar{\xiv}_{t} - \bar{\xiv}_{t-1} ) \\
			&=  \frac{ 1 }{ 1 - u } ( u \vv_{t+1} - \gamma \bar{\gg}_{t+\frac{1}{2}} ) - \frac{ u }{ 1 - u } ( u \vv_{t} - \gamma \bar{\gg}_{t-\frac{1}{2}} )
			+ \frac{ \gamma u }{ 1 - u } ( \bar{\gg}_{t+\frac{1}{2}} - \bar{\gg}_{t-\frac{1}{2}} ) \\
			&=  - \frac{ \gamma }{ 1 - u } \bar{\gg}_{t+\frac{1}{2}} + \frac{ u }{ 1 - u } \vv_{t+1} - \frac{ u }{ 1 - u } ( u \vv_{t} - \gamma \bar{\gg}_{t+\frac{1}{2}} )
			+ \frac{ \gamma u }{ 1 - u } ( \bar{\gg}_{t-\frac{1}{2}} - \bar{\gg}_{t-\frac{1}{2}} ) \\
			&= - \frac{ \gamma }{ 1 - u } \bar{\gg}_{t+\frac{1}{2}} \,.
		\end{split}
	\end{align*}
\end{proof}

\begin{lemma} \label{lemma:y_x_dist}
	For $T \ge 2$ and $\bar{\xx}_{t + \frac{1}{2}}$ defined in~\eqref{eq:extrap_averaged_model}, we have
	\begin{align*}
		\Eb{ \sum_{t=0}^{T-1} \norm{ \bar{\yy}_t - \bar{\xx}_{t+\frac{1}{2}} }^2 }
		\leq (1 + \beta) \frac{ u^4 \gamma^2 }{ ( 1 - u )^4 } \Eb{ \sum_{t=0}^{T-1} \norm{ \bar{\gg}_{ t + \frac{1}{2}} }^2 }
		+ (1 + \frac{1}{\beta}) \hat{\gamma}^2 \Eb{ \sum_{t=0}^{T-1} \norm{ \bar{\xiv}_t }^2 } ,\, \forall \beta > 0 \,.
	\end{align*}

	If we reuse past local stochastic gradients with mini-batch size $B$, i.e.
	$ \bar{\xiv}_t =
		\begin{cases}
			\0                        & \mathrm{if} \; t = 0   \\
			\bar{\gg}_{t-\frac{1}{2}} & \mathrm{if} \; t \ge 1
		\end{cases}
		\, ,
	$
	by choosing $ \beta = 1 $ and $ \hat{\gamma} \leq \frac{u^2}{ (1 - u)^2 } \gamma $,
	we obtain the simplified expression:
	\begin{align*}
		\Eb{ \sum_{t=0}^{T-1} \norm{ \bar{\yy}_t - \bar{\xx}_{t+\frac{1}{2}} }^2 }
		\leq \frac{ 4 u^4 \gamma^2 }{ ( 1 - u )^4 } \Eb{ \sum_{t=0}^{T-1} \norm{ \bar{\gg}_{ t + \frac{1}{2}} }^2 }
		\leq \frac{ 4 u^4 \gamma^2 }{ ( 1 - u )^4 } \frac{1}{BK} T \sigma^2
		+ \frac{ 4 u^4 \gamma^2 }{ ( 1 - u )^4 } \sum_{t=0}^{T-1} \norm{ \frac{1}{K} \sum_{k=1}^K \nabla f (\xx^k_{t + \frac{1}{2}}) }_2^2 \,.
	\end{align*}
\end{lemma}

\begin{proof}
	The proof starts from
	\begin{align*}
		\begin{split}
			\bar{\yy}_{t} - \bar{\xx}_{t+\frac{1}{2}}
			& = \frac{u}{1-u} ( \bar{\xx}_{t+\frac{1}{2}} - \bar{\xx}_{t-\frac{1}{2}} ) + \frac{ \gamma u }{ 1 - u } \bar{\gg}_{t-\frac{1}{2}}
			+ \frac{ \hat{\gamma} }{ 1 - u } ( \bar{\xiv}_t - u \bar{\xiv}_{t-1}  )\\
			& = \frac{u}{1-u} ( u \vv_{t} - \gamma \bar{\gg}_{t-\frac{1}{2}} + \hat{\gamma} \bar{\xiv}_{t-1} - \hat{\gamma} \bar{\xiv}_{t} )
			+ \frac{ \gamma u }{ 1 - u } \bar{\gg}_{t-\frac{1}{2}} + \frac{ \hat{\gamma} }{ 1 - u } ( \bar{\xiv}_t - u \bar{\xiv}_{t-1} ) \\
			& = \frac{u^2}{1-u} \vv_{t} - \frac{u \gamma}{1-u} \bar{\gg}_{t-\frac{1}{2}} + \frac{u \hat{\gamma} }{1-u} \bar{\xiv}_{t-1} - \frac{u \hat{\gamma} }{1-u} \bar{\xiv}_{t}
			+ \frac{ \gamma u }{ 1 - u } \bar{\gg}_{t-\frac{1}{2}} + \frac{ \hat{\gamma} }{ 1 - u } \bar{\xiv}_t - \frac{ \hat{\gamma} u }{ 1 - u } \bar{\xiv}_{t-1} \\
			& = \frac{u^2}{1-u} \vv_{t} - \frac{u \hat{\gamma} }{1-u} \bar{\xiv}_{t} + \frac{ \hat{\gamma} }{ 1 - u } \bar{\xiv}_t \\
			& = \frac{ u^2 }{ 1 - u } \vv_{t} + \hat{\gamma} \bar{\xiv}_t \,.
		\end{split}
	\end{align*}

	Thus, we have
	$
		\norm{ \bar{\yy}_t - \bar{\xx}_{t+\frac{1}{2}} }^2
		= \norm{ \frac{ u^2 }{ 1 - u } \vv_{t} + \hat{\gamma} \bar{\xiv}_t }^2
		\leq ( 1 + \beta ) \frac{u^4}{(1-u)^2} \norm{ \vv_t }^2 + ( 1 + \frac{1}{\beta} ) \hat{\gamma}^2 \norm{ \bar{ \xiv }_t }^2
	$,
	where $\norm{\aa + \bb}^2 \leq (1 + \beta) \norm{\aa}^2 + (1 + \beta^{-1}) \norm{\bb}^2$ for $\beta > 0$.

	Summing over $\{ 0, \ldots, T-1 \}$, we then have
	\begin{align*}
		\sum_{t=0}^{T-1} \norm{ \bar{\yy}_t - \bar{\xx}_{t+\frac{1}{2}} }^2
		\leq ( 1 + \beta ) \frac{u^4}{(1-u)^2} \sum_{t=0}^{T-1} \norm{ \vv_t }^2 + ( 1 + \frac{1}{\beta} ) \hat{\gamma}^2 \sum_{t=0}^{T-1} \norm{ \bar{ \xiv }_t }^2 \,.
	\end{align*}

	We define
	$ s_t := \sum_{i=0}^{t-1} u^{t - i - 1} =  \frac{ 1 - u^t }{ 1 - u } \le \frac{1}{ 1 - u } $. Noticing $ \vv_0 = 0 $ and $ T \ge 2 $,
	we have
	\begin{align*}
		\begin{split}
			\sum_{t=0}^{T-1} \norm{ \vv_{t} }^2
			&= \sum_{ t = 1 }^{ T - 1 } \norm{ \vv_{t} }^2 \\
			&= \gamma^2 \sum_{t = 1}^{ T-1 } \norm{ \sum_{ i = 1 }^{ t } u^{ t - i } \bar{\gg}_{ i - \frac{1}{2}} }^2
			= \gamma^2 \sum_{t = 1}^{ T-1 } s_t^2 \norm{ \sum_{ i = 0 }^{ t - 1 } \frac{u^{ t - i - 1 }}{ s_t } \bar{\gg}_{ i + \frac{1}{2}} }^2 \\
			& \leq \gamma^2 \sum_{t = 1}^{ T-1 } s_t^2 \sum_{ i = 0 }^{ t - 1 } \frac{u^{ t - i - 1 }}{ s_t } \norm{ \bar{\gg}_{ i + \frac{1}{2}} }^2
			= \gamma^2 \sum_{t = 1}^{ T-1 } s_t \sum_{ i = 0 }^{ t - 1 } u^{ t - i - 1 } \norm{ \bar{\gg}_{ i + \frac{1}{2}} }^2 \\
			&\leq \frac{ \gamma^2 }{ 1 - u } \sum_{t = 1}^{ T - 1 } \sum_{ i = 0 }^{ t - 1 } u^{ t - i - 1 } \norm{ \bar{\gg}_{ i + \frac{1}{2}} }^2
			= \frac{ \gamma^2 }{ 1 - u } \sum_{i = 0}^{ T - 2 } \norm{ \bar{\gg}_{ i + \frac{1}{2}} }^2 \sum_{ t = i + 1 }^{ T - 1 } u^{ t - i - 1 } \\
			& \leq \frac{ \gamma^2 }{ 1 - u } \sum_{ i = 0 }^{ T - 1 } \norm{ \bar{\gg}_{ i + \frac{1}{2}} }^2 \sum_{t=0}^\infty u^{t}
			= \frac{ \gamma^2 }{ (1 - u)^2 } \sum_{ t = 0 }^{ T - 1 } \norm{ \bar{\gg}_{ t + \frac{1}{2}} }^2 \,,
		\end{split}
	\end{align*}
	which implies
	\begin{align*}
		\sum_{t=0}^{T-1}  \norm{ \yy_t - \xx_{t+\frac{1}{2}} }^2
		\leq (1 + \beta) \frac{ u^4 \gamma^2 }{ ( 1 - u )^4 } \sum_{t=0}^{T-1} \norm{ \bar{\gg}_{ t + \frac{1}{2}} }^2
		+ ( 1 + \frac{1}{\beta} ) \hat{\gamma}^2 \sum_{t=0}^{T-1} \norm{ \bar{ \xiv }_t }^2 \,.
	\end{align*}

	Further if we reuse the past local gradients with mini-batch $B$, i.e.
	$ \bar{\xiv}_t =
		\begin{cases}
			\0                        & \mathrm{if} \; t = 0   \\
			\bar{\gg}_{t-\frac{1}{2}} & \mathrm{if} \; t \ge 1
		\end{cases}
		\,
	$,
	and choose $ \beta = 1 $ and $\hat{\gamma} \leq \frac{ u^2 }{ ( 1 - u )^2 } \gamma$,
	we obtain
	\begin{align*}
		\sum_{t=0}^{T-1} \norm{ \yy_t - \xx_{t+\frac{1}{2}} }^2
		\leq \frac{ 4 u^4 \gamma^2 }{ ( 1 - u )^4 } \sum_{t=0}^{T-1} \norm{ \bar{\gg}_{ t + \frac{1}{2}} }^2  \,.
	\end{align*}

	By using the fact that $ \Eb{ \norm{X}^2 } = \var{X} + \norm{ \Eb{X} }^2  $ and
	that $\var{\sum_i X_i} = \sum_i \var{X_i}$ if $X_i$'s are independent , we get
	\begin{align*}
		\begin{split}
			\Eb{ \norm{\bar{\gg}_{ t + \frac{1}{2} } }^2 }
			&= \Eb{ \norm{ \frac{1}{BK}  \sum_{k=1}^K \sum_{i \in \cI_t^k} \gg_{ t + \frac{1}{2}, i }^k }^2 }
			= \frac{1}{B^2 K^2}  \sum_{k=1}^K \sum_{i \in \cI_t^k} \var{ \gg_{t + \frac{1}{2}, i}^k }
			+ \norm{ \Eb{ \frac{1}{BK}  \sum_{k=1}^K \sum_{i \in \cI_t^k} \gg_{t + \frac{1}{2}, i}^k } }^2_2 \\
			&\leq \frac{1}{BK} \sigma^2 + \norm{ \frac{1}{K} \sum_{k=1}^K \nabla f (\xx^k_{t + \frac{1}{2}}) }_2^2 \,,
		\end{split}
	\end{align*}
	where we rely on Assumption~\ref{assumption:unbiase},~\ref{assumption:boundedvariance}.
	Thus,
	\begin{align*}
		\Eb{ \sum_{t=0}^{T-1} \norm{ \bar{\yy}_t - \bar{\xx}_{t+\frac{1}{2}} }^2 }
		  & \leq \frac{ 4 u^4 \gamma^2 }{ ( 1 - u )^4 } \sum_{t=0}^{T-1}
		\left( \frac{1}{BK} \sigma^2 + \norm{ \frac{1}{K} \sum_{k=1}^K \nabla f (\xx^k_{t + \frac{1}{2}}) }_2^2 \right) \\
		  & = \frac{ 4 u^4 \gamma^2 }{ ( 1 - u )^4 } \frac{1}{BK} T \sigma^2
		+ \frac{ 4 u^4 \gamma^2 }{ ( 1 - u )^4 } \sum_{t=0}^{T-1} \norm{ \frac{1}{K} \sum_{k=1}^K \nabla f (\xx^k_{t + \frac{1}{2}}) }_2^2 \,.
	\end{align*}
\end{proof}

\begin{lemma} \label{lemma:x_x_dist}
	For all $t \geq 0$, $\xx_{t+\frac{1}{2}}^k$ defined in~\eqref{eq:general_form_algopt}
	and $\bar{\xx}_{t + \frac{1}{2}}$ in~\eqref{eq:rewrote_general_extrap_sgd}, we have
	\begin{align}
		\Eb{ \frac{1}{K} \sum_{k=1}^K \norm{ \bar{\xx}_{t+\frac{1}{2}} - \xx^k_{t+\frac{1}{2}} }^2 }
		\leq \Eb{ \hat{\gamma}^2 \norm{ \bar{\xiv}_t - \frac{1}{b} \sum_{i \in \hat{\cI}_t^k} \xiv_{t, i}^k }^2 }  \,,
	\end{align}
	where $b := \abs{ \hat{\cI}_{t}^k }$.

	If  we reuse past local gradients, i.e.
	$ \bar{\xiv}_t =
		\begin{cases}
			\0                        & \mathrm{if} \; t = 0   \\
			\bar{\gg}_{t-\frac{1}{2}} & \mathrm{if} \; t \ge 1
		\end{cases}
		\, ,
	$
	we have $\Eb{ \frac{1}{K} \sum_{k=1}^K \norm{ \bar{\xx}_{t+\frac{1}{2}} - \xx^k_{t+\frac{1}{2}} }^2 }  \leq \frac{4 \hat{\gamma}^2}{b} \sigma^2 $.

	Alternatively if we use i.i.d. noise, i.e.
	$ \frac{ 1 }{ b } \sum_{i \in \hat{\cI}_t^k} \xiv_{t, i}^k  =
		\begin{cases}
			\0           & \mathrm{if} \; t = 0   \\
			\zetav_{t}^k & \mathrm{if} \; t \ge 1
		\end{cases}
		\, ,
	$
	with $ \zetav_t^k $ being i.i.d., $ \Eb{ \zetav_t^k } = 0 $ and $ \Eb{ \norm{ \zetav^k_t }^2 } \leq \hat{\sigma}^2 $,
	we have $\Eb{ \frac{1}{K} \sum_{k=1}^K \norm{ \bar{\xx}_{t+\frac{1}{2}} - \xx^k_{t+\frac{1}{2}} }^2 } \leq 2 \hat{\gamma}^2 \hat{\sigma}^2$.

\end{lemma}

\begin{proof}
	By definition, for $ t = 0 $, $  \Eb{ \norm{ \bar{\xx}_{ \frac{1}{2} } - \xx^k_{ \frac{1}{2} } }^2 } = 0 $. \\
	For $ t \ge 1 $, we have the following
	\begin{align*}
		\begin{split}
			\Eb{ \norm{ \bar{\xx}_{t+\frac{1}{2}} - \xx^k_{t+\frac{1}{2}} }^2 }
			= \Eb{ \norm{ \xx_t + u \vv_t - \hat{\gamma} \bar{\xiv}_t - \left( \xx_t + u \vv_t - \frac{\hat{\gamma}}{b} \sum_{i \in \hat{\cI}_t^k} \xiv_{t, i}^k \right)}^2 }
			= \Eb{ \hat{\gamma}^2 \norm{ \bar{\xiv}_t - \frac{1}{b} \sum_{i \in \hat{\cI}_t^k} \xiv_{t, i}^k }^2 }
		\end{split}
	\end{align*}

	If we reuse past local gradients i.e. $ \xiv_{t, i}^k = \gg_{t-\frac{1}{2}}^k $,
	then we have
	\begin{align*}
		\begin{split}
			\Eb{ \hat{\gamma}^2 \norm{ \bar{\xiv}_t - \frac{1}{b} \sum_{i \in \hat{\cI}_t^k} \xiv_{t, i}^k }^2 }
			&= \Eb{ \hat{\gamma}^2 \norm{ \bar{\gg}_{t-\frac{1}{2}} - \nabla f ( \xx_{t-\frac{1}{2}} ) + \nabla f ( \xx_{t-\frac{1}{2}} ) - \gg_{t-\frac{1}{2}}^k } } \\
			&\leq 2 \hat{\gamma}^2 \Eb{ \left( \norm{ \bar{\gg}_{t-\frac{1}{2}} - \nabla f ( \xx_{t-\frac{1}{2}} ) }^2 + \norm{ \nabla f ( \xx_{t-\frac{1}{2}}) - \gg_{t-\frac{1}{2}}^k }^2 \right) } \\
			&\overset{(a)}{\leq} 2 \hat{\gamma}^2 ( \frac{1}{Kb} + \frac{1}{b} ) \sigma^2
			\leq \frac{4 \hat{\gamma}^2}{b} \sigma^2 \, ,
		\end{split}
	\end{align*}
	where (a) is from the Assumption~\ref{assumption:boundedvariance} and the independence between data samples.

	Alternatively if we use i.i.d. noise, i.e.
	$ \frac{ 1 }{ b } \sum_{i \in \hat{\cI}_t^k} \xiv_{t, i}^k  =
		\begin{cases}
			\0           & \mathrm{if} \; t = 0   \\
			\zetav_{t}^k & \mathrm{if} \; t \ge 1
		\end{cases}
		\, ,
	$
	with $ \zetav_t^k $ being i.i.d., $ \Eb{ \zetav_t^k } = 0 $ and $ \Eb{ \norm{ \zetav^k_t }^2 } \leq \hat{\sigma}^2 $,
	we have
	\begin{align*}
		\begin{split}
			\Eb{ \hat{\gamma}^2 \norm{ \bar{\xiv}_t - \frac{1}{b} \sum_{i \in \hat{\cI}_t^k} \xiv_{t, i}^k }^2 }
			&= \hat{\gamma}^2 \Eb{ \norm{ \zetav_t^k - \frac{1}{K} \sum_{k=1}^K \zetav_t^k }^2 } \\
			&= \hat{\gamma}^2 \Eb{ \norm{ \zetav_t^k }^2 + \norm{ \frac{1}{K} \sum_{k=1}^K \zetav_t^k }^2 - \frac{2}{K} \sum_{i = 1}^K \langle  \zetav_t^i, \, \zetav_t^k \rangle } \\
			&= \hat{\gamma}^2 \Eb{ \norm{ \zetav_t^k }^2 + \frac{1}{K^2} \sum_{k=1}^K \norm{ \zetav_t^k }^2 - \frac{2}{K} \norm{ \zetav_t^k }^2 } \\
			&\leq \hat{\gamma}^2 (\hat{\sigma}^2 + \frac{1}{K} \hat{\sigma}^2)
			\leq 2 \hat{\gamma}^2 \hat{\sigma}^2 \, .
		\end{split}
	\end{align*}

	Combining these yields the lemma.

\end{proof}

\subsection{Main proof of Theorem~\ref{thm:extrasgd_nonconvex_convergence}}
We duplicate the Theorem~\ref{thm:extrasgd_nonconvex_convergence} appearing in the main paper as the Theorem~\ref{thm:extrasgd_nonconvex_convergence_}.
\begin{theorem} \label{thm:extrasgd_nonconvex_convergence_}
	Consider the update rule of~\algopt with notations defined in \eqref{eq:extrap_notations},
	\begin{align*}
		\begin{split}
			\xx_{t+\frac{1}{4}}     = \xx_{t} + u \vv_t \,,
			\xx_{t+\frac{1}{2}}^k   = \xx_{t+\frac{1}{4}} -  \hat{\gamma} \gg_{t - \frac{1}{2}}^k \,,
			\vv_{t+1}               = u \vv_t - \gamma \bar{\gg}_{t+\frac{1}{2}} \,,
			\xx_{t + 1}             = \xx_{t} + \vv_{t+1} \,.
		\end{split}
	\end{align*}
	We define $\bar{\xx}_{t + \frac{1}{2}} := \frac{1}{K} \sum_{k=1} \xx^k_{t+\frac{1}{2}}$
	as well as the following virtual sequence
	\begin{align*}
		\bar{\yy}_{t}
		  & :=
		\left\{ \begin{array}{ll}
			\bar{\xx}_{ \frac{1}{2} } = \xx_{ 0 }                                                              & \textrm{if $t = 0$}    \\
			\frac{1}{ 1 - u } \bar{\xx}_{ t + \frac{1}{2} } - \frac{ u }{ 1 - u } \bar{\xx}_{ t - \frac{1}{2} }
			+ \frac{ \gamma u  }{ 1 - u } \bar{\gg}_{t-\frac{1}{2}}
			+ \frac{ \hat{ \gamma } }{ 1 - u } ( \bar{\gg}_{t - \frac{1}{2}} - u \bar{\gg}_{t - \frac{3}{2}} ) & \textrm{if $t \geq 1$}
		\end{array} \right. \,,
	\end{align*}
	where $ \bar{\gg}_{- \frac{1}{2}} = 0 $ by default. \\
	Under Assumption~\ref{assumption:unbiase}-~\ref{assumption:smoothness},
	the convergence rate of $\bar{\xx}_{t + \frac{1}{2}}$ for a non-convex function follows
	\begin{align*}
		\begin{split}
			\Eb{ \frac{1}{T} \sum_{t=0}^{T-1} \norm{ \nabla f ( \bar{\xx}_{t+\frac{1}{2}} ) }^2 }
			\leq \frac{2 (1 - u)}{\gamma T} \Eb{ f( \bar{\xx}_{0} ) - f^\star }
			+ \left( \frac{ 4 \hat{\gamma}^2 L^2 }{ B } + \frac{ \gamma L (1 + 3 u) }{ (1 - u)^2 B K } \right) \sigma^2 \,,
		\end{split}
	\end{align*}
	where $\hat{\gamma} \leq \frac{ u^2 }{ ( 1 - u )^2 } \gamma$ and $\gamma \leq \frac{ ( 1 - u )^2 }{ L ( 1 + 3 u + u^3 ) }$.
\end{theorem}

\begin{proof}
	We start from the standard smoothness inequality
	\begin{align*}
		\begin{split}
			\Eb{ f( \bar{\yy}_{t+1} ) - f( \bar{\yy}_{t} ) }
			&\leq \Eb{ \langle \nabla f ( \bar{\yy}_t ) ,\; \bar{\yy}_{t+1} - \bar{\yy}_t \rangle + \frac{L}{2} \norm{ \bar{\yy}_{t+1} - \bar{\yy}_{t} }^2 } \\
			&\overset{(1)}{=} \Eb{ \langle \nabla f ( \bar{\yy}_t ) ,\;  -\frac{ \gamma }{ 1 - u } \bar{\gg}_{ t + \frac{1}{2} } \rangle }
			+ \Eb{ \frac{ \gamma^2 L }{ 2 (1 - u)^2 } \norm{  \bar{\gg}_{ t + \frac{1}{2}} }^2 } \\
			&= \Eb{ \left\langle \nabla f ( \bar{\yy}_t ) - \nabla f ( \bar{\xx}_{t+\frac{1}{2}} ) + \nabla f ( \bar{\xx}_{t+\frac{1}{2}} ),\; -\frac{ \gamma }{ 1 - u } \frac{1}{K} \sum_{k=1}^K \nabla f ( \xx^k_{t + \frac{1}{2} } ) \right\rangle }
			+ \Eb{ \frac{ \gamma^2 L }{ 2 (1 - u)^2 } \norm{ \bar{\gg}_{ t + \frac{1}{2} } }^2 } \\
			&= \underbrace{ \Eb{ \left\langle \nabla f ( \bar{\yy}_t ) - \nabla f ( \bar{\xx}_{t+\frac{1}{2}} ) ,\; -\frac{ \gamma }{ 1 - u } \frac{1}{K} \sum_{k=1}^K \nabla f ( \xx^k_{t + \frac{1}{2} } ) \right\rangle } }_{(a)}
			+ \underbrace{ \Eb{ \left\langle \nabla f ( \bar{\xx}_{t+\frac{1}{2}} ) ,\; -\frac{ \gamma }{ 1 - u } \frac{1}{K} \sum_{k=1}^K \nabla f ( \xx^k_{t + \frac{1}{2} } ) \right\rangle } }_{(b)} \\
			&\qquad + \frac{ \gamma^2 L }{ 2 (1 - u)^2 } \underbrace{ \Eb{ \norm{\bar{\gg}_{ t + \frac{1}{2} } }^2 } }_{(c)} \,,
		\end{split}
	\end{align*}
	where we use Lemma~\ref{lemma:virtual_iterate} for (1).

	For (a), we have
	\begin{align*}
		\begin{split}
			&\Eb{ \left\langle \nabla f ( \bar{\yy}_t ) - \nabla f ( \bar{\xx}_{t+\frac{1}{2}} ) ,\;
			-\frac{ \gamma }{ 1 - u } \frac{1}{K} \sum_{k=1}^K \nabla f ( \xx^k_{t + \frac{1}{2} } ) \right\rangle } \\
			&\qquad \overset{(2)}{\leq}
			\Eb{ \frac{ 1 - u }{ 2 u^3 L } \norm{ \nabla f ( \bar{\yy}_t ) - \nabla f ( \bar{\xx}_{t+\frac{1}{2}} ) }^2 }
			+ \Eb{ \frac{ L \gamma^2 u^3 }{ 2 (1 - u)^3 } \norm{  \frac{1}{K} \sum_{k=1}^K \nabla f ( \xx^k_{t + \frac{1}{2} } ) }^2 } \\
			&\qquad \overset{(3)}{\leq} \Eb{ \frac{ (1 - u) L }{ 2 u^3 } \norm{ \bar{\yy}_t -  \bar{\xx}_{t+\frac{1}{2}} }^2 }
			+ \Eb{ \frac{ L \gamma^2 u^3 }{ 2 (1 - u)^3 } \norm{  \frac{1}{K} \sum_{k=1}^K \nabla f ( \xx^k_{t + \frac{1}{2} } ) }^2 } \,,
		\end{split}
	\end{align*}
	where the inequality (2) is due to the basic inequality $ \langle \aa, \bb \rangle \leq \frac{\beta}{2} \norm{\aa}^2 + \frac{1}{2\beta} \norm{\bb}^2 $
	with $ \beta = \frac{ 1 - u }{ L u^3 } $ and (3) is from Assumption~\ref{assumption:smoothness}.

	For (b), we have
	\begin{align*}
		\begin{split}
			&\Eb{
			\left\langle \nabla f ( \bar{\xx}_{t+\frac{1}{2}} ) ,\; -\frac{ \gamma }{ 1 - u } \frac{1}{K} \sum_{k=1}^K \nabla f ( \xx^k_{t + \frac{1}{2} } ) \right\rangle
			} \\
			&= -\frac{ \gamma }{ 1 - u } \Eb{
			\left\langle \nabla f ( \bar{\xx}_{t+\frac{1}{2}} ) ,\; \frac{1}{K} \sum_{k=1}^K \nabla f ( \xx^k_{t + \frac{1}{2} } ) \right\rangle
			} \\
			&\overset{(4)}{=} -\frac{ \gamma }{ 2 (1 - u ) } \Eb{
			\norm{\nabla f ( \bar{\xx}_{t+\frac{1}{2}} )}^2
			+ \norm{ \frac{1}{K} \sum_{k=1}^K \nabla f ( \xx^k_{t + \frac{1}{2} } ) }^2
			- \norm{ \nabla f ( \bar{\xx}_{t+\frac{1}{2}} ) - \frac{1}{K} \sum_{k=1}^K \nabla f ( \xx^k_{t + \frac{1}{2} } ) }^2
			} \\
			&\overset{(5)}{\leq} - \frac{\gamma}{2 (1 - u)} \Eb{
			\norm{\nabla f ( \bar{\xx}_{t+\frac{1}{2}} )}^2
			+ \norm{ \frac{1}{K} \sum_{k=1}^K \nabla f ( \xx^k_{t + \frac{1}{2} } ) }^2
			- \frac{L^2}{K} \Eb{ \sum_{k=1}^K \norm{ \bar{\xx}_{t+\frac{1}{2}} - \xx^k_{t + \frac{1}{2} } }^2 }
			} \\
			&\overset{(6)}{\leq} - \frac{ \gamma }{ 2 (1 - u) } \left(
			\Eb{ \norm{ \nabla f ( \bar{\xx}_{t+\frac{1}{2}} ) }^2 }
			+ \Eb{ \norm{ \frac{1}{K} \sum_{k=1}^K \nabla f ( \xx^k_{t + \frac{1}{2} } ) }^2 }
			- \frac{ 4 \hat{\gamma}^2 L^2 }{ B } \sigma^2
			\right) \,,
		\end{split}
	\end{align*}
	where the equality (4) comes from $\langle \aa, \bb \rangle = \frac{1}{2} \left( \norm{\aa}^2 + \norm{\bb}^2 - \norm{\aa - \bb}^2 \right)$,
	the inequality (5) is due to Assumption~\ref{assumption:smoothness}
	and (6) is from Lemma~\ref{lemma:x_x_dist}.

	For (c), by using $ \Eb{ \norm{X}^2 } = \var{X} + \norm{ \Eb{X} }^2 $ and the fact
	that $\var{\sum_i X_i} = \sum_i \var{X_i}$ if $X_i$'s are independent , we get
	\begin{align*}
		\begin{split}
			\Eb{ \norm{\bar{\gg}_{ t + \frac{1}{2} } }^2 }
			& = \Eb{ \norm{ \frac{1}{BK} \sum_{k=1}^K \sum_{i \in \cI_t^k} \gg_{ t + \frac{1}{2}, i }^k }^2 } \\
			& = \frac{1}{B^2 K^2} \sum_{k=1}^K \sum_{i \in \cI_t^k} \var{ \gg_{t + \frac{1}{2}, i}^k }
			+ \norm{ \Eb{ \frac{1}{BK} \sum_{k=1}^K \sum_{i \in \cI_t^k} \gg_{t + \frac{1}{2}, i}^k } }^2_2 \\
			&\overset{(7)}{\leq} \frac{1}{BK} \sigma^2 + \norm{ \frac{1}{K} \sum_{k=1}^K \nabla f (\xx^k_{t + \frac{1}{2}}) }_2^2 \,,
		\end{split}
	\end{align*}
	where (7) follows from Assumption~\ref{assumption:boundedvariance}.

	Thus the smoothness inequality becomes
	\begin{align*}
		\begin{split}
			\Eb{ f ( \bar{\yy}_{t+1} ) - f( \bar{\yy}_{t} ) }
			&\leq \frac{ (1 - u) L }{ 2 u^3 } \Eb{ \norm{ \bar{\yy}_t - \bar{\xx}_{t+\frac{1}{2}} }^2 }
			+ \left( \frac{ L \gamma^2 u^3 }{ 2 (1 - u)^3 } - \frac{ \gamma }{ 2 (1 - u) } \right)
			\Eb{ \norm{ \frac{1}{K} \sum_{k=1}^K \nabla f ( \xx^k_{t + \frac{1}{2} } ) }^2 } \\
			&\qquad - \frac{ \gamma }{ 2 (1 - u) }
			\Eb{ \norm{ \nabla f ( \bar{\xx}_{t+\frac{1}{2}} ) }^2 }
			+ \frac{ \gamma }{ 2 (1 - u) } \frac{ 4 \hat{\gamma}^2 L^2 }{ B } \sigma^2
			+ \frac{ \gamma^2 L }{ 2 (1 - u)^2 } \left( \frac{1}{BK} \sigma^2 + \Eb{ \norm{ \frac{1}{K} \sum_{k=1}^K \nabla f (\xx^k_{t + \frac{1}{2}}) }^2 } \right) \\
			&= \frac{ (1 - u) L }{ 2 u^3 } \Eb{ \norm{ \bar{\yy}_t - \bar{\xx}_{t+\frac{1}{2}} }^2 }
			+ \left( \frac{ L \gamma^2 u^3 }{ 2 (1 - u)^3 } - \frac{ \gamma }{ 2 (1 - u) } + \frac{ \gamma^2 L }{ 2 (1 - u)^2 } \right)
			\Eb{ \norm{ \frac{1}{K} \sum_{k=1}^K \nabla f ( \xx^k_{t + \frac{1}{2} } ) }^2 } \\
			&\qquad - \frac{ \gamma }{ 2 (1 - u) }
			\Eb{ \norm{ \nabla f ( \bar{\xx}_{t+\frac{1}{2}} ) }^2 }
			+ \left( \frac{ \gamma }{ 2 (1 - u) } \frac{ 4 \hat{\gamma}^2 L^2 }{ B } + \frac{ \gamma^2 L }{ 2 (1 - u)^2 } \frac{1}{BK} \right) \sigma^2 \,.
		\end{split}
	\end{align*}
	By rearranging, summing over $t$, and diving both sides by $ \frac{ \gamma }{ 2 ( 1 - u ) } $, we have
	\begin{align*}
		\begin{split}
			\Eb{ \sum_{t=0}^{T-1} \norm{ \nabla f ( \bar{\xx}_{t+\frac{1}{2}} ) }^2 }
			&\leq \frac{2 (1 - u)}{\gamma} \Eb{ f( \bar{\yy}_{0} ) - f ( \bar{\yy}_{T} ) }
			+ \frac{ (1 - u)^2 L }{ \gamma u^3 } \Eb{ \sum_{t=0}^{T-1} \norm{ \bar{\yy}_t -  \bar{\xx}_{t+\frac{1}{2}} }^2 } \\
			&\qquad + \left( \frac{ \gamma L u^3 + \gamma L (1 - u) }{(1-u)^2} - 1 \right)
			\sum_{t=0}^{T-1} \Eb{ \norm{ \frac{1}{K} \sum_{k=1}^K \nabla f ( \xx^k_{t + \frac{1}{2} } ) }^2 }
			+ \left( \frac{ 4 \hat{\gamma}^2 L^2 }{ B } + \frac{ \gamma L }{ 1 - u } \frac{1}{BK} \right) T \sigma^2 \,.
		\end{split}
	\end{align*}

	Through Lemma~\ref{lemma:y_x_dist}, we have
	\begin{align*}
		\Eb{ \sum_{t=0}^{T-1} \norm{ \bar{\yy}_t - \bar{\xx}_{t+\frac{1}{2}} }^2 }
		\leq \Eb{ \frac{ 4 u^4 \gamma^2 }{ ( 1 - u )^4 } \sum_{t=0}^{T-1} \norm{ \bar{\gg}_{ t + \frac{1}{2}} }^2 }
		\leq \frac{ 4 u^4 \gamma^2 }{ ( 1 - u )^4 } \frac{1}{BK} T \sigma^2
		+ \frac{ 4 u^4 \gamma^2 }{ ( 1 - u )^4 } \sum_{t=0}^{T-1} \Eb{ \norm{ \frac{1}{K} \sum_{k=1}^K \nabla f (\xx^k_{t + \frac{1}{2}}) }^2 } \,,
	\end{align*}
	thus,
	\begin{align*}
		\begin{split}
			\Eb{ \sum_{t=0}^{T-1} \norm{ \nabla f ( \bar{\xx}_{t+\frac{1}{2}} ) }^2 }
			&\leq \frac{2 (1 - u)}{\gamma} \Eb{ f( \bar{\yy}_{0} ) - f ( \bar{\yy}_{T} ) }
			+ \frac{ (1 - u)^2 L }{ \gamma u^3 } \left(
			\frac{ 4 u^4 \gamma^2 }{ ( 1 - u )^4 } \frac{1}{BK} T \sigma^2
			+ \frac{ 4 u^4 \gamma^2 }{ ( 1 - u )^4 } \sum_{t=0}^{T-1} \norm{ \frac{1}{K} \sum_{k=1}^K \nabla f (\xx^k_{t + \frac{1}{2}}) }_2^2
			\right) \\
			&\qquad + \left( \frac{ \gamma L u^3 + \gamma L (1 - u) }{(1-u)^2} - 1 \right)
			\sum_{t=0}^{T-1} \norm{ \frac{1}{K} \sum_{k=1}^K \nabla f ( \xx^k_{t + \frac{1}{2} } ) }^2
			+ \left( \frac{ 4 \hat{\gamma}^2 L^2 }{ B } + \frac{ \gamma L }{ 1 - u } \frac{1}{BK} \right) T \sigma^2 \\
			&\leq \frac{2 (1 - u)}{\gamma} \Eb{ f( \bar{\yy}_{0} ) - f ( \bar{\yy}_{T} ) } \\
			&\qquad + \left( \frac{ \gamma L u^3 + \gamma L (1 - u) }{(1-u)^2} + \frac{ 4 u \gamma L }{ ( 1 - u )^2 } - 1 \right)
			\sum_{t=0}^{T-1} \norm{ \frac{1}{K} \sum_{k=1}^K \nabla f ( \xx^k_{t + \frac{1}{2} } ) }^2 \\
			&\qquad+ \left( \frac{ 4 \hat{\gamma}^2 L^2 }{ B } + \frac{ \gamma L }{ 1 - u } \frac{1}{BK}
			+ \frac{ 4 u \gamma L }{ ( 1 - u )^2 } \frac{1}{BK} \right) T \sigma^2 \\
			&\leq \frac{2 (1 - u)}{\gamma} \Eb{ f( \bar{\yy}_{0} ) - f ( \bar{\yy}_{T} ) } \\
			&\qquad + \left( \frac{ \gamma L \left( u^3 + 3u + 1 \right) }{(1-u)^2} - 1 \right)
			\sum_{t=0}^{T-1} \norm{ \frac{1}{K} \sum_{k=1}^K \nabla f ( \xx^k_{t + \frac{1}{2} } ) }^2 \\
			&\qquad + \left( \frac{ 4 \hat{\gamma}^2 L^2 }{ B } + \frac{ \gamma L (1 + 3 u) }{ (1 - u)^2 B K }
			\right) T \sigma^2 \,.
		\end{split}
	\end{align*}

	Dividing both sides by $ T $, we have
	\begin{align*}
		\begin{split}
			&\Eb{ \frac{1}{T} \sum_{t=0}^{T-1} \norm{ \nabla f ( \bar{\xx}_{t+\frac{1}{2}} ) }^2 } \\
			&\leq \frac{2 (1 - u)}{\gamma T} \Eb{ f( \bar{\yy}_{0} ) - f ( \bar{\yy}_{T} ) }
			+ \frac{1}{T}
			\left( \frac{ \gamma L \left( u^3 + 3u + 1 \right) }{(1-u)^2} - 1 \right)
			\sum_{t=0}^{T-1} \norm{ \frac{1}{K} \sum_{k=1}^K \nabla f ( \xx^k_{t + \frac{1}{2} } ) }^2
			+ \left( \frac{ 4 \hat{\gamma}^2 L^2 }{ B } + \frac{ \gamma L (1 + 3 u) }{ (1 - u)^2 B K }
			\right) \sigma^2 \, .
		\end{split}
	\end{align*}
	Further we choose $ \gamma $ to ensure the coefficient of the gradient norm on the RHS nonpositive
	(i.e., $\frac{ L \gamma ( 1 + 3 u + u^3 ) }{ (1 - u)^2 } - 1 \leq 0 $,
	which implies $ \gamma \leq \frac{ ( 1 - u )^2 }{ L ( 1 + 3 u + u^3 ) } $.
	As a result, we can conclude that
	\begin{align*}
		\begin{split}
			\Eb{ \frac{1}{T} \sum_{t=0}^{T-1} \norm{ \nabla f ( \bar{\xx}_{t+\frac{1}{2}} ) }^2 }
			\leq \frac{2 (1 - u)}{\gamma T} \Eb{ f( \bar{\yy}_{0} ) - f ( \bar{\yy}_{T} ) }
			+ \left( \frac{ 4 \hat{\gamma}^2 L^2 }{ B } + \frac{ \gamma L (1 + 3 u) }{ (1 - u)^2 B K } \right) \sigma^2 \,.
		\end{split}
	\end{align*}
\end{proof}

\begin{theorem} \label{thm:extrasgd_nonconvex_convergence_noise_}
	Under the extrapolation framework,
	if we use i.i.d. noise, i.e.
	$ \frac{ 1 }{ b } \sum_{i \in \hat{\cI}_t^k} \xiv_{t, i}^k  =
		\begin{cases}
			\0           & \mathrm{if} \; t = 0   \\
			\zetav_{t}^k & \mathrm{if} \; t \ge 1
		\end{cases}
		\, ,
	$
	with $ \zetav_t^k $ being i.i.d., $ \Eb{ \zetav_t^k } = 0 $ and $ \Eb{ \norm{ \zetav^k_t }^2 } \leq \hat{\sigma}^2 $.
	Under the same assumption as in Theorem~\ref{thm:extrasgd_nonconvex_convergence},
	we have the following convergence rate for a non-convex function:
	\begin{align*}
		\begin{split}
			\Eb{ \frac{1}{T} \sum_{t=0}^{T-1} \norm{ \nabla f ( \bar{\xx}_{t+\frac{1}{2}} ) }^2 }
			\leq \frac{2 (1 - u)}{\gamma T} \Eb{ f( \bar{\yy}_{0} ) - f ( \bar{\yy}_{T} ) }
			+ \frac{ \gamma L ( 1 + u ) }{ (1 - u)^2 BK } \hat{\sigma}^2
			+ (L^2 + \frac{ (1 - u)^2 L }{ \gamma u^3 K })2 \hat{\gamma}^2 T \hat{\sigma}^2 \,,
		\end{split}
	\end{align*}
	where $ \gamma \leq \frac{ ( 1 - u )^2 }{ L ( 1 + u + u^3 ) } $.
\end{theorem}

\begin{proof}
	Following similar procedures in the proof of Theorem~\ref{thm:extrasgd_nonconvex_convergence_},
	we have
	\begin{align*}
		\begin{split}
			\Eb{ \sum_{t=0}^{T-1} \norm{ \nabla f ( \bar{\xx}_{t+\frac{1}{2}} ) }^2 }
			&\leq \frac{2 (1 - u)}{\gamma} \Eb{ f( \bar{\yy}_{0} ) - f ( \bar{\yy}_{T} ) }
			+ \frac{ (1 - u)^2 L }{ \gamma u^3 } \Eb{ \sum_{t=0}^{T-1} \norm{ \bar{\yy}_t -  \bar{\xx}_{t+\frac{1}{2}} }^2 } \\
			&\qquad + \left( \frac{ \gamma L u^3 + \gamma L (1 - u) }{(1-u)^2} - 1 \right)
			\sum_{t=0}^{T-1} \norm{ \frac{1}{K} \sum_{k=1}^K \nabla f ( \xx^k_{t + \frac{1}{2} } ) }^2
			+ 2 \hat{\gamma}^2 L^2 T \hat{\sigma}^2 + \frac{ \gamma L }{ 1 - u } \frac{1}{BK} T \sigma^2  \,.
		\end{split}
	\end{align*}
	Simplifying Lemma~\ref{lemma:y_x_dist} by choosing $ \beta = 1 $ and
	$ \bar{\xiv}_t = \norm{ \frac{1}{K} \sum_{k=1}^K \zetav_t^k }^2 $ gives:
	\begin{align*}
		\begin{split}
			\Eb{ \sum_{t=0}^{T-1} \norm{ \yy_t - \xx_{t+\frac{1}{2}} }^2 }
			& \leq \frac{ 2 u^4 \gamma^2 }{ ( 1 - u )^4 } \Eb{ \sum_{t=0}^{T-1} \norm{ \bar{\gg}_{ t + \frac{1}{2}} }^2} + \frac{2 \hat{\gamma}^2 T}{K} \hat{\sigma}^2 \\
			& \leq \frac{2 \hat{\gamma}^2 T}{K} \hat{\sigma}^2
			+\frac{ 2 u^4 \gamma^2 }{ ( 1 - u )^4 } \frac{1}{BK} T \sigma^2
			+ \frac{ 2 u^4 \gamma^2 }{ ( 1 - u )^4 } \sum_{t=0}^{T-1} \norm{ \frac{1}{K} \sum_{k=1}^K \nabla f (\xx^k_{t + \frac{1}{2}}) }^2 \,,
		\end{split}
	\end{align*}
	we have the main proof as follows:
	\begin{align*}
		\begin{split}
			&\Eb{ \sum_{t=0}^{T-1} \norm{ \nabla f ( \bar{\xx}_{t+\frac{1}{2}} ) }^2 } \\
			&\leq \frac{2 (1 - u)}{\gamma} \Eb{ f( \bar{\yy}_{0} ) - f ( \bar{\yy}_{T} ) }
			+ \frac{ (1 - u)^2 L }{ \gamma u^3 } \left(
			\frac{2 \hat{\gamma}^2 T}{K} \hat{\sigma}^2
			+\frac{ 2 u^4 \gamma^2 }{ ( 1 - u )^4 } \frac{1}{BK} T \sigma^2
			+ \frac{ 2 u^4 \gamma^2 }{ ( 1 - u )^4 } \sum_{t=0}^{T-1} \norm{ \frac{1}{K} \sum_{k=1}^K \nabla f (\xx^k_{t + \frac{1}{2}}) }^2
			\right) \\
			&\qquad + \left( \frac{ \gamma L u^3 + \gamma L (1 - u) }{(1-u)^2} - 1 \right)
			\sum_{t=0}^{T-1} \norm{ \frac{1}{K} \sum_{k=1}^K \nabla f ( \xx^k_{t + \frac{1}{2} } ) }^2
			+ \left( \frac{ \gamma L }{ 1 - u } \frac{1}{BK} \right) T \sigma^2
			+ 2 \hat{\gamma}^2 L^2 T \hat{\sigma}^2 \\
			&\leq \frac{2 (1 - u)}{\gamma} \Eb{ f( \bar{\yy}_{0} ) - f ( \bar{\yy}_{T} ) }
			+ \left( \frac{ \gamma L u^3 + \gamma L (1 - u) }{(1-u)^2} + \frac{ 2 u \gamma L }{ ( 1 - u )^2 } - 1 \right)
			\sum_{t=0}^{T-1} \norm{ \frac{1}{K} \sum_{k=1}^K \nabla f ( \xx^k_{t + \frac{1}{2} } ) }^2 \\
			&\qquad
			+ \left( \frac{ \gamma L }{ 1 - u }
			+ \frac{ 2 u \gamma L }{ ( 1 - u )^2 }  \right) \frac{T}{BK} \sigma^2
			+ (L^2 + \frac{ (1 - u)^2 L }{ \gamma u^3 K })2 \hat{\gamma}^2 T \hat{\sigma}^2 \\
			&\leq \frac{2 (1 - u)}{\gamma} \Eb{ f( \bar{\yy}_{0} ) - f ( \bar{\yy}_{T} ) } \\
			&\qquad + \left( \frac{ \gamma L \left( u^3 + u + 1 \right) }{(1-u)^2} - 1 \right)
			\sum_{t=0}^{T-1} \norm{ \frac{1}{K} \sum_{k=1}^K \nabla f ( \xx^k_{t + \frac{1}{2} } ) }^2
			+ \frac{ \gamma L ( 1 + u ) }{ (1 - u)^2 BK } T \sigma^2
			+ (L^2 + \frac{ (1 - u)^2 L }{ \gamma u^3 K })2 \hat{\gamma}^2 T \hat{\sigma}^2 \\
		\end{split}
	\end{align*}

	Dividing both sides by $ T $ and choosing $ \gamma $ such that $\frac{ L \gamma ( 1 + u + u^3 ) }{ (1 - u)^2 } - 1 < 0$
	yields:
	\begin{align*}
		\begin{split}
			\Eb{ \frac{1}{T} \sum_{t=0}^{T-1} \norm{ \nabla f ( \bar{\xx}_{t+\frac{1}{2}} ) }^2 }
			\leq \frac{2 (1 - u)}{\gamma T} \Eb{ f( \bar{\yy}_{0} ) - f ( \bar{\yy}_{T} ) }
			+ \frac{ \gamma L ( 1 + u ) }{ (1 - u)^2 BK } \sigma^2
			+ (L^2 + \frac{ (1 - u)^2 L }{ \gamma u^3 K })2 \hat{\gamma}^2 T \hat{\sigma}^2
		\end{split}
	\end{align*}
	where we have $ \gamma \leq \frac{ ( 1 - u )^2 }{ L ( 1 + u + u^3 ) } $.
\end{proof}

\subsection{Proof of Corollary~\ref{corollary:critical_minibatch_size_for_extragradient}}
\begin{proof}
	Following the proof of Corollary~\ref{corollary:critical_minibatch_size_for_nesterov_momentum},  we define $ r_0 := f ( \xx_{0} ) - f^\star$
	and $\Psi_T' := \frac{2}{T \frac{\gamma}{1-u} } r_0  + \left( \frac{ 4 \hat{\gamma}^2 L^2 }{ B } + \frac{ \gamma L (1 + 3 u) }{ (1 - u)^2 B K } \right) \sigma^2$.
	\begin{talign*}
		\begin{split}
			& \Psi_T' \overset{(a)}{\leq} \frac{2}{ T \frac{\gamma}{1-u} } r_0  + \left( \frac{ 4 u^2 \gamma^2  L^2}{ BK ( 1 - u )^2 } + \frac{ \gamma L (1 + 3 u) }{ (1 - u)^2 B K } \right) \sigma^2 \\
			&\, \overset{(b)}{\leq} \frac{2}{ T \frac{\gamma}{1-u} } r_0  + \left( \frac{ 4 u^2 \gamma L }{ BK ( 1 - u )^2 } \frac{ ( 1 - u )^2 }{ u^3 + 3 u + 1 } + \frac{ \gamma L (1 + 3 u) }{ (1 - u)^2 B K } \right) \sigma^2 \\
			&\, = \frac{2}{ T \frac{\gamma}{1-u} } r_0  + \frac{ \gamma L \sigma^2 }{ B K } \left( \frac{4 u^2}{ u^3 + 3 u + 1 } + \frac{ 1 + 3 u }{ ( 1 - u )^2 } \right) \\
			&\, = \frac{2}{ T \frac{\gamma}{1-u} } r_0  + \frac{ \gamma L \sigma^2 }{ B K } \frac{ 7 u^4 - 7 u^3 + 13 u^2 + 6 u + 1 }{ ( u^3 + 3 u + 1 ) ( 1 - u )^2 } \\
			&\, \overset{(c)}{\leq} \frac{2}{ T \frac{\gamma}{1-u} } r_0  + \frac{ \gamma L \sigma^2 }{ B K } \frac{ 19 u + 1 }{ ( u^3 + 3 u + 1 ) ( 1 - u )^2 } \,,
		\end{split}
	\end{talign*}
	where (a) follows from $ \hat{ \gamma }^2 \leq \frac{ u^2 }{ K ( 1 - u )^2 } \gamma^2 $
	( because of $\hat{\gamma} \leq \frac{ \gamma }{ K }$ and $\hat{\gamma} \leq \frac{ u^2 }{ ( 1 - u )^2 } \gamma$),
	(b) is from $ \gamma L \leq \frac{ ( 1 - u )^2 }{ 1 + 3 u + u^3 } $
	and (c) comes from $u^4 \leq u^3 \leq u^2 \leq u$.

	Similarly, we could choose two values of $\gamma$ as the minimizer,
	where $\gamma = \sqrt{ \frac{ 2 r_0 K B }{L \sigma^2 T } \frac{ ( u^3 + 3 u + 1 ) ( 1 - u )^3 }{ 19 u + 1 } } $
	and $ \gamma = \frac{ ( 1 - u )^2 }{ L ( 1 + 3 u + u^3 ) } $.
	Thus, by considering two cases as in Lemma~\ref{lemma:bound}, we have
	$\Psi_T' \leq \frac{4 L r_0 (u^3 + 3 u + 1)}{T (1-u)} + 2 \sqrt{ \frac{2 L r_0 \sigma^2 ( 19 u + 1 ) }{ KBT ( u^3 + 3 u + 1 ) ( 1 - u ) } } $.
\end{proof}

\part{Experiments}
\section{Detailed experimental setup} \label{appendix:detailed_experiment_setup}
\paragraph{Datasets.} We evaluate all methods on the following two tasks:
(1) Image classification for CIFAR-10/100~\citep{krizhevsky2009learning} ($50$K training samples and $10$K test samples with $10/100$ classes)
with the standard data augmentation and preprocessing scheme~\citep{he2016deep,huang2016deep};
(2) Language modeling for WikiText2~\citep{merity2016pointer}
(the vocabulary size is $33$K, and its train and validation set have $2$ million tokens and $217$K tokens respectively);
and (3) Neural Machine Translation for Multi30k~\citep{elliott2016multi30k}.

\paragraph{Large-batch training in practice.}
We detail the SOTA large-batch training techniques used in our experiment evaluation:
\begin{itemize}
	\item \textbf{Better Optimization:}

	      \citet{goyal2017accurate} from optimization aspect propose to linearly scale the learning rate with a few epochs warmup.
	      (1) multiply the learning rate by $K$ when the mini-batch size is multiplied by $K$;
	      (2) warmup the learning rate from $\gamma$ to $K \gamma$ through $H$ epochs, where
	      the incremental learning rate for each iteration is calculated from $\frac{K \gamma - \gamma}{H N / (KB) }$.
	      Note that $N$ is the number of total training samples and $B$ is the local mini-batch size.

	      LARS~\citep{you2017large,You2020Large} argue the layerwise difference in the weight magnitude
	      and propose to scale the gradient of each layer accordingly for better optimization.
	      The scaled gradient for $j$-th layer of $\xx_t$ at $i$-th sample follows
	      $ \nabla f_{i, j} (\xx_t) \times \left( \tilde{\gamma} \times \frac{ \norm{\xx_{t, j}} }{ \norm{ \nabla f_{i, j} (\xx_t) } + \lambda \norm{\xx_{t, j}} } \right)$,
	      where $\tilde{\gamma}$ defines how much we trust the layer to change its weights during one update.

	\item \textbf{Better Generalization:}

	      On top of these optimization techniques~\citep{goyal2017accurate,you2017large}, Post-local SGD~\citep{lin2020dont}
	      propose to further inject stochastic noise when converging to the local minima, targeting better generalization.
	      The stochastic noise is introduced by performing local SGD updates.
\end{itemize}

\paragraph{Models and training schemes.}
Three models are used in our experimental evaluation.
(1) ResNet-20~\citep{he2016deep} and VGG-11\footnote{
	Due to the resource constraints (GPU memory bound), we down-scaled the original VGG-11 by reducing the number of filters by the factor of $2$.
}~\citep{simonyan2014very} on CIFAR for image classification,
(2) two-layer LSTM\footnote{
	We borrowed and adapted the general experimental setup of~\citet{merity2017regularizing}.
	The gradient clip magnitude is $0.4$, and dropout rate is $0.40$.
	The loss is averaged over all examples and timesteps.
}~\citep{merity2017regularizing} with hidden dimension of size $128$ on WikiText-2 for language modeling,
and (3) a down-scaled transformer (factor of $2$ w.r.t. the transformer base model in~\citet{vaswani2017attention}) for neural machine translation.
Weight initialization schemes for the three models follow~\citet{goyal2017accurate,he2015delving},~\citet{merity2017regularizing} and~\citet{vaswani2017attention} respectively.

We use mini-batch SGD with a Nesterov momentum of $0.9$ without dampening for image classification and language modeling tasks,
and Adam for neural machine translation tasks.
Unless mentioned otherwise in the following experiment section,
the term ``mini-batch SGD'' indicates the mini-batch SGD with Nesterov momentum.

For experiments on image classification and language modeling, the models are trained for $300$ epochs;
the local mini-batch sizes are set to $256$ and $64$ respectively.
By default all related experiments will use learning rate scaling and warmup scheme\footnote{
	Since we will fine-tune the (to be scaled) learning rate,
	there is no difference between learning rate linear scaling~\citep{goyal2017accurate} and square root scaling~\citep{hoffer2017train} in our case.
}~\citep{goyal2017accurate,hoffer2017train}.
The learning rate will gradually warm up from a relative small value
(e.g. $0.1$) for the first few epochs.
The weight decay of ResNet-20 and LSTM are $1e$-$4$~\citep{he2016deep} and $0$~\citep{merity2017regularizing} respectively.
For the Batch Normalization (BN)~\citep{ioffe2015batch} for distributed training we follow~\citet{goyal2017accurate} and compute the BN statistics independently for each worker;
we also do not apply weight decay on the learnable BN coefficients~\citep{he2016deep}.
In addition,
the learning rate $\gamma$ in image classification task will be dropped by a factor of $10$
when the model has accessed $50\%$ and $75\%$ of the total number of training samples~\citep{he2016deep,huang2016densely}.
The LARS is only applied on image classification task\footnote{
	Our implementation relies on the PyTorch extension of \href{https://github.com/NVIDIA/apex}{NVIDIA apex}
	for mixed precision and distributed training.
}~\citep{you2017large,You2020Large}.

For experiments on neural machine translation, we use standard inverse square root learning rate schedule as in~\citet{vaswani2017attention}.
The warmup step is set to $4000$ for mini-batch size of $64$ and will be linearly scaled down by the global mini-batch size\footnote{
	We follow an \href{https://github.com/NVIDIA/DeepLearningExamples/tree/master/PyTorch/Translation/Transformer}{instruction} from NVIDIA.
}.
Other hyper-parameters follow the default setting in~\citet{vaswani2017attention}.
The first and second moment factor of the adam are set to $0.90$ and $0.98$ respectively, with the epsilon $10^{-9}$.
The values of dropout rate and the label smoothing factor are set to $0.1$. The weight decay factor is set to $0$.

\paragraph{Implementation and platform.}
Our algorithms are implemented in PyTorch\footnote{
	Our code is included in the submission for reproducibility.
}~\citep{paszke2017automatic} and the distributed training is supported by MPI and Kubernetes.

\subsection{Hyper-parameter tuning procedure and the corresponding values} \label{appendix:hyperparameter_values}
We carefully tune the learning rate,
the trust term $\tilde{\gamma}$ in~\citet{you2017large}
and our extrapolation term $\hat{\gamma}$ for each experimental setup.
For example, for ResNet-20 on CIFAR-10 we tune the optimal unscaled learning rate (i.e. $\gamma / K$) for a fixed mini-batch size $B$
in the range of $\{ 0.05, 0.10, 0.15, 0.20 \}$ and then linearly scale (and warmup) the learning rate by the factor of $K$.
The $\tilde{\gamma}$ is initially searched within $\{ 0.01, 0.02, 0.03 \}$,
where $0.02$ is the default hyperparameter used in \href{https://github.com/NVIDIA/apex}{NVIDIA apex}.
The extrapolation term $\hat{\gamma}$ for~\algopt and~\algoptadam is tuned by scaling the $\frac{\gamma}{K}$ with the factor searched from $\{ 1, 2, 4 \}$;
we perform extensive hyper-parameter tuning for the noise variants of the extrapolated SGD, where the scaling factor of $\frac{\gamma}{K}$ is searched from $\{ 0.1, 0.25, 0.5, 1, 2, 4 \}$.
The tuning procedure of the hyper-parameters ensures that the best hyper-parameter lies in the middle of our search grids;
otherwise we extend our search grid.

\section{Algorithmic details}
\subsection{\algopt with Post-local SGD} \label{appendix:local_sgd}
We combine \algopt with post-local SGD in Algorithm~\ref{algo:our_main_algo_postlocal_variant}.
We omit the extrapolation step in line $3$ and line $8$ when $t \!=\! 0$.
The original form of post-local SGD refers to~\citet{lin2020dont}.

\begin{algorithm}[!]\footnotesize
	\caption{\small\itshape {\algopt integrated with post-local SGD.}}
	\label{algo:our_main_algo_postlocal_variant}
	\begin{algorithmic}[1]
		\REQUIRE learning rate $\gamma$, inner learning rate $\hat{\gamma}$, momentum factor $u$,
		initial parameter $\xx_0$, initial moment vector $\vv_0 = 0$, time step $t = 0$, worker index $k$,
		transition phase (iteration $t_0$) for post-local SGD, and local update steps $H$.

		\WHILE{$\xx_t$ not converged}
		\IF{$t \leq t_0$}
		\STATE $\xx_{t+\frac{1}{4}}^k = \xx_t - \hat{\gamma} \nabla f( \xx_{t-\frac{1}{2}}^k )$
		\STATE $\xx_{t+\frac{1}{2}}^k = \xx_{t+\frac{1}{4}}^k + u \vv_t$
		\STATE $\vv_{t+1} = u \vv_t - \frac{\gamma}{K} \sum_{k=1}^{K} \nabla f( \xx_{t+\frac{1}{2}}^k )$
		\STATE $\xx_{t + 1} = \xx_{t} + \vv_{t+1}$
		\ELSE
		\STATE $\xx_{t+\frac{1}{4}}^k = \xx_t^k - \hat{\gamma} \nabla f( \xx_{t-\frac{1}{2}}^k )$
		\STATE $\xx_{t+\frac{1}{2}}^k = \xx_{t+\frac{1}{4}}^k + u \vv_t^k$
		\STATE $\vv_{t+1}^k = u \vv_t^k - \gamma \nabla f( \xx_{t+\frac{1}{2}}^k )$
		\STATE $\xx_{t + 1}^k = \xx_{t}^k + \vv_{t+1}^k$

		\IF{$t \text{ mod } H = 0$}
		\STATE $\xx_{t+1}^k = \frac{1}{K} \sum_{k=1}^K \xx_{t+1}^k$
		\ENDIF
		\ENDIF

		\ENDWHILE
		\ENSURE $\xx_t$.
	\end{algorithmic}
\end{algorithm}

\subsection{Generalized \algopt with Adam} \label{appendix:extrap_adam}
We present the generalized \algopt with Adam (\algoptadam) in Algorithm~\ref{algo:our_main_algo_adam_variant}.
We omit the extrapolation step in line $2$ when $t \!=\! 0$.

\begin{algorithm}[!]\footnotesize
	\caption{\small\itshape {\algoptadam.}}
	\label{algo:our_main_algo_adam_variant}
	\begin{algorithmic}[1]
		\REQUIRE learning rate $\gamma$, inner learning rate $\hat{\gamma}$,
		initial parameter $\xx_0$, initial first-order moment vector $\mm_0 = 0$, initial second-order moment vector $\vv_0 = 0$,
		time step $t = 0$, worker index $k$.
		\WHILE{$\xx_t$ not converged}
		\STATE $\xx_{t+\frac{1}{2}} = \xx_t - \hat{\gamma} \frac{ \beta_1 \mm_{t - 1} + (1 - \beta_1) \nabla F( \xx_{t-\frac{1}{2}}^k ) }{ \beta_2 \vv_{t - 1} + (1 - \beta_2) \nabla F( \xx_{t-\frac{1}{2}}^k )^2 + \epsilon }$
		\STATE $\gg_t = \frac{1}{K} \sum_{k=1}^{K} \nabla F( \xx_{t+\frac{1}{2}}^k )$
		\STATE $\mm_t = \beta_1 \mm_{t - 1} + (1 - \beta_1) \gg_t$
		\STATE $\vv_t = \beta_2 \vv_{t - 1} + (1 - \beta_2) \gg_t^2$
		\STATE $\xx_{t+1} = \xx_t - \gamma \frac{ \mm_t }{ \sqrt{ \vv_t } + \epsilon }$
		\ENDWHILE
		\ENSURE $\xx_t$.
	\end{algorithmic}
\end{algorithm}

\clearpage
\section{Additional results}
\subsection{ResNet-20 on CIFAR-10} \label{appendix:resnet20_cifar}
The top-1 test accuracy for all related methods on CIFAR-10 is shown in Figure~\ref{fig:resnet20_cifar10_learning_curves_complete_1}.
\begin{figure}[!h]
	\centering
	\includegraphics[width=0.5\textwidth,]{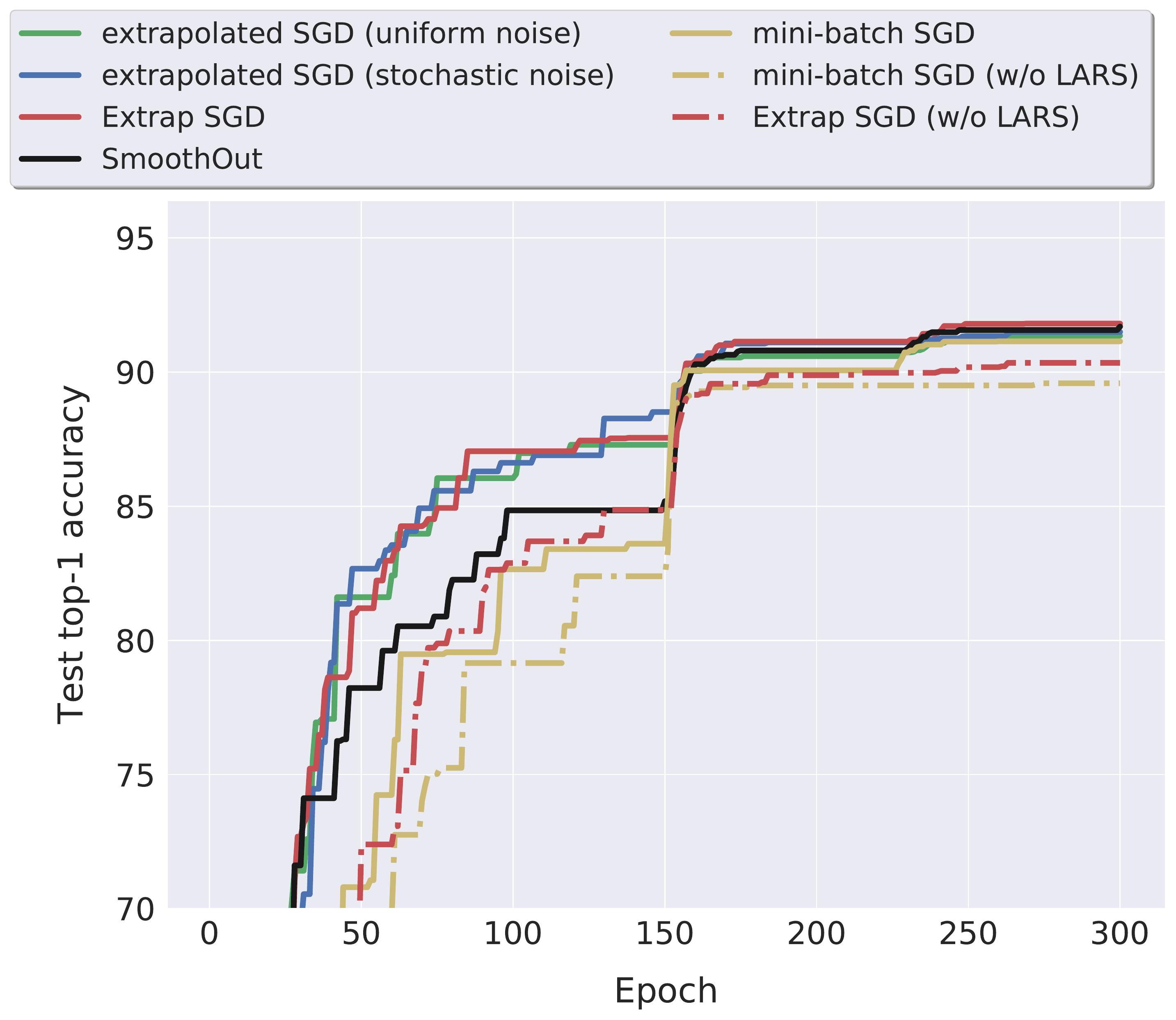}
	\caption{\small
		Understanding the learning behaviors of different methods on the large-batch training
		(with mini-batch size $8,192$ on $32$ workers) for training ResNet-20 on CIFAR-10.
		The hyper-parameters are fine-tuned, which are corresponding to the results shown in Table~\ref{tab:two_tasks_extreme_minibatch_size_for_all_methods};
		by default the learning rate is decayed by $10$ at epoch $150$ and $225$.
	}
	\label{fig:resnet20_cifar10_learning_curves_complete_1}
\end{figure}

The performance of post-local SGD, \algopt and their combination on CIFAR10/100 is shown Table~\ref{tab:resnet20_cifar_8k_extrapsgd_with_postlocal}.
\begin{table*}[!h]
	\caption{\small
		The test top-1 accuracy of integrating \algopt with post-local SGD~\citep{lin2020dont}.
		The performance of ResNet-20 on CIFAR-10/100 is evaluated with global mini-batch size $8,192$ ($K\!=\!32$).
		By default we use the learning rate scaling and warmup in~\citet{goyal2017accurate} and LARS in~\citet{you2017large}.
		We individually finetune $\gamma$ and $\tilde{\gamma}$ for each method;
		the local update step $H$ tuned from $\{ 4, 8, 16 \}$ (as in~\citet{lin2020dont}) in general improves the performance and we report the best performance with $H\!=\!8$.
		The results are averaged over three different seeds.
	}
	\label{tab:resnet20_cifar_8k_extrapsgd_with_postlocal}
	\centering
	\resizebox{.8\textwidth}{!}{%
		\begin{tabular}{ccccccc}
			\toprule
			& \multicolumn{1}{c}{mini-batch SGD} & \parbox{4cm}{\centering mini-batch SGD \\ (with post-local SGD)}             & \algopt                                                       & \parbox{4cm}{\centering \algopt \\ (with post-local SGD)}     \\ \midrule
			CIFAR-10  & $91.36 \pm 0.19$ & $91.73 \pm 0.25$ & $91.72 \pm 0.11$ & $92.23 \pm 0.02$ \\
			CIFAR-100 & $65.79 \pm 0.46$ & $67.39 \pm 0.18$ & $66.63 \pm 0.32$ & $68.06 \pm 0.21$ \\
			\bottomrule
		\end{tabular}%
	}
\end{table*}

The effects of different combinations between local batch sizes and worker numbers are evaluated in
Table~\ref{tab:resnet20_cifar10_8k_different_combination_of_minibatch_size_and_n_workers}.
\begin{table*}[!h]
	\caption{\small The test top-1 accuracy for ResNet-20 on CIFAR-10 under different combinations of local mini-batch size $B$ and number of workers $K$.
		The global mini-batch size is always set to $8,192$ and we vary $K$ workers (MPI processes).
		We individually finetune $\gamma$ and $\tilde{\gamma}$ for each method on different setups,
		and the reported results are averaged over three different seeds.
	}
	\label{tab:resnet20_cifar10_8k_different_combination_of_minibatch_size_and_n_workers}
	\centering
	\resizebox{.8\textwidth}{!}{%
		\begin{tabular}{ccccc}
			\toprule
			               & $(B\!=\!512, K\!=\!16)$ & $(B\!=\!256, K\!=\!32)$ & $(B\!=\!128, K\!=\!64)$ & $(B\!=\!64, K\!=\!128)$ \\ \midrule
			Mini-batch SGD & $91.35 \pm 0.19$        & $91.36 \pm 0.19$        & $91.29 \pm 0.13$        & $91.32 \pm 0.17$        \\
			\algopt        & $91.62 \pm 0.32$        & $91.72 \pm 0.11$        & $91.88 \pm 0.27$        & $91.89 \pm 0.24$        \\ \bottomrule
		\end{tabular}%
	}
\end{table*}

\subsection{LSTM on WikiText2} \label{appendix:lstm_wikitext2}
The learning curves of \algopt and mini-batch SGD with LSTM model on WikiText2 dataset
are presented in Figure~\ref{fig:lstm_wikitext2_learning_curves_}.
\begin{figure}[!h]
	\centering
	\subfigure[\small $K\!=\!24$.]{
		\includegraphics[width=0.475\textwidth,]{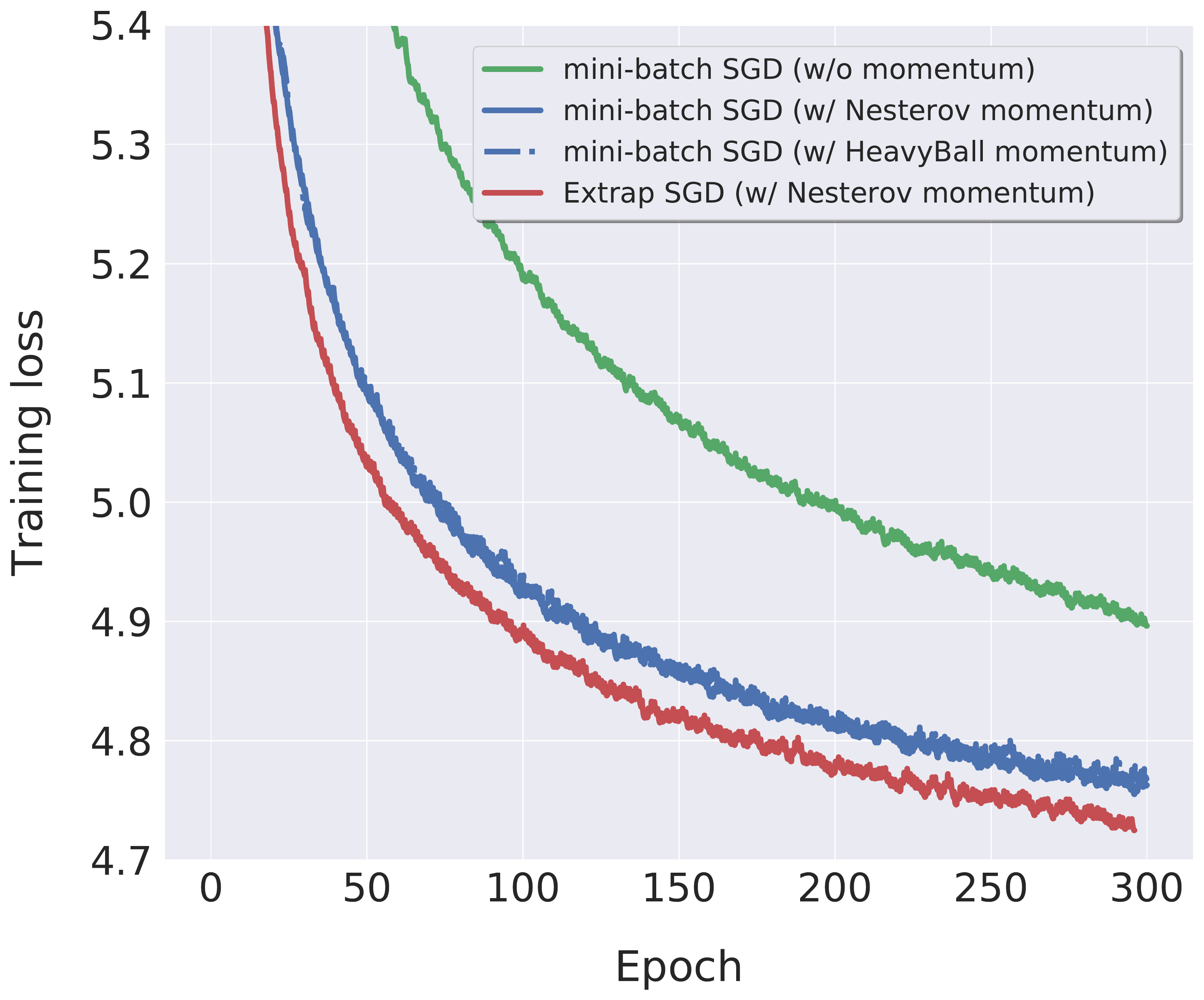}
		\label{fig:lstm_wikitext2_k24_tr_loss_fixed_warmup_}
	}
	\hfill
	\subfigure[\small $K\!=\!24$.]{
		\includegraphics[width=0.475\textwidth,]{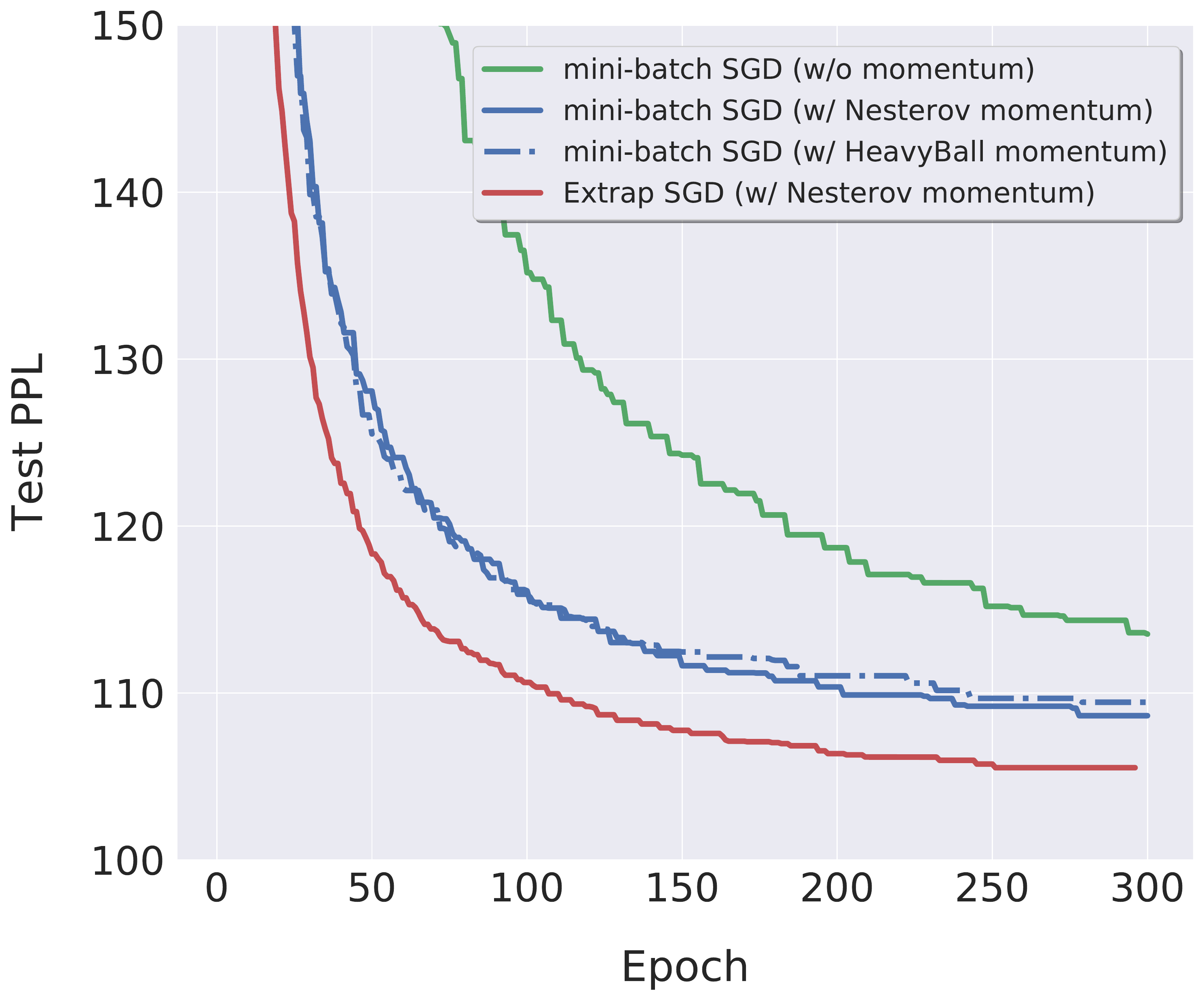}
		\label{fig:lstm_wikitext2_k24_te_ppl_fixed_warmup_}
	}
	\hfill
	\subfigure[\small $K\!=\!24$ V.S. $K\!=\!48$]{
		\includegraphics[width=0.475\textwidth,]{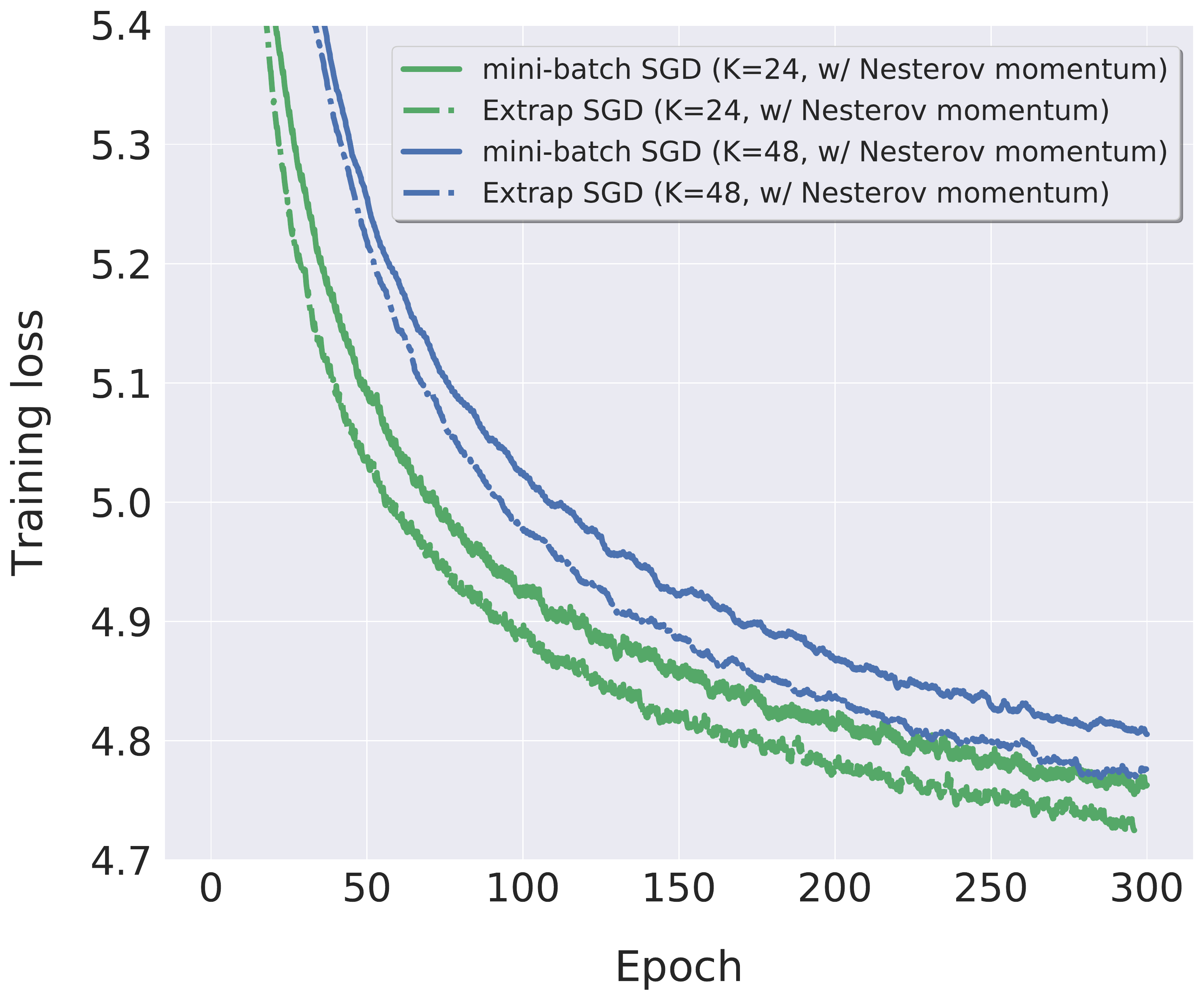}
		\label{fig:lstm_wikitext2_k48_te_loss_fixed_warmup_}
	}
	\hfill
	\subfigure[\small $K\!=\!24$ V.S. $K\!=\!48$]{
		\includegraphics[width=0.475\textwidth,]{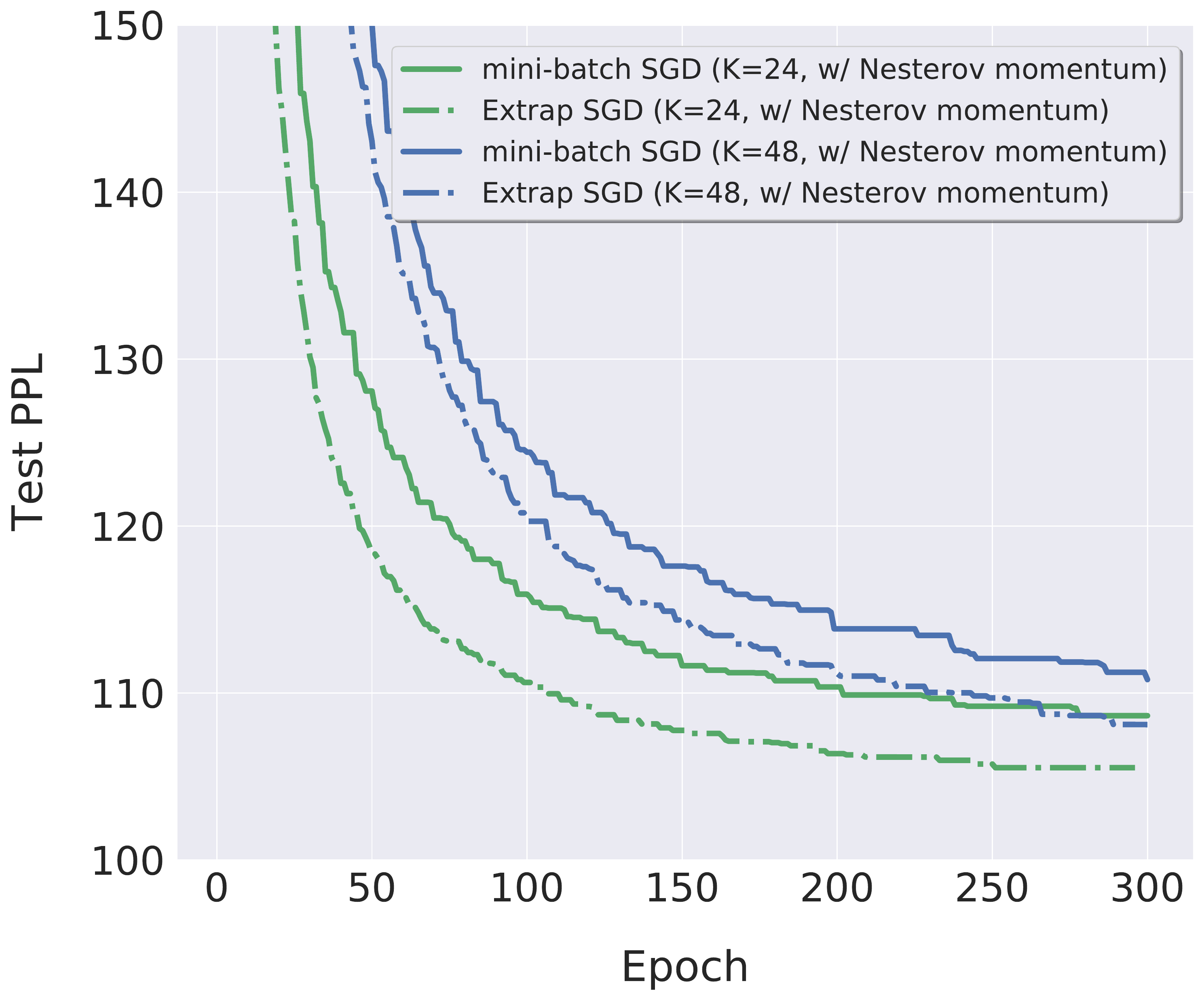}
		\label{fig:lstm_wikitext2_k48_te_ppl_fixed_warmup_}
	}
	\caption{\small
		The learning curves and perplexity (PPL, the lower the better) of training Wikitext-2 on LSTM.
		The global mini-batch size are $1536$ and $3072$ for $K\!=\!24$ and $K\!=\!48$ respectively,
		accounting for $2\%$ and $4\%$ of the total training data.
		We use the learning rate scaling and warmup in~\citet{goyal2017accurate}.
		We finetune the $\gamma$ for different variants of mini-batch SGD
		and~\algopt have no additional tuning.
		The results of the inline table are averaged over three different seeds.
	}
	\vspace{-0.5em}
	\label{fig:lstm_wikitext2_learning_curves_}
\end{figure}

\subsection{Impact of different momentum factors} \label{appendix:resnet20_cifar10_momentum_impact}
The impacts of momentum factors on \algopt and mini-batch SGD are demonstrated in Figure~\ref{fig:resnet20_cifar10_impact_of_momentum_highlight}
and Figure~\ref{fig:resnet20_cifar10_impact_of_momentum}.

\begin{figure}[!h]
	\centering
	\subfigure[\small The training loss of \algopt v.s. Mini-batch SGD.]{
		\includegraphics[width=0.475\textwidth,]{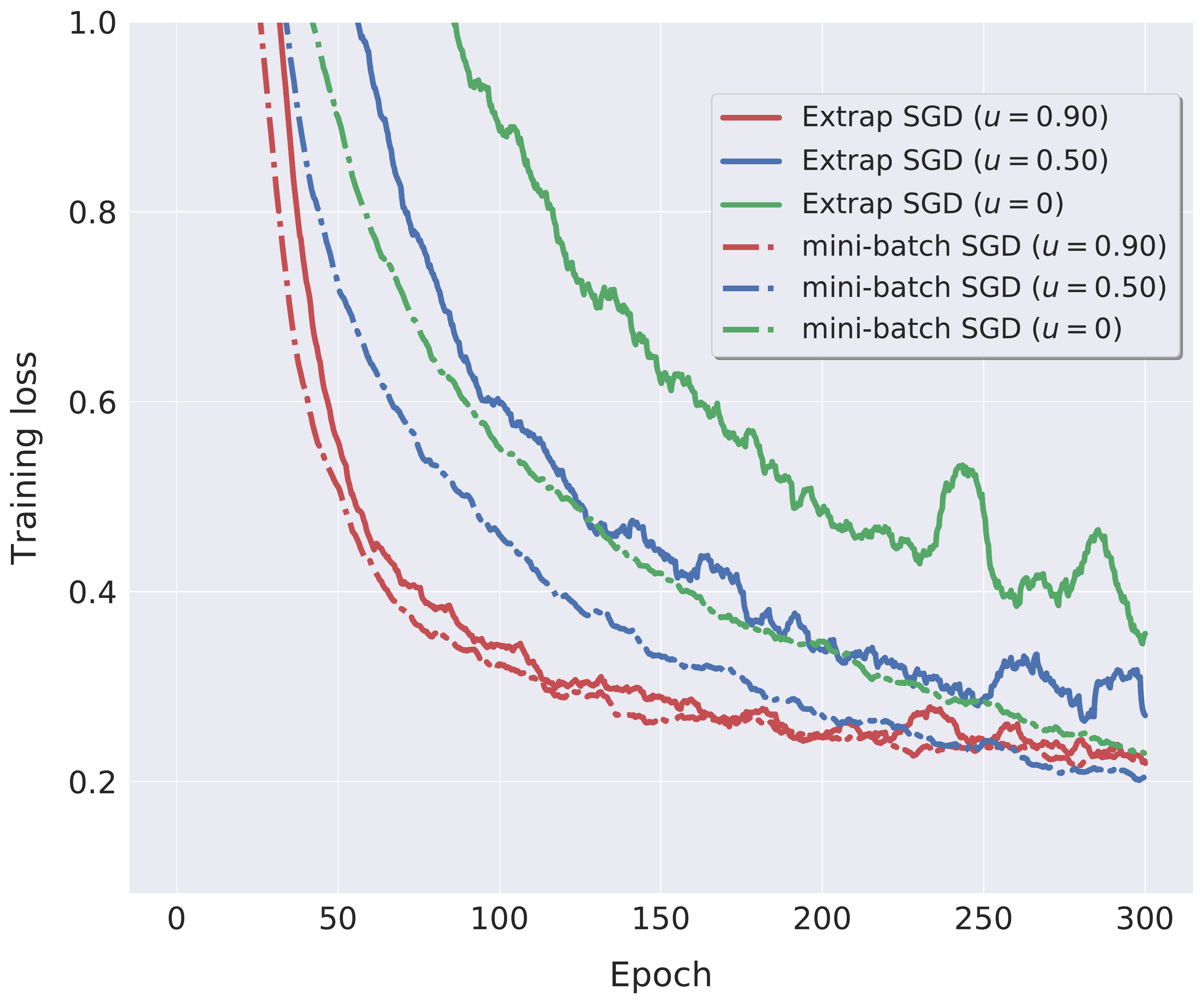}
		\label{fig:resnet20_cifar10_impact_of_momentum_minibatch_sgd_vs_extrap_sgd_tr_loss}
	}
	\hfill
	\subfigure[\small The test top-1 accuracy of \algopt v.s. Mini-batch SGD.]{
		\includegraphics[width=0.475\textwidth,]{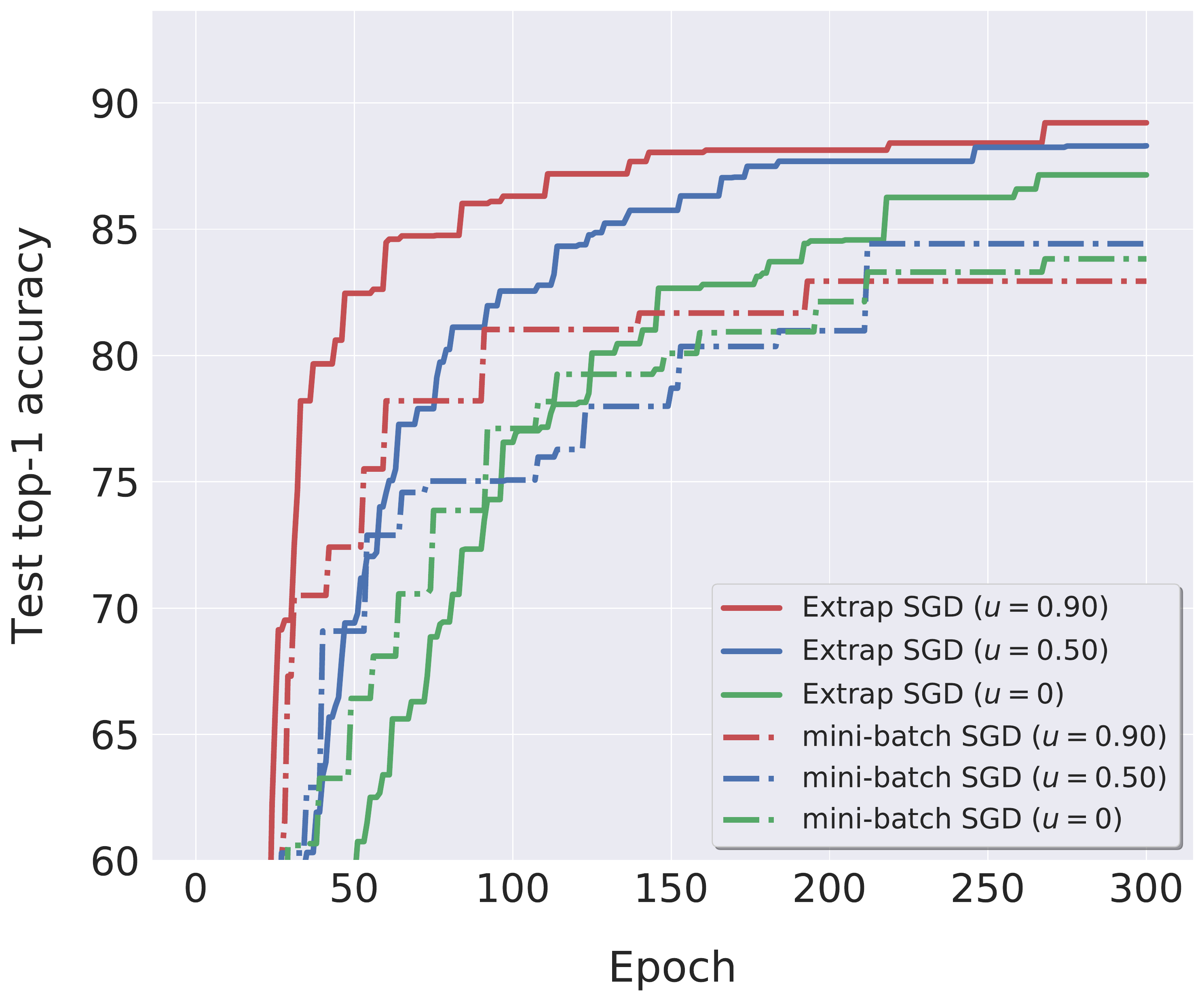}
		\label{fig:resnet20_cifar10_impact_of_momentum_minibatch_sgd_vs_extrap_sgd_te_top1_}
	}
	\caption{\small
		Understanding the behaviors of \algopt and mini-batch SGD for different momentum factors $u$.
		The curves are evaluated on ResNet-20 with CIFAR-10 for mini-batch size $8,192$;
		we use constant learning rate $\gamma$ and the LARS trust term $\tilde{\gamma}$ over the whole training procedure,
		for both of mini-batch SGD and \algopt (with default extrapolation term $\hat{\gamma}$).
		The value of $\gamma$ and $\tilde{\gamma}$ correspond to the tuned optimal value in Table~\ref{tab:two_tasks_extreme_minibatch_size_for_all_methods}.
	}
	\vspace{-0.5em}
	\label{fig:resnet20_cifar10_impact_of_momentum_highlight}
\end{figure}

\begin{figure}[!h]
	\centering
	\subfigure[\small The training loss of Mini-batch SGD.]{
		\includegraphics[width=0.475\textwidth,]{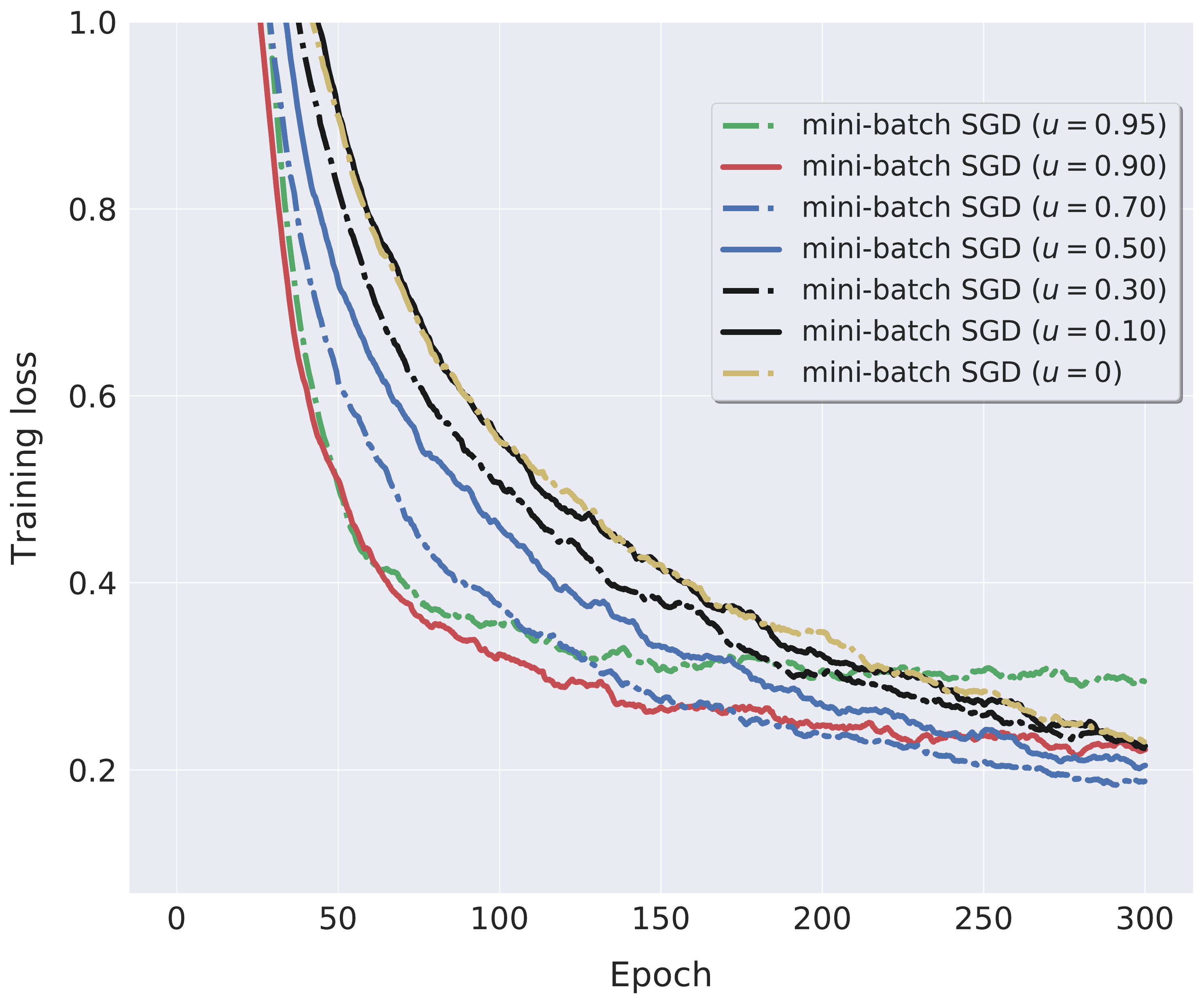}
		\label{fig:resnet20_cifar10_impact_of_momentum_minibatch_sgd_tr_loss}
	}
	\hfill
	\subfigure[\small The test top-1 accuracy of Mini-batch SGD.]{
		\includegraphics[width=0.475\textwidth,]{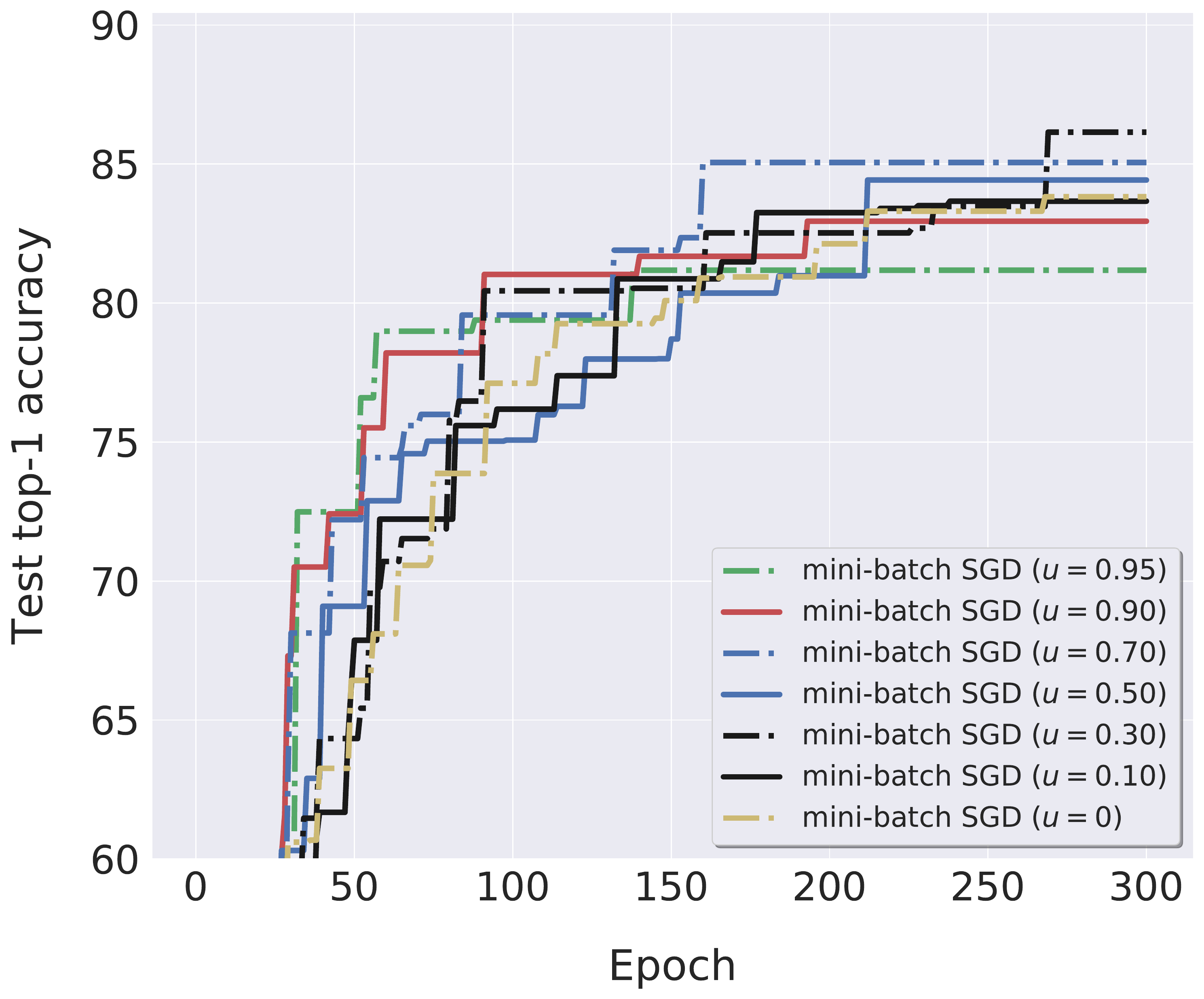}
		\label{fig:resnet20_cifar10_impact_of_momentum_minibatch_sgd_te_top1}
	}
	\hfill
	\subfigure[\small The training loss of \algopt.]{
		\includegraphics[width=0.475\textwidth,]{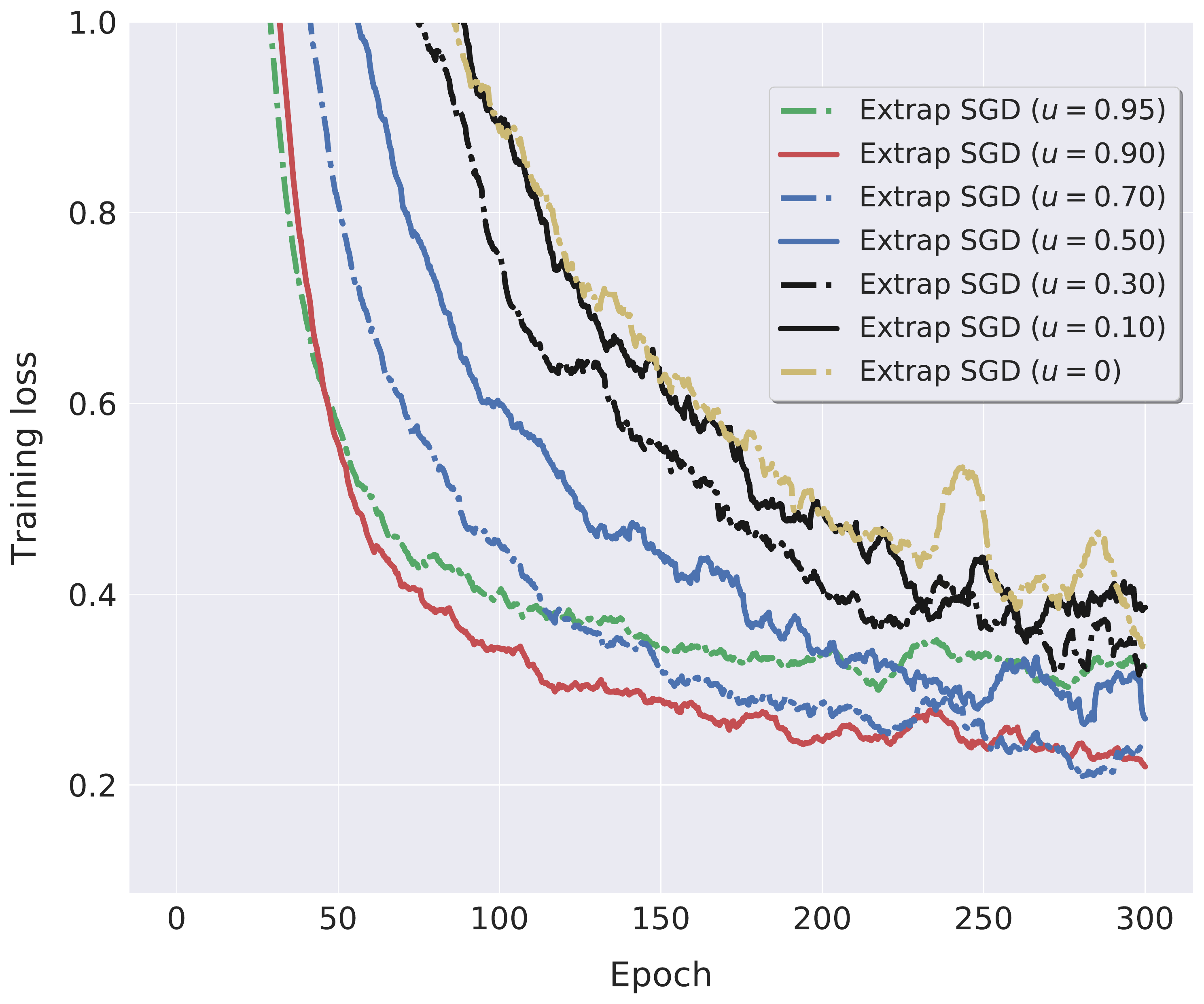}
		\label{fig:resnet20_cifar10_impact_of_momentum_extrap_sgd_tr_loss}
	}
	\hfill
	\subfigure[\small The test top-1 accuracy of \algopt.]{
		\includegraphics[width=0.475\textwidth,]{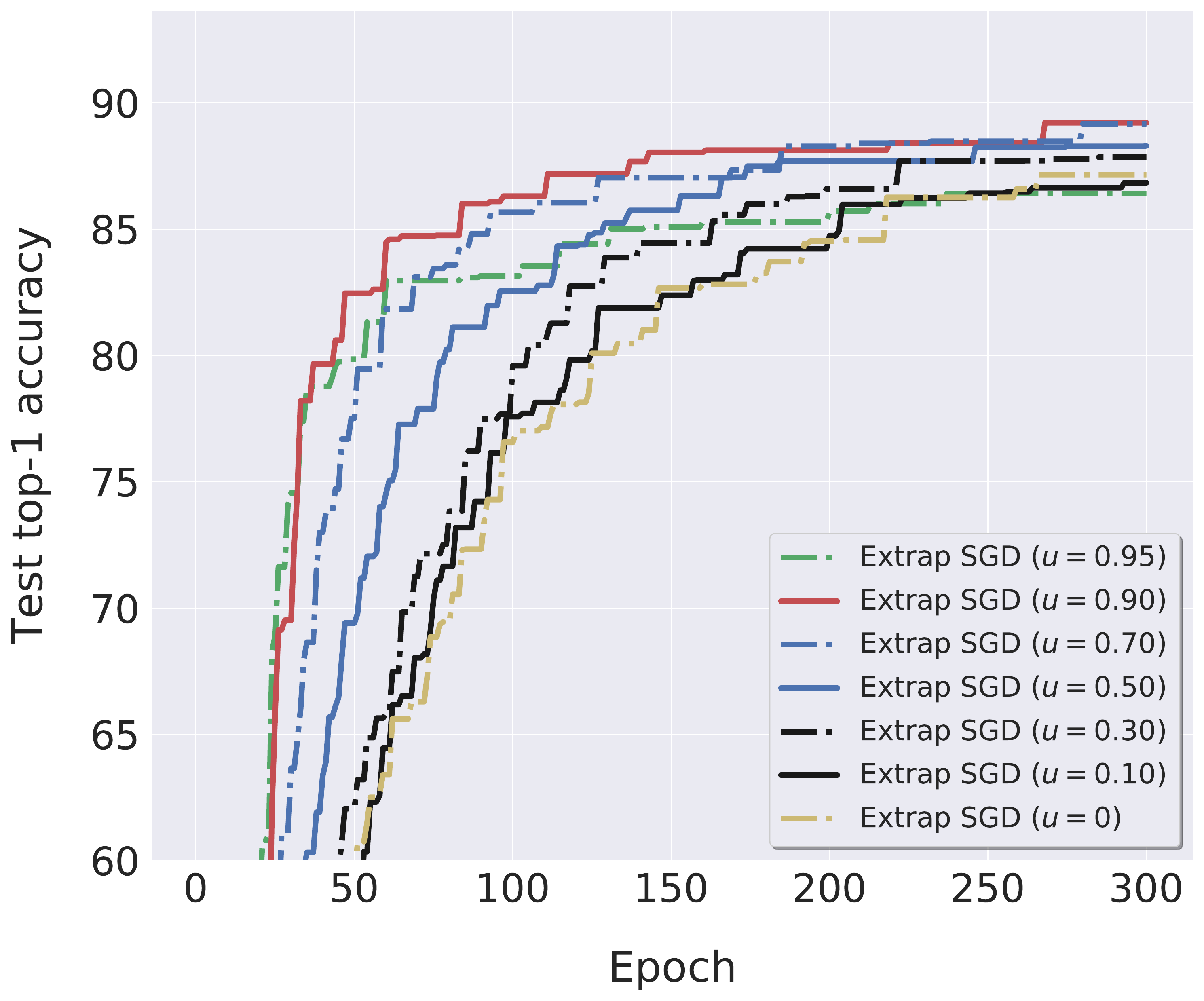}
		\label{fig:resnet20_cifar10_impact_of_momentum_extrap_sgd_te_top1}
	}
	\caption{\small
		The performance of~\algopt and mini-batch SGD for different momentum factors $u$.
		The curves are evaluated on ResNet-20 with CIFAR-10 for mini-batch size $8,192$;
		we use constant learning rate $\gamma$ and the LARS trust term $\tilde{\gamma}$ over the whole training procedure,
		for both of mini-batch SGD and \algopt (with default extrapolation term $\hat{\gamma}$).
		The value of $\gamma$ and $\tilde{\gamma}$ correspond to the tuned optimal value in Table~\ref{tab:two_tasks_extreme_minibatch_size_for_all_methods}.
	}
	\vspace{-0.5em}
	\label{fig:resnet20_cifar10_impact_of_momentum}
\end{figure}

\end{document}